\documentclass[12pt]{article}
\usepackage{enumerate}
\usepackage[margin=1in, top=0.8in]{geometry}
\usepackage[OT1]{fontenc}
\usepackage{amsmath,amssymb}
\usepackage{mathtools}
\usepackage{natbib}
\usepackage[usenames]{color}
\usepackage{dsfont}
\usepackage{pifont}
\usepackage{pgfplots}
\usepackage{smile}
\usepackage{multirow}
\usepackage{rotating}
\usepackage{enumerate}
\usepackage{esvect}
\usepackage{tikz}
\usetikzlibrary{patterns}
\usetikzlibrary{arrows}
\usepackage{color,xcolor}
\usepackage{colortbl}
\usepackage[colorlinks,
            linkcolor=red,
            anchorcolor=red,
            citecolor=blue
            ]{hyperref}
\usepackage{algorithm}
\usepackage{algorithmic}
\usepackage{cancel}
\usepackage{diagbox}
\usepackage{tcolorbox}
\usepackage{graphicx,subfig,wrapfig,diagbox,adjustbox}
\newcommand{\cmark}{\ding{51}}%
\newcommand{\xmark}{\ding{55}}%
\usepackage[frozencache=true,cachedir=minted-cache]{minted} 
\usepackage{makecell}
\usepackage{listings}
\usepackage{tabularx}
\definecolor{DarkGreen}{RGB}{30,120,22}

\usepackage{fancyvrb}
\RecustomVerbatimCommand{\VerbatimInput}{VerbatimInput}
{fontsize=\footnotesize,
 breaklines=true,
 breakanywhere=true, 
 breaksymbol=,
 frame=single,  
 framesep=0.5em,
 labelposition=topline,
}

\definecolor{ultramarine}{rgb}{0.,0.1,0.9}

\usepackage{paralist}

\def\supp{\mathop{\text{supp}}}

\long\def\comment#1{}

\def\cS{{\mathcal{S}}}

\newcommand{\bel}{\begin{eqnarray}\label}
\newcommand{\eel}{\end{eqnarray}}
\newcommand{\bes}{\begin{eqnarray*}}
\newcommand{\ees}{\end{eqnarray*}}

\let\emptyset\varnothing
\let\hat\widehat
\let\tilde\widetilde
\def\mid{\,|\,}
\def\eps{\epsilon}

\def\supp{\mathop{\text{supp}\kern.2ex}}

\def\supp{\mathop{\text{supp}}}

\def\##1\#{\begin{align}#1\end{align}}
\def\$#1\${\begin{align*}#1\end{align*}}

\usepackage{undertilde}
\theoremstyle{plain}

\title{Reason for Future, Act for Now: A Principled Framework for Autonomous LLM Agents with Provable Sample Efficiency}
\author{Zhihan Liu\thanks{Equal contribution.} \thanks{Northwestern University. \texttt{\{zhihanliu2027,shenaozhang2028,hongyiguo2025,boyiliu2018\}\\@u.northwestern.edu,zhaoranwang@gmail.com}} \qquad
    Hao Hu\footnotemark[1] \thanks{Tsinghua University. \texttt{huh22@mails.tsinghua.edu.cn}} \qquad
    Shenao Zhang\footnotemark[1] \footnotemark[2]\\
    Hongyi Guo\footnotemark[2] \qquad
    Shuqi Ke \thanks{The Chinese University of Hong Kong. \texttt{shuqike@link.cuhk.edu.cn}} \qquad
    Boyi Liu\footnotemark[2] \qquad
    Zhaoran Wang\footnotemark[2]}

\begin{document}

\maketitle
\begin{abstract}
Large language models (LLMs) demonstrate impressive reasoning abilities, but translating reasoning into actions in the real world remains challenging. In particular, it is unclear how to complete a given task provably within a minimum number of interactions with the external environment, e.g., through an internal mechanism of reasoning. To this end, we propose the first framework with provable regret guarantees to orchestrate reasoning and acting, which we call ``reason for future, act for now" (\texttt{RAFA}). Specifically, we design a prompt template for reasoning that learns from the memory buffer and plans a future trajectory over a long horizon (``reason for future"). At each step, the LLM agent takes the initial action of the planned trajectory (``act for now"), stores the collected feedback in the memory buffer, and reinvokes the reasoning routine to replan the future trajectory from the new state. 
The key idea is to cast reasoning in LLMs as learning and planning in Bayesian adaptive Markov decision processes (MDPs). Correspondingly, we prompt LLMs with the memory buffer to estimate the unknown environment (learning) and generate an optimal trajectory for multiple future steps that maximize a value function (planning). The learning and planning subroutines are performed in an ``in-context" manner to emulate the actor-critic update for MDPs. Our theoretical analysis establishes a $\sqrt{T}$ regret, while our experimental validation demonstrates superior empirical performance.  Here, $T$ denotes the number of online interactions. Project page:  \url{https://agentification.github.io/RAFA}. 
\end{abstract}

\tableofcontents

\newpage

\newtheorem{proposition}{Proposition}
\section{Introduction}
Large language models (LLMs) exhibit remarkable reasoning abilities, which open a new avenue for agents to interact with the real world autonomously. However, turning reasoning into actions remains challenging. Specifically, although LLMs are equipped with the prior knowledge obtained through pretraining, it is stateless in nature and ungrounded in the real world, which makes the resulting action suboptimal. To bridge the reasoning-acting gap, we aim to design an internal mechanism of reasoning on top of LLMs, which optimizes actions iteratively by incorporating feedback from the external environment. In particular, we focus on the sample efficiency of autonomous LLM agents in interactive decision-making tasks, which plays a key role in their practical adoption, especially when interactions are costly and risky. Our primary goal is to enable agents to complete a given task in a guaranteed manner through reasoning within a minimum number of interactions with the external environment. 

Reinforcement learning (RL) is a well-studied paradigm for improving actions by collecting feedback. However, to tailor existing RL techniques for autonomous LLM agents, we lack a rigorous mapping between RL and LLMs, which leads to various conceptual discrepancies. For example, RL operates in a numerical system, where rewards and transitions are defined by scalars and probabilities. In comparison, the inputs and outputs of LLMs are described by tokens in a linguistic system. As another example, LLMs are trained on a general-purpose corpus and remain fixed throughout the interactive process. In contrast, RL trains actors and critics via parameter updates on the collected feedback iteratively. Thus, it appears inappropriate to treat LLMs as actors or critics under the RL framework, although all of them are parameterized by deep neural networks. Moreover, it remains unclear what reasoning with LLMs means under the RL framework, e.g., what are the inputs and outputs of a reasoning routine and how reasoning should be coordinated with acting. Such conceptual discrepancies prevent us from establishing a principled framework beyond borrowing the ``trial and error'' concept from RL straightforwardly and make it difficult to establish the theoretical guarantee. 

To address such conceptual discrepancies, we formalize reasoning and acting with LLMs under a Bayesian adaptive Markov decision process (MDP) framework, where the latent variable of interest is the unknown environment. The starting point is to cast the full history of states (of the external environment), actions, rewards, and their linguistic summaries in the memory buffer as the information state of Bayesian adaptive MDPs. Throughout the interactive process, the information state accumulates a growing collection of feedback from the external environment, which is mapped to an optimized action at each step by an internal mechanism of reasoning. As detailed below, we construct the reasoning routine through two key subroutines, namely learning and planning, which are instantiated by LLMs with specially designed prompts. 
\textbf{(a)} The learning subroutine forms an estimate of the external environment given the memory buffer, where LLMs are prompted to infer the transition and reward models (model) or/and the value function (critic).  
\textbf{(b)} The planning subroutine generates an optimal policy (actor) or trajectory for multiple future steps, which maximizes the value function (up to a certain error). Depending on the specific configuration of the state and action spaces (continuous versus discrete) and the transition and reward models (stochastic versus deterministic), the planning subroutine emulates the value iteration algorithm, the random shooting algorithm, or the Monte-Carlo tree-search algorithm.  

Although LLMs remain fixed throughout the interactive process, we can reduce their estimation uncertainty by prompting the growing collection of feedback from the external environment as contexts, which is verified both theoretically and empirically in this paper. From the perspective of Bayesian adaptive MDPs, LLMs can be considered as some functional of the posterior of the environment (for example, Bayesian model averaging \citep{wasserman2000bayesian}), hence the estimation uncertainty is reduced with increasing information via interactions.  For several tasks, we demonstrate that LLMs can make a more precise prediction when prompted with more data as contexts. 
Hence, LLMs can play a similar role of model estimators in the design of online RL algorithms for interactions. We improve the accuracy of LLMs by simply adding the new feedback to the memory buffer as contexts, instead of performing explicit parameter updates (such as gradient descent) on deep neural networks as in existing RL methods.

\begin{figure*}[t]
  \begin{minipage}[c]{0.775\textwidth}
    \includegraphics[width=\textwidth]{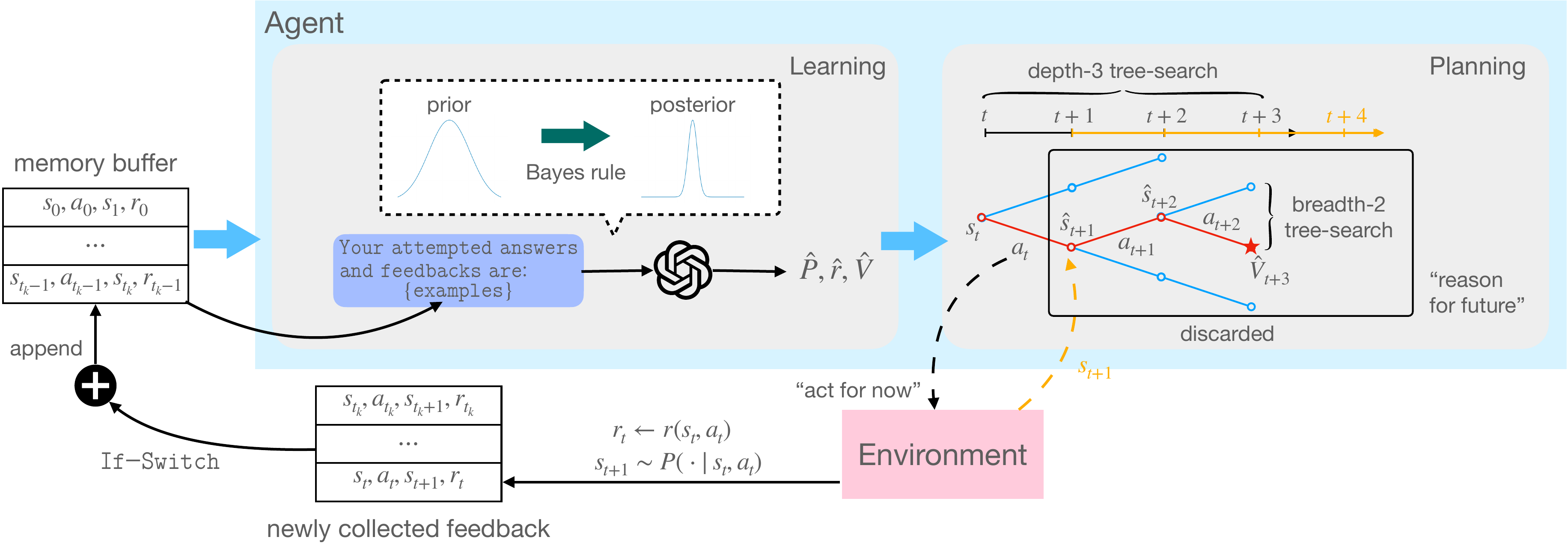}
  \end{minipage}\hfill
  \begin{minipage}[c]{0.2\textwidth}
    \caption{Illustration of the \texttt{RAFA} (``reason for future, act for now'') framework.} \label{fig:rafa_illustration}
  \end{minipage}
\end{figure*}

We conclude our contributions in this paper from three perspectives.
\textbf{(a)} We establish the LLM-RL correspondence and design a principled framework \texttt{RAFA} for orchestrating the reasoning and acting of LLMs. 
\textbf{(b)} Our empirical validation shows that \texttt{RAFA} outperforms various existing frameworks in interactive decision-making tasks, including ALFWorld, BlocksWorld, Game of 24, and a new benchmark based on Tic-Tac-Toe.
\textbf{(c)} Our theoretical analysis proves that \texttt{RAFA} achieves a $\sqrt{T}$ regret, explaining why    \texttt{RAFA} demonstrates strong empirical performance. Here, $T$ denotes the number of online interactions. We also provide two provably efficient variants of \texttt{RAFA} to implement efficient exploration for more complex tasks.
 
\subsection{Literature}

\paragraph{Reasoning with LLM.} We build on a recent line of work that develops various prompting schemes to improve the reasoning performance of LLMs. ``Chain of thoughts'' (``CoT'') \citep{chain_of_thought} decomposes a challenging problem into several reasoning stages and guides LLMs to solve them one by one. As generalizations, ``tree of thoughts'' \citep{tree_of_thought}, ``graph of thoughts'' \citep{graph_of_thought}, ``algorithm of thoughts'' \citep{sel2023algorithm}, and ``cumulative reasoning'' \citep{zhang2023cumulative} provide different graph-search schemes to guide LLMs. See also \cite{wang2022self, creswell2022selection, creswell2022faithful,guo2024can,zhang2024can}. Also, ``reasoning via planning'' (``RAP'') \citep{rap} emulates the Monte-Carlo tree-search (MCTS) algorithm to reduce the search complexity. \cite{pouplin2024retrieval} improve LLM reasoning process with MCTS and formulate the reasoning process as an MDP. \cite{sun2023query} use offline inverse RL to optimize the prompts for arithmetic problems.  For embodied LLM agents, \cite{huang2022language} propose to decompose a complex task into multiple executable steps. Most of them focus on general reasoning tasks, e.g., solving a mathematical or logic puzzle, where LLMs generate a detailed trace (trajectory) of arguments through an internal mechanism to reach a final answer. Here, LLMs play the same role as the planning subroutine in \texttt{RAFA}. In contrast, we focus on interactive decision-making tasks, where autonomous LLM agents collect feedback from the external environment to optimize actions iteratively. In particular, we aim to complete a given task within a minimum number of interactions with the external environment. To this end, it is essential to operate three interleaved modules, namely learning, planning, and acting, in a closed loop. While it is feasible to incorporate existing graph-search or MCTS schemes as the planning subroutine for generating trajectories, our core contribution is a principled framework that executes a selected subset of the planned trajectory to collect feedback (``act for now'') and replans an improved trajectory from the new state by learning from feedback (``reason for future''). From an RL perspective, existing graph-search or MCTS schemes are analogous to an open-loop method, e.g., motion planning or trajectory optimization \citep{betts1998survey}, which does not involve interactions with the external environment. To integrate them into a closed-loop approach, e.g., model predictive control \citep{rawlings2000tutorial}, one has to specify how to act given the planned trajectory and when to reinvoke the reasoning (learning and planning) routine, which is the key technique of \texttt{RAFA}. Another recent line of work tackles more complex tasks by allowing LLMs to access various additional modules, e.g., tools, programs, and other learning algorithms \citep{ahn2022can, shen2023hugginggpt, lu2023chameleon, liu2023llm+, cai2023large}, or by finetuning LLMs on the feedback \citep{zelikman2022star, li2022language, paul2023refiner,sun2023reinforcement}. 

\paragraph{Acting (and Reasoning) with LLM.} We build on a recent line of work that develops various closed-loop frameworks for interacting with the external environment. ``Inner monologue'' \citep{huang2022inner} and  ``ReAct'' \citep{yao2022react} combine reasoning and acting to refine each other for the first time. In comparison, \texttt{RAFA} provides a specific schedule for orchestrating reasoning and acting (as discussed above). As generalizations, ``Reflexion'' \citep{shinn2023reflexion} enables autonomous LLM agents to revise the current action of a pregenerated trajectory by learning from feedback, especially when they make mistakes. See also \cite{kim2023language}. However, making a local revision to the pre-generated trajectory is myopic because it fails to consider the long-term consequences of actions. Consequently, the obtained policy may get trapped by a local optimum. From an RL perspective, ``Reflexion'' \citep{shinn2023reflexion} is an oversimplified version of \texttt{RAFA}, where the planning subroutine revises the current action to maximize the reward function (``reason for now'') instead of planning multiple future steps to maximize the value function (``reason for future''), which measures the expected cumulative future reward. To remedy this issue, ``AdaPlanner'' \citep{sun2023adaplanner} regenerates the whole trajectory at each step, which yields a global improvement. See also \cite{wang2023describe}. However, the reasoning routine of ``AdaPlanner'' requires a handcrafted set of programs to reject suboptimal candidate trajectories. Without the domain knowledge of the current task, the regenerated trajectory is not necessarily optimal, i.e., maximizing the value function (up to a certain error). In contrast, the reasoning routine of \texttt{RAFA} is designed following the principled approach in RL. In particular, the learning subroutine infers the transition and reward models (model) or/and the value function (critic), while the planning subroutine emulates the value iteration algorithm, the random shooting algorithm, or the MCTS algorithm, none of which use any domain knowledge. \texttt{RAFA} also achieves provable sample efficiency guarantees for the first time and outperforms those existing frameworks empirically. 

\begin{table}[]
    \centering
 \begin{tabular}{l|ccc}
\toprule
 \makecell{Closed-Loop Mechanisms}& \makecell{No Parameter  Update} & \makecell{Theoretical  Guarantee}\\
\midrule
\texttt{RAFA}  & \cmark & \cmark \\
Model-Based Deep RL  & \xmark &  \cmark  \\
Model Predictive Control & \xmark &  \cmark  \\
Thompson Sampling & \xmark &  \cmark  \\
{``React'', ``Reflexion'', and ``Adaplanner''} &  \cmark  &  \xmark \\
\bottomrule
\end{tabular}
    \caption{{Comparison between \texttt{RAFA} and other mechanisms.}}
    \label{tab:com}\vspace{-0.5cm}
\end{table}
\paragraph{Large Language Model (LLM) and In-Context Learning (ICL). } 
LLMs \citep{radford2019language, brown2020language, hoffmann2022training, chowdhery2022palm, openai2023gpt4, touvron2023llama} display notable reasoning abilities. A pivotal aspect of reasoning is the ICL ability \citep{liang2022holistic, razeghi2022impact, shin2022effect, olsson2022context, akyurek2022learning, kirsch2022general, garg2022can, von2023transformers, li2023transformers, abernethy2023mechanism}, which allows LLMs to solve a broad range of tasks with only a few in-context examples instead of finetuning parameters on a specific dataset. We focus on harnessing the ICL ability of LLMs to optimize actions in the real world, which is crucial to autonomous LLM agents. In particular, we build on a recent line of work \citep{xie2021explanation, zhang2022analysis, zhang2023and, wang2023large, wies2023learnability, jiang2023latent,lee2023supervised} that attributes the ICL ability to implicit Bayesian inference, i.e., an implicit mechanism that enables LLMs to infer a latent concept from those in-context examples, which is verified both theoretically and empirically. In \texttt{RAFA}, the latent concept is the transition and reward models (model) of the unknown environment or/and the value function (critic), which is inferred from the memory buffer in the learning subroutine. Claim \ref{claim:llm} can also be considered as a result of ICL ability.  

\paragraph{Reinforcement Learning (RL) under a Bayesian Framework. }
We build on a recent line of work on the infinite-horizon \citep{abbasi2015bayesian, dong2019q, wei2020model, zhou2021nearly, zhou2021provably, chen2022sample,chua2018deep,hafner2019learning,sekar2020planning} and Bayesian \citep{strens2000bayesian, osband2013more, russo2014learningx, russo2014learning, russo2016information,lu2019information} settings of RL, which include model-based deep RL \citep{janner2019trust,NEURIPS2023_4640d5da,Wang_Wang_Zhou_Li_Li_2022,liu2024maximize}, model predictive control \citep{morari1999model}, and Thompson sampling \citep{russo2014learning}. The infinite-horizon setting allows \texttt{RAFA} to interact with the external environment continuously without resetting to an initial state, while the Bayesian setting allows us to connect \texttt{RAFA} with BMA and establish the theoretical guarantee. RL operates in a numerical system, where rewards and transitions are defined by scalars and probabilities, and trains actors and critics on the collected feedback iteratively. We focus on emulating the actor-model or actor-critic update in RL through an internal mechanism of reasoning on top of LLMs, which allows data and actions to be tokens in a linguistic system while bypassing the explicit update of parameters in model-based RL \citep{chua2018deep,hafner2019learning,sekar2020planning,liu2022learning,zhong2022gec,zheng2022optimistic,liu2022welfare}. 
In particular, the learning and planning subroutines of \texttt{RAFA} emulate the posterior update and various planning algorithms in RL. Moreover, \texttt{RAFA} orchestrates reasoning (learning and planning) and acting following the principled approach in RL, i.e., (re)planning a future trajectory over a long horizon (``reason for future'') at the new state and taking the initial action of the planned trajectory (``act for now''). As a result, \texttt{RAFA} inherits provable sample efficiency guarantees from RL. We summarize the comparison between \texttt{RAFA} and other closed-loop mechanisms in Table~\ref{tab:com}.

\subsection{Notations}
We provide a table of notations in Appendix \ref{app:not}.

\section{Bridging LLM and RL}\label{sec: pre}

\label{sec:rl_llm}
\paragraph{Interaction Protocol.} We use Bayesian adaptive Markov decision processes (MDPs) \citep{ghavamzadeh2015bayesian} to model how autonomous LLM agents interact with the external environment. We consider an infinite-horizon MDP $M = (\mathcal{S}, \mathcal{A}, P, r, \rho, \gamma,\mathbb{P}_0)$, where $\mathcal{S}$ is the state space, $\mathcal{A}$ is the action space, $P:\mathcal{S}\times\mathcal{A}\mapsto \Delta(\mathcal{S})$ is the transition kernel, $r:\mathcal{S}\times\mathcal{A}\mapsto\mathbb{R}$ is the reward function, $\rho$ is the initial distribution of states, $\gamma\in(0,1)$ is the discount factor, and $\mathbb{P}_0$ is the prior distribution of the transition kernel and the reward function. Here, $P$ gives the probability distribution of the next state given the current state and action, while $r$ is assumed to be deterministic without loss of generality. For notational simplicity, we parameterize $P$ and $r$ by a shared parameter $\theta\in \Theta$ and denote them as $P_{\theta}$ and $r_{\theta}$. At the beginning of the interaction, the data-generating parameter $\theta^\star$ is sampled from the prior $\mathbb{P}_0$.  At the $t$-th step during the interaction, the LLM agent receives a state $s_t \in \cS$, takes an action $a_t \in \cA$ following the current policy $\pi_t: \cS \mapsto \cA $, and receives a reward $r_t = r_{\theta^\star}(s_t, a_t)$. Subsequently, the external environment transits to the next state $s_{t+1} \sim P_{\theta^\star}(\cdot \mid s_{t}, a_{t})$, while the LLM agent computes the updated policy $\pi_{t+1}$ through an internal mechanism of reasoning (as discussed below). Note that $\cS$ and $\cA$ are represented by tokens in a linguistic system. Here, $\pi \in \Pi$ is assumed to be deterministic without loss of generality, where $\Pi$ is the feasible set of policies.

\paragraph{Value Function.} For a policy $\pi$ and a parameter $\theta$ of the transition and reward models, we define the state-value and action-value functions as
\begin{align}
    V^\pi_\theta(s) &= \mathbb{E}\Bigl[\sum_{t=0}^\infty \gamma ^t r_\theta(s_t,a_t)\Big|\,s_0 = s\Bigl],\notag\\ 
    Q_\theta^\pi(s,a)&=\mathbb{E}\Bigl[\sum_{t=0}^\infty \gamma ^t r_\theta(s_t,a_t)\Big|\,s_0 = s,a_0 = a\Bigr],
\end{align}
where $\mathbb{E}$ is taken with respect to $a_t = \pi(s_t)$ and $s_{t+1}\sim P_{\theta}(\cdot \mid s_t, a_t)$ for all $t\geq0$. In other words, $V^\pi_\theta$ (and $Q_\theta^\pi$) gives the expected cumulative future reward from the current state $s$ (and action $a$). 

We define the optimal policy $\pi_\theta^\star$ with respect to a given parameter $\theta$ as $\pi_{\theta}^\star=\operatornamewithlimits{argmax}_{a\in\cA}Q_\theta^\star$, where $Q_\theta^\star$ is the fixed point of the following Bellman optimality equation,
\begin{align}
        Q^\star_\theta\left(s, a\right)&=\left(B_\theta V_{\theta}^\star\right)\left(s, a\right),\notag\\ V_\theta^\star(s) &= \max_{a\in\mathcal{A}}Q_\theta^\star(s,a),\label{eq:bellman_opt}
\end{align} where 
$Q^\star_\theta$ and $V_\theta^\star$ are the fixed-point solutions. 
 Here, we define $(B_\theta V)(s, a)= r_{\theta}(s,a)+\gamma\cdot (P_{\theta^\star} V)(s,a)$ and $(P_{\theta}V)(s,a)=\mathbb{E}_{s^\prime\sim P_{\theta}(\cdot\mid s,a)}[V(s^\prime)]$ for any value function $V$. See \citet{sutton2018reinforcement} for the existence and uniqueness guarantees for $Q^\star_\theta$, $V_\theta^\star$, and $\pi_\theta^\star$.

\paragraph{Posterior, Entropy, and Information Gain.} By Bayes' rule, the posterior of $\theta^\star$ given any in-context dataset $\cD$ is 
\begin{align}
    \mathbb{P}_{\text{post}}(\theta\mid \cD) \propto \mathbb{P}_0(\theta) L(\cD\mid \theta),\label{eq:post_theta}
\end{align}
where we denote by $L(\cD\mid \theta)$ the likelihood of $\cD$ conditioned on $\theta$. 
We define the random variable $\xi_{(s,a)}$ as the pair of the next state and the current reward $(s^\prime,r)$ given the query state-action pair $(s,a)$. Given any in-context dataset $\mathcal{D}$ and query state-action pair $(s,a)$, the posterior of $\xi_{(s,a)}$ can be specified as
\begin{align}
    &\mathbb{P}_{\text{post}}(\xi_{(s,a)}\mid \cD,s,a) =\mathbb{E}_{\theta \sim \mathbb{P}_{\text{post}}(\cdot\mid \cD)}\bigl[P_\theta(s^\prime\mid s, a)\cdot \textbf{1}(r =r_\theta(s,a))\bigr], \label{eq:posterior}
\end{align}
where we use Bayes' rule and the fact that the query state-action pair $(s,a)$ is conditionally independent of $\theta^\star$ given $\cD$. 
To characterize the uncertainty of $\theta^\star$ conditioned on $\cD$, we define the posterior entropy $H(\theta\mid \cD)$ as
\begin{align}
    H(\theta\mid \cD) = \mathbb{E}_{\theta\sim \mathbb{P}_{\text{post}}(\cdot\mid \cD)} \bigl[-\log \bigl(p_{\text{post}}(\theta\mid \cD)\bigr)\bigr],\label{eq:ent_def}
\end{align}
where $p_{\text{post}}$ is the probability mass (or density) function of $\mathbb{P}_{\text{post}}$. 
High posterior entropy $H(\theta\mid \cD)$ means high uncertainty of $\theta^\star$, which suggests that it is hard for the agent to make a precise prediction given $\cD$.  
We also define the information gain $I(\theta;\xi\mid\cD)$ as $H(\theta\mid\mathcal{D})-H(\theta\mid \mathcal{D},\xi)$, which characterizes how much information $\xi_{(s,a)}$ carries to reduce the uncertainty of $\theta^\star$ conditioned on $\cD$.

\paragraph{Sample Efficiency.} 
As the performance metric, we define the Bayesian regret after $T$ steps of interactions, 
\begin{align}\label{eq:optbellmaneq}
    \mathfrak{R}(T) =\mathbb{E}\Bigl[ \sum_{t=0}^{T-1} V_{\theta^\star}^{\pi^\star} (s_t) - V_{\theta^\star}^{\pi_t}(s_t)\Bigr], 
\end{align}
where $\pi^\star = \pi_{\theta^\star}^\star$, $\mathbb{E}$ is taken with respect to the prior distribution $\mathbb{P}_0$ of $\theta^\star$, the stochastic outcome of $s_t$, and the iterative update of $\pi_t$, which involves states, actions, and rewards until the $t$-th step, i.e., the full history $\mathcal{D}_t=\{(s_i,a_i,s_{i+1},r_i)\}_{i=0}^{t-1}$. We aim to design a sample-efficient agent that satisfies $\mathfrak{R}(T) = o(T)$, i.e., the Bayesian regret is sublinear in the total number of interactions $T$.

\paragraph{What Reasoning Means and Role of LLM.} To bridge LLM mechanisms with online RL algorithms, we claim that LLMs can play a similar role of model estimators in the design of online RL algorithms for interactions, which is one aspect of In-Context Learning (ICL) ability  of LLMs. 
\begin{claim}
    {LLMs provide a more accurate estimate for the environment with more feedback from online interactions.}\label{claim:llm}
\end{claim}
In Proposition \ref{prop:llm_bma} in Section \ref{sec:theory}, we prove that LLMs with posterior alignment perform Bayesian model averaging (BMA). This theoretical result supports Claim \ref{claim:llm}, as the estimation uncertainty of BMA is reduced given more feedback from  interactions with the environment \citep{wasserman2000bayesian}.
We also provide empirical evidence on three tasks for Claim \ref{claim:llm} as follows. 

\begin{figure}
    \centering
    \subfloat[Information bandit.]{ \includegraphics[width=0.3\linewidth]{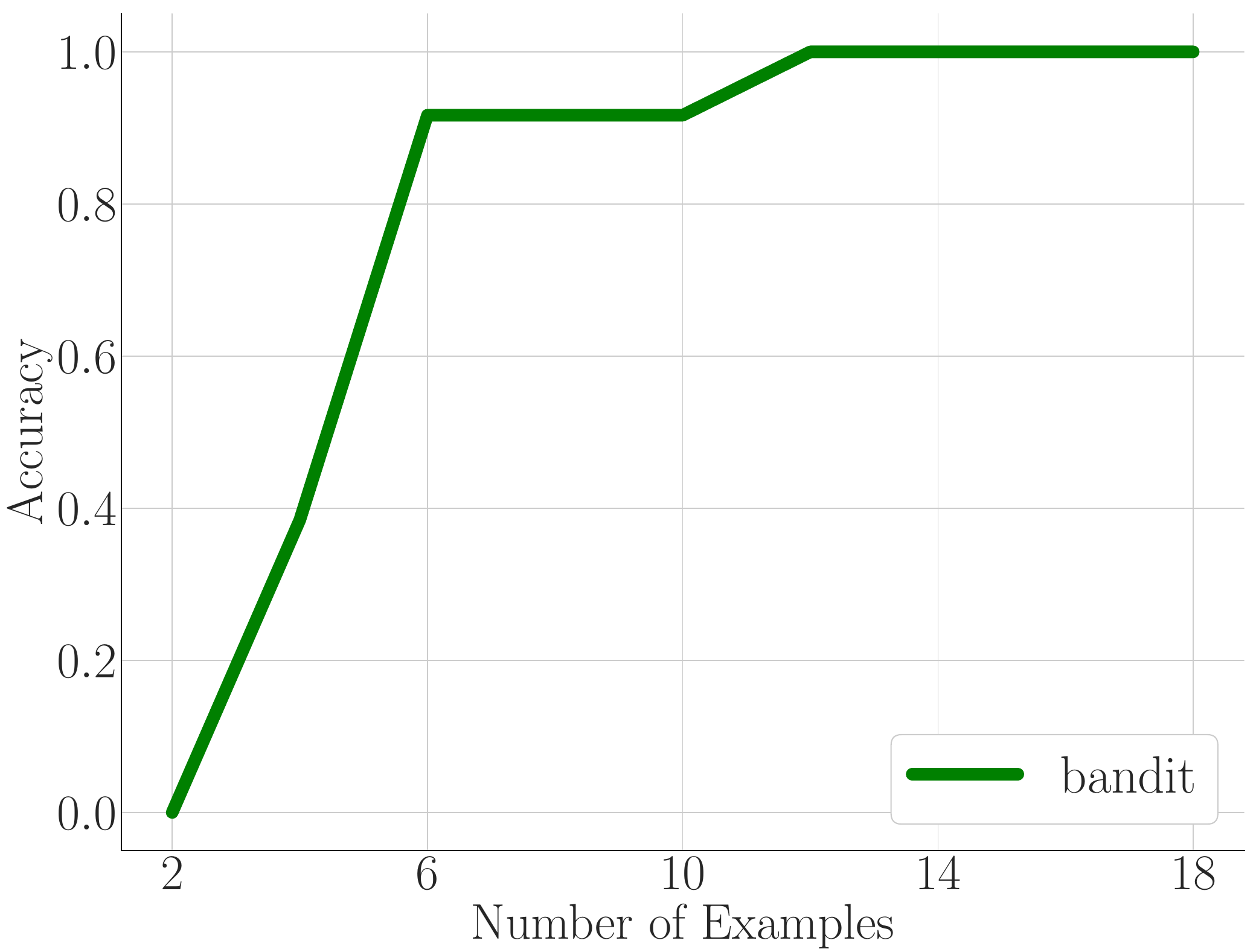}}
    \hfill
    \subfloat[Concept learning.]{ \includegraphics[width=0.3\linewidth]{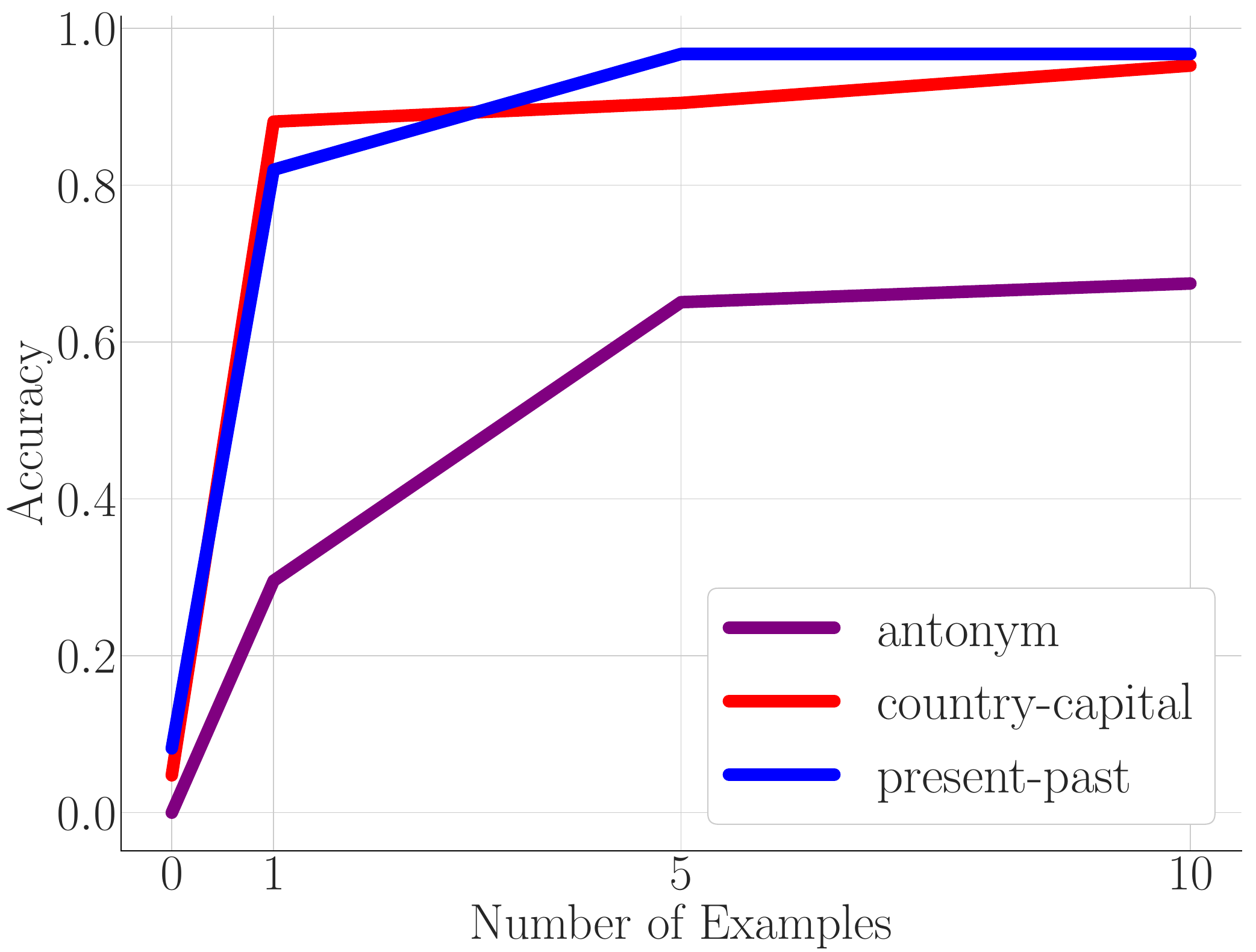}}
    \hfill
    \subfloat[Function value prediction.]{ \includegraphics[width=0.3\linewidth]{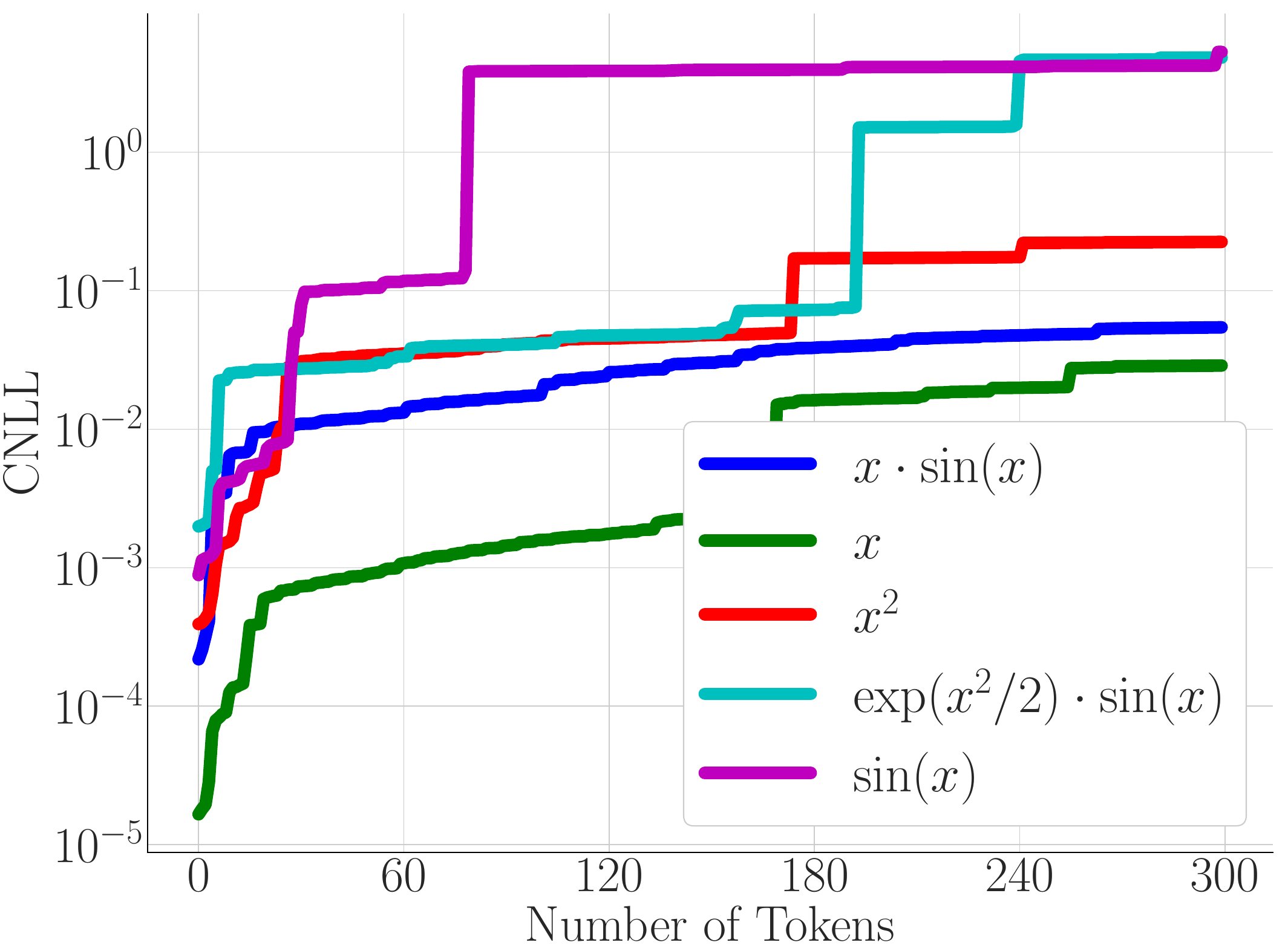}}
    \caption{Empirical evidences for Claim~\ref{claim:llm} on different tasks. LLMs demonstrate improving prediction abilities as the number of in-context samples grows.}
    \label{fig:claim}
\end{figure}
\noindent\textit{(a) Information Bandit.} The goal of our $100$-arm bandit experiment is to find the arm with the highest reward. There is an informative arm whose reward is the index of the best arm. We prompt the LLM (\texttt{gpt-4}) to pull the arm by providing it with the historical data of several bandit instances that share the same informative arm. It can be observed from the right figure that the LLM can learn the best arm with only $6$ examples, and is thus an effective reward estimator.

\noindent\textit{(b) Concept Learning.} 
We evaluate LLMs (\texttt{Llama 2-7B}) in three tasks \citep{todd2023function} with hidden concepts: (1) Antonym: Generate the word with the opposite meaning given an input word; 
(2) Country-Capital: Generate the capital city of a given country; and (3) Present-Past: Generate the verb’s past inflection given a verb in the present tense. We observe that with more in-context examples provided to the LLM, the accuracy of the test instance monotonically increases, indicating that the hidden concepts of the tasks are learned.

\noindent\textit{(c) Function Value Prediction.} 
The goal of this experiment is to let the LLM (\texttt{gpt-3})  predict the values of a function on unseen data points given the values on the points with fixed intervals.  Following \citet{gruver2023large}, we report the $t$-interval cumulative negative log-likelihood $\text{CNLL} = -\sum_i^t \log P(v_i | \text{prompt}_{i-1})$, where $v_i$ is the value of the function at data point $i$. It can be observed that the LLMs are good time series forecasters.

Under Claim \ref{claim:llm}, we establish the correspondence between LLMs and RL by using LLMs as model estimators in RL algorithms, which opens the door to creating a practical algorithm that combines the strengths of both LLMs and RL. LLMs excel in accuracy with minimal feedback, which improves the sample efficiency. LLMs can also refine estimates using new feedback as prompts, which avoids explicit parameter updates. RL algorithms benefit from online interaction to improve estimates and policies and have theoretical guarantees with optimal planning algorithms like value iteration. This LLM-RL correspondence inspires us to introduce a new framework in the next section, aiming to orchestrate the reasoning (learning and planning) and acting of LLMs.

\section{Algorithm}\label{sec:mech}\vspace{-0.1cm}

\begin{algorithm}[h]\small
	\caption {Reason for future, act for now (\texttt{RAFA}): The LLM version.}
	\begin{algorithmic}[1]	\label{alg: llm}
	\STATE \textbf{input}:  An LLM learner-planner \texttt{LLM-LR-PL}, which aims at generating an optimal trajectory given an initial state and returns the initial action (e.g., Algorithm \ref{alg: example}), and a switching condition \texttt{If-Switch}.
 \STATE \textbf{initialization}: Sample the initial state $s_0 \sim \rho$, set $t = 0$, and initialize the memory buffer $\mathcal{D}_0 = \emptyset$.
    \FOR{$k=0,1,\ldots,$}
    \STATE Set $t_k \leftarrow{t}$. 
    \REPEAT
    \STATE Learn and plan given memory $\mathcal{D}_{t_k}$ to get action $a_t\leftarrow \texttt{LLM-LR-PL}(\mathcal{D}_{t_k},s_t)$.\hfill (``reason for future'')
        \STATE  Execute action $a_t$ to receive reward $r_t$ and state $s_{t+1}$ from environment. \hfill (``act for now'')
        \STATE Update memory $\mathcal{D}_{t+1} \leftarrow \mathcal{D}_{t} \cup \{(s_t,a_t,s_{t+1},r_t)\}$. 
        \STATE Set $t\leftarrow t+1$.
        \UNTIL $\texttt{If-Switch}(\mathcal{D}_t)$ is \texttt{True}. \hfill (the switching condition is satisfied)
    \ENDFOR
	\end{algorithmic}
\end{algorithm}

\noindent\textbf{Architecture of \texttt{RAFA}.} By leveraging the LLM-RL correspondence in Section \ref{sec:rl_llm}, we provide a principled framework for orchestrating reasoning and acting, namely ``reason for future, act for now'' (\texttt{RAFA}), in Algorithms \ref{alg: llm} and \ref{alg: example}. 
At the $t$-th step of Algorithm \ref{alg: llm}, the LLM agent invokes the reasoning routine, which learns from the memory buffer and plans a future trajectory over a long horizon (``reason for future'' in Line 6), takes the initial action of the planned trajectory (``act for now'' in Line 7), and stores the collected feedback (state, action, and reward) in the memory buffer (Line 8). Upon the state transition of the external environment, the LLM agent reinvokes the reasoning routine to replan another future trajectory from the new state (Line 6 following Line 9). To ensure the learning and planning stability, we impose the switching condition (Line 10) to decide whether to incorporate the newest chunk of history, i.e., the set difference $\cD_t - \cD_{t_k}$, into the information state, which is used in the reasoning routine as contexts. In other words, the reasoning routine uses the same history $\mathcal{D}_{t_k}$ for all $t_k \leq t < t_{k+1}$ until the $(k+1)$-th switch at the $(t_{k+1}-1)$-th step, which guarantees that the posterior distribution and the optimized action or the corresponding policy are updated in a conservative manner. We specify the switching condition in Sections \ref{sec: experiments} and \ref{sec:theory}.

\begin{figure*}[t]
\vspace{-0.02cm}
  \begin{minipage}[c]{0.77\textwidth}
    \includegraphics[width=\textwidth]{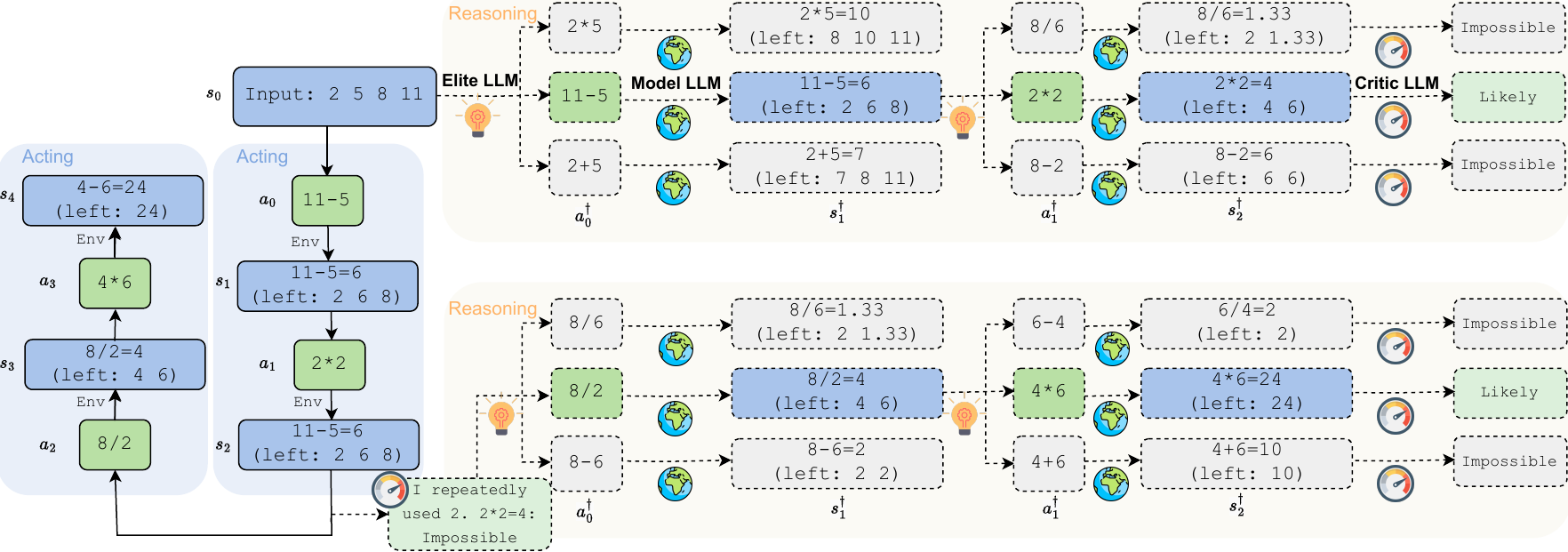}
  \end{minipage}\hfill
  \begin{minipage}[c]{0.21\textwidth}
    \caption{\small \texttt{RAFA} for Game~of 24. Actions are proposed (dotted) and selected (\textcolor{DarkGreen}{green}). Hallucinations that the same number can be reused are mitigated through interactions.} \label{fig:game24_illustration}
  \end{minipage}
\vskip+0.02in
\end{figure*}
\noindent\textbf{``Reason for Future'' (Line 6 in Algorithm \ref{alg: llm} and Lines 3-11 in Algorithm \ref{alg: example}).} As detailed below, the reasoning routine composes the learning and planning subroutines to map the full history $\cD_{t_k}$ (until the $t_k$-th step) to an optimized action $a_t$. Note that the reasoning routine does not interact with the external environment throughout the learning and planning subroutines. 

    \noindent$\bullet$ The learning subroutine (Lines 3-4 in Algorithm \ref{alg: example}) maps the memory buffer $\cD_{t_k}$ to a transition kernel (\texttt{Model}) and a value function (\texttt{Critic}), which are used in the planning subroutine. In practice, we prompt LLMs to form an estimate of the external environment based on the memory buffer. Here, the estimate is instantiated by two LLMs: \texttt{Model} and \texttt{Critic}, which estimate their ground-truth counterparts in association with the data-generating parameter. 
    Under Claim \ref{claim:llm}, the learning subroutine yields more accurate versions of \texttt{Model} and \texttt{Critic} when we prompt them with a growing collection of feedback from the external environment. 
    Consequently, the planning subroutine can use them to assess the long-term outcome of actions with a higher accuracy. Depending on whether we emulate the model-based or model-free approach of RL, we may choose to emulate \texttt{Model} or \texttt{Critic} individually.  Compared with the learning subroutine in RL,  we replace the parameterized function approximation (usually deep neural networks) with LLMs and use an ``in-context'' manner to update the LLMs, which eliminates the need for explicit parameter updates. Because LLMs are pretrained and undergo supervised fine-tuning, they provide much better estimates compared to learning from scratch, leading to an improvement in sample efficiency for online interactions. 

 \begin{algorithm}[h] \small
	\caption{The LLM learner-planner (\texttt{LLM-LR-PL}): A tree-search example. {(the deterministic case)}}
	\label{alg: example}
	\begin{algorithmic}[1]
	\STATE \textbf{input}: The memory buffer $\mathcal{D}$, the initial state $s$, the search breadth $B$, and the search depth $U$.
 \STATE \textbf{initialization}: Initialize the state array $\mathcal{S}_0 \leftarrow \{s\}$ and the action array $\mathcal{A}_0 \leftarrow \emptyset$.
\\--------------------------------------------- (the learning subroutine) --------------------------------------------
\STATE Set \texttt{Model} as an LLM instance prompted to use $\mathcal{D}$ as contexts to \textit{generate the next state}.
\STATE Set \texttt{Critic} as an LLM instance prompted to use $\mathcal{D}$ as contexts to \textit{estimate the value function}.
\\--------------------------------------------- (the planning subroutine) --------------------------------------------\vspace{-1pt}
\STATE Set \texttt{Elite} as an LLM instance prompted to use $\mathcal{D}$ as contexts to \textit{generate multiple candidate actions}.
     \FOR{$u = 0,\ldots,U$}
\STATE For each current state in $\mathcal{S}_{u}$, invoke \texttt{Elite} to generate $B$ candidate actions and store them in $\mathcal{A}_u$.
\STATE For each candidate action in $\mathcal{A}_{u}$, invoke \texttt{Model} to generate the next state and store it in $\mathcal{S}_{u+1}$.
     \ENDFOR
    \STATE For all resulting rollouts in $\mathcal{S}_0\times\mathcal{A}_0\times \cdots \times \mathcal{S}_U \times \mathcal{A}_U$, invoke \texttt{Critic} to evaluate the expected cumulative future reward and select the best one $(s_0^\dagger,a_0^\dagger,\ldots,s_U^\dagger, a_U^\dagger)$, where $s_0^\dagger = s$.
    \STATE \textbf{output}: The initial action $a_0^\dagger$ of the selected rollout.
	\end{algorithmic}
\end{algorithm}

     \noindent$\bullet$ The planning subroutine (Lines 5-11 in Algorithm \ref{alg: example}) maps \texttt{Model} and \texttt{Critic} to a future trajectory $(s_0^\dagger,a_0^\dagger,\ldots,s_U^\dagger, a_U^\dagger)$, where $s_0^\dagger$ is the current state $s_t$ and $a_0^\dagger$ is executed in the external environment as the current action $a_t$ during the acting phase. Intuitively, we prompt LLMs to generate an optimal policy (actor) for multiple future steps, which maximizes the value function (\texttt{Critic}). From an RL perspective (Sections \ref{sec:rl_llm} and \ref{sec:theory}), the planning subroutine approximately solves the Bellman equation \citep{sutton2018reinforcement}, where we solve the optimal policy  (or the corresponding action) given the estimated transition kernel and reward function (or critic) by LLMs. As two LLM instances from the learning subroutine, \texttt{Model} and \texttt{Critic} instantiate the estimated transition kernel and the estimated value function. Hence, we can simulate a given number of trajectories with \texttt{Model}, evaluate them with \texttt{Critic}, and obtain an improved policy, which is achieved by specially designed prompts instead of a numerical algorithm. By maximizing the expected cumulative future reward (instead of the immediate reward), the planning subroutine returns an optimized action that improves the long-term outcome. There are two error sources that affect the planning subroutine, namely the posterior uncertainty, which is inherited from \texttt{Model} and \texttt{Critic} due to the finite size of $\cD_{t_k}$, and the planning suboptimality, which is induced by the limited capacity for computation, e.g., the bounded width and depth of tree-search (Lines 6-9 in Algorithm \ref{alg: example}). Depending on the specific configuration of the state and action spaces (continuous versus discrete) and the transition and reward models (stochastic versus deterministic), we may choose to emulate the value iteration algorithm, the random shooting algorithm, or the Monte-Carlo tree-search algorithm. All of them allow \texttt{RAFA} to achieve provable sample efficiency guarantees as long as they satisfy a specific requirement of optimality (Definition \ref{def: epsoptplan}). For illustration, we emulate the tree-search algorithm and defer its stochastic variant to Appendix \ref{app:omitted_alg}.
    
\noindent\textbf{``Act for Now'' (Lines 7-10 in Algorithm \ref{alg: llm}).} At the current state $s_t$, the LLM agent executes the optimized action $a_t$ in the external environment, which is obtained from the reasoning routine. Specifically, we take the initial action $a_0^\dagger$ of the planned trajectory $(s_0^\dagger,a_0^\dagger,\ldots,s_U^\dagger, a_U^\dagger)$, where $s_0^\dagger = s_t$ and $a_0^\dagger = a_t$, and discard the remaining subset. At the next state $s_{t+1}$, the LLM agent replans another future trajectory $(s_0^\dagger,a_0^\dagger,\ldots,s_U^\dagger, a_U^\dagger)$ with $s_0^\dagger = s_{t+1}$ and $a_0^\dagger = a_{t+1}$. In other words, the acting phase follows a short-term subset of the long-term plan, which is regenerated at every new state. The LLM agent stores the collected feedback $(s_t,a_t,r_t,s_{t+1})$ in the memory buffer $\cD_t$ and queries a switching condition \texttt{If-Switch} to decide when to update the  memory buffer $\cD_{t_k}$, which is used in the reasoning routine as contexts for learning and planning. Intuitively, we incorporate the newest chunk of history $\cD_t - \cD_{t_k}$ to improve the current policy only when it carries significant novel information, e.g., when the LLM agent loses for the first time following a winning streak.

\begin{figure}[H]
    \includegraphics[width=\textwidth]{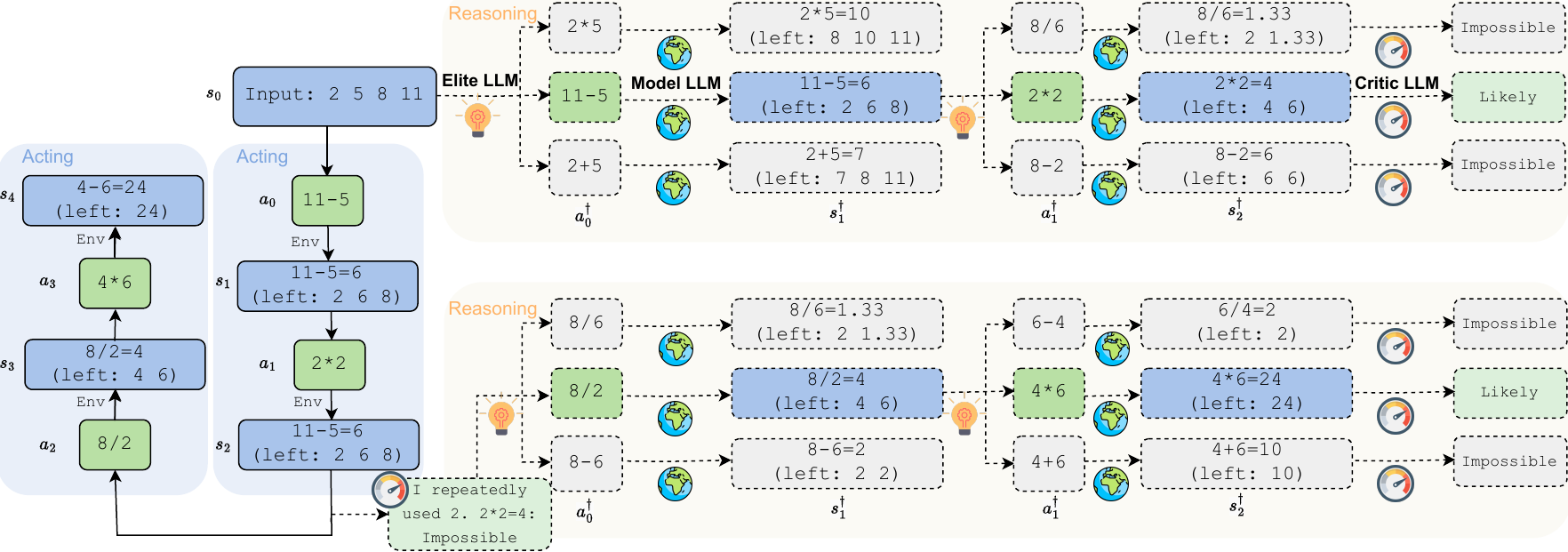}
    \caption{\texttt{RAFA} for Game~of 24. Actions are proposed (dotted) and selected (\textcolor{DarkGreen}{green}). Hallucinations that the same number can be reused are mitigated through interactions.} \label{fig:game24_illustration}
\end{figure}

\section{Experiment}\label{sec: experiments}
We evaluate \texttt{RAFA} in several text-based benchmarks, e.g., Game of 24, ALFWorld, BlocksWorld, and Tic-Tac-Toe. The detailed setups, results, and ablations are provided in Appendix~\ref{sec:add_exp}, while the detailed prompts are found in Appendix~\ref{sec:prompts}. We release all the codes on the page: \url{https://agentification.github.io/RAFA}. 

\subsection{Game of 24}
Game of 24 \citep{tree_of_thought} is a mathematical puzzle to obtain 24 from four natural numbers through basic arithmetic operations. The state is the  (possibly unfinished) current formula and the action is the next formula (or the modified part). 

\paragraph{Setup.} We emulate the tree-search algorithm to plan ($B\in \{1,2\}$). At the $t$-th step, \texttt{RAFA} learns from the memory buffer and switches to a new policy upon receiving an unexpected reward, which is the switching condition. After the $t$-th step, \texttt{RAFA} digests the collected feedback and generates a linguistic summary, which is saved into the memory buffer to avoid similar previous mistakes.

\begin{table}[H]
\centering
\begin{tabular}{l|ccccc}
\toprule
 & \texttt{RAFA} $(B=1)$ & \texttt{RAFA} $(B=2)$ & \texttt{ToT} $(B=1)$ & \texttt{ToT} $(B=2)$ & \texttt{Reflexion} \\
\midrule
\texttt{gpt-4} & 89\% & \textbf{93\%} & 73\% & 81\% &  21\% \\
\texttt{gpt-3.5} & 29\% & \textbf{46\%} &  10\% & 17\% &16\% \\
\bottomrule
\end{tabular}
\setlength{\belowcaptionskip}{-10pt}
\captionof{table}{\small Game of 24 results.}
\label{tab:game24_results}
\end{table}
\paragraph{Result.} \texttt{RAFA} attains SOTA performances as shown in Table~\ref{tab:game24_results}. \texttt{RAFA} achieves superior sample efficiency by mitigating hallucinations and avoid careless trials (Figures \ref{fig:game24_illustration} and \ref{fig:game24_sample_efficiency}). 

\begin{figure}[H]
\centering
\subfloat{
    \begin{minipage}[t]{0.47\linewidth}
        \centering
        \includegraphics[width=1\textwidth]{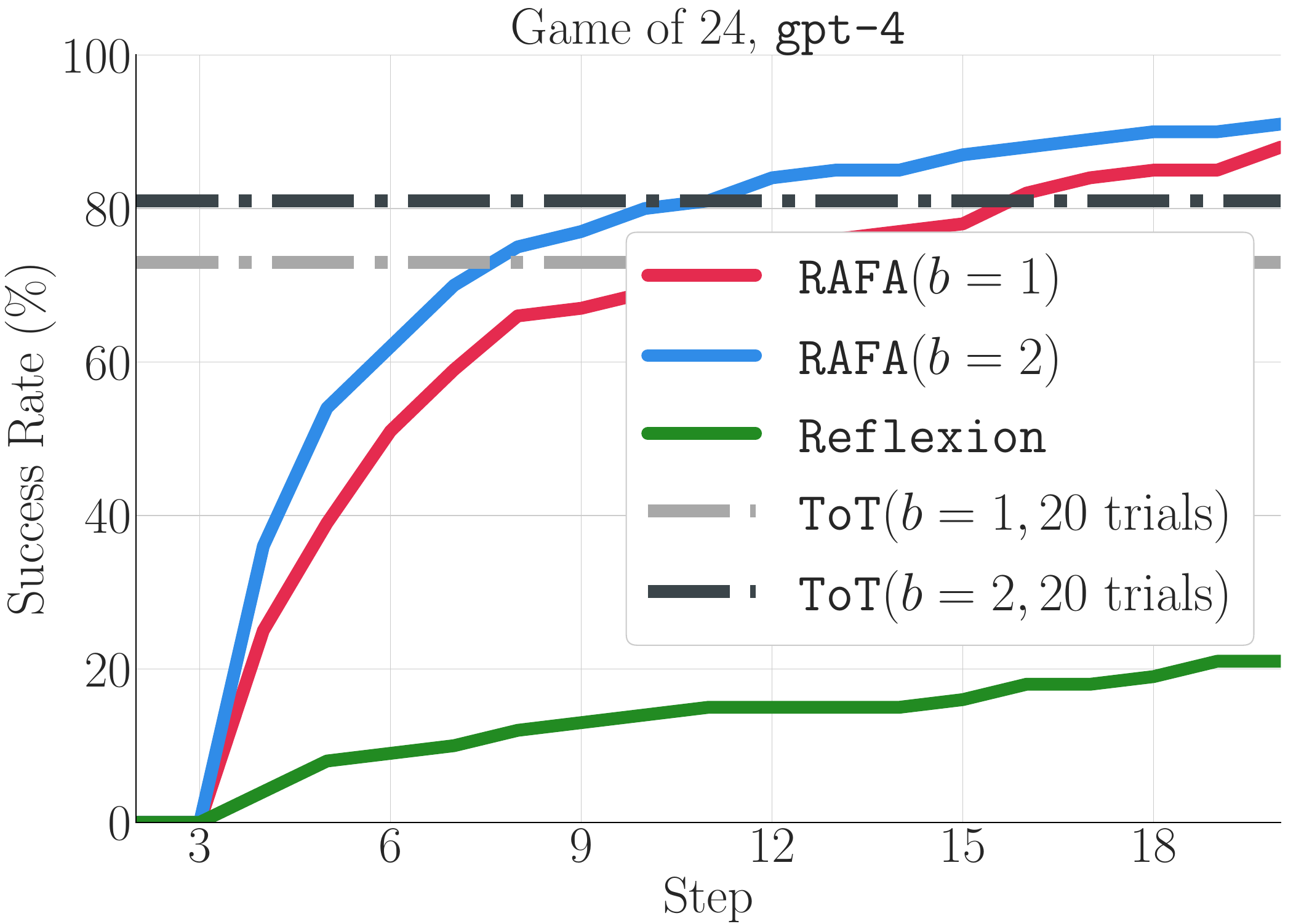}\\
    \end{minipage}
}
\subfloat{
    \begin{minipage}[t]{0.47\linewidth}
        \centering
        \includegraphics[width=1\textwidth]{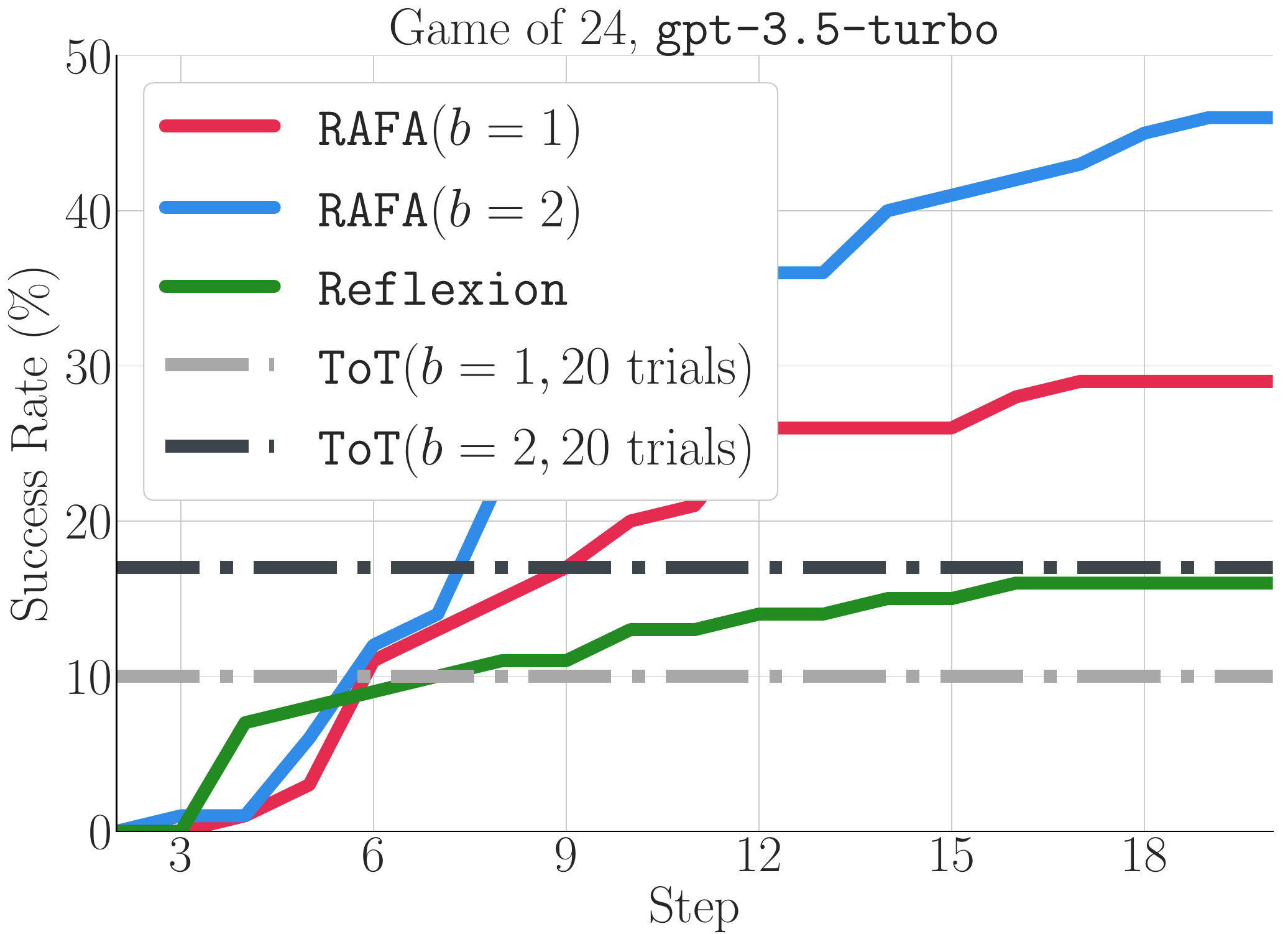}\\
    \end{minipage}
}
\setlength{\belowcaptionskip}{-10pt}
\caption{Sample efficiency on Game of 24.}
\label{fig:game24_sample_efficiency}
\end{figure}

\subsection{ALFWorld}
\begin{figure}[H]
\centering
\includegraphics[width=\textwidth]{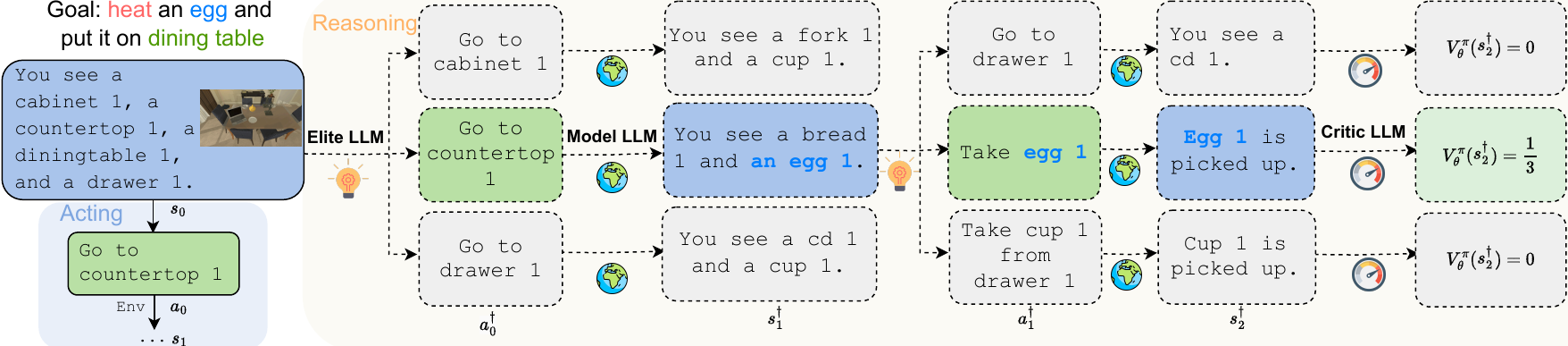}  
\caption{An illustration of \texttt{RAFA} in the ALFWorld environment.\vspace{-0.1cm}}
\label{fig:alf_demo}
\end{figure}

ALFWorld \citep{shridhar2020alfworld} is an interactive environment for embodied agent simulations, which encompasses $134$ household tasks in six overall categories (Table \ref{tab:alf}). We use \texttt{gpt-3} (\texttt{text-davinci-003}).

\begin{wrapfigure}{r}{0.5\textwidth}
  \centering
  \includegraphics[width=\linewidth]{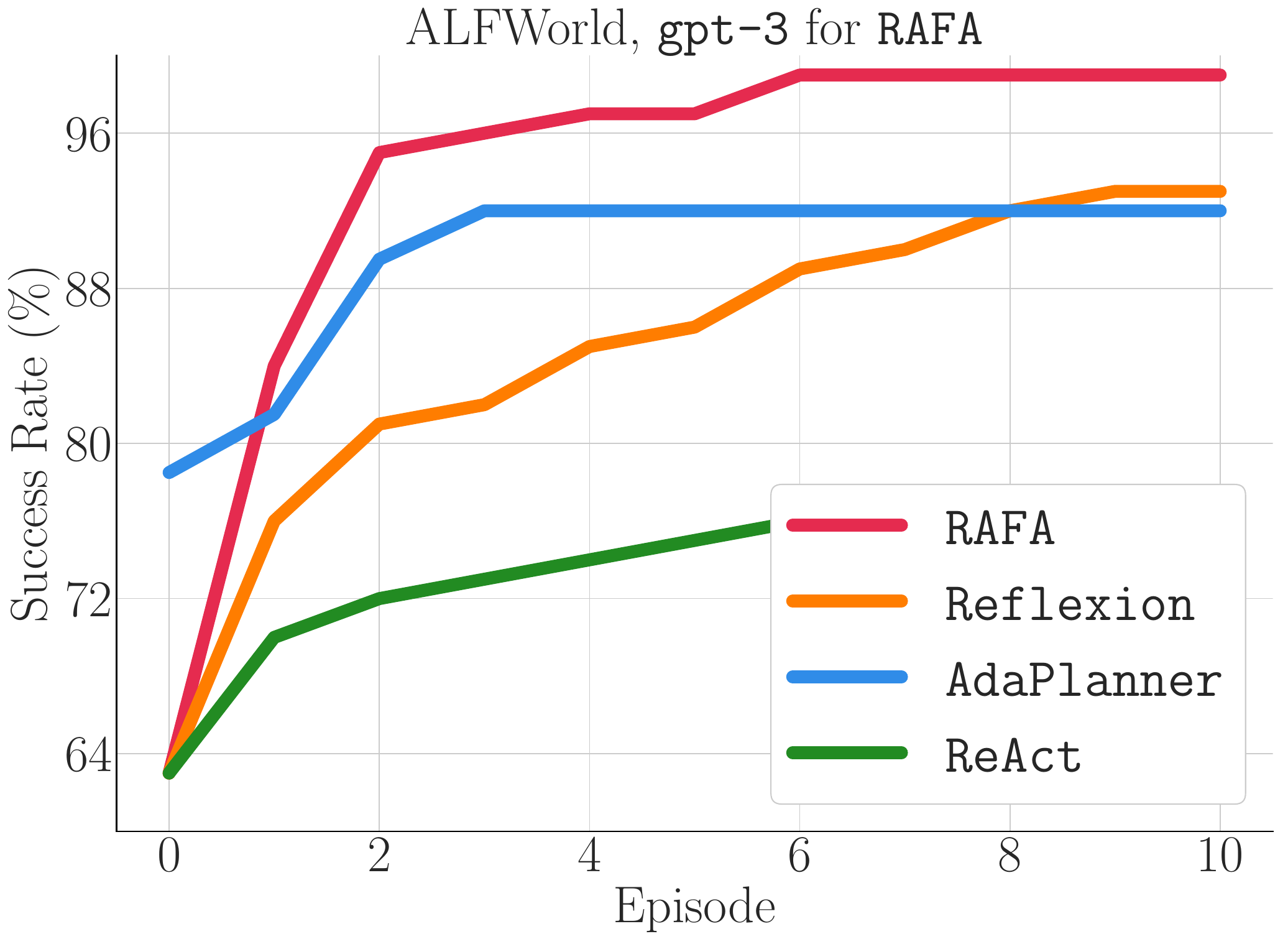}
  \caption{Sample efficiency on ALFWorld.}
  \label{fig:alf_curve}
\end{wrapfigure}
\paragraph{Setup.} We emulate the tree-search algorithm to plan ($B=2$). \texttt{RAFA} invokes \texttt{Critic} to evaluate the completed portion of the desired goal and switches to a new policy after $20$ consecutive failures.

\paragraph{Result.} \texttt{RAFA} outperforms various existing frameworks (Figure \ref{fig:alf_curve}). The better performance of \texttt{AdaPlanner} at the initial episode is attributed to a handcrafted set of programs for rejecting suboptimal candidate trajectories, which is challenging to construct without the domain knowledge of a specific task. One such example is the PickTwo category.

\begin{table}
\centering
\begin{tabularx}{\textwidth}{c|cccccc|c}
\toprule
           & Pick   & Clean & Heat   & Cool   & Examine & PickTwo & \hspace{0.6cm} Total \\ \midrule
\texttt{BUTLER}     & 46.00  & 39.00 & 74.00  & \textbf{100.00} & 22.00   & 24.00   & \hspace{0.6cm} 37.00 \\ 
\texttt{ReAct}  & 66.67  & 41.94 & 91.03  & 80.95  & 55.56   & 35.29   & \hspace{0.6cm} 61.94 \\ 
\texttt{AdaPlanner} & \textbf{100.00} & \textbf{96.77} & 95.65  & \textbf{100.00} & \textbf{100.00}  & 47.06   &  \hspace{0.6cm} 91.79 \\ 
\texttt{Reflexion}  & \textbf{100.00} & 90.32 & 82.61  & 90.48  & \textbf{100.00}   & 94.12   & \hspace{0.6cm} 92.54 \\ 
\texttt{RAFA}   & \textbf{100.00} & \textbf{96.77} & \textbf{100.00} & \textbf{100.00} & \textbf{100.00}  & \textbf{100.00}  & \hspace{0.6cm} \textbf{99.25} \\  \bottomrule
\end{tabularx}
\caption{ALFWorld results (success rates \%).}
\label{tab:alf}
\end{table}

\subsection{BlocksWorld}
BlocksWorld \citep{rap} contains tasks to arrange
blocks in specific configurations.

\paragraph{Setup.} We use the \texttt{Vicuna} \citep{zheng2023judging} model and emulate the MCTS algorithm to plan. 

\begin{figure}[H]
\centering
\includegraphics[width=0.9\textwidth]{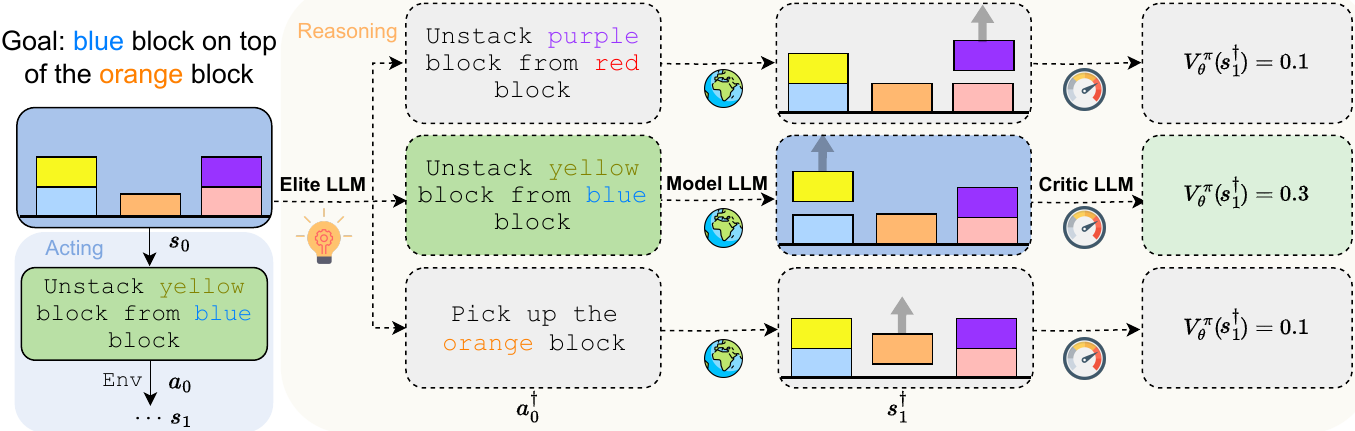}
\setlength{\belowcaptionskip}{-10pt}
\caption{\small \texttt{RAFA} for BlocksWorld.}
\label{fig:blocksworld-example}
\end{figure}

\paragraph{Result.} \texttt{RAFA} achieves superior success rates across multiple \texttt{Vicuna} versions (Figure \ref{fig:step4}). Comparisons with \texttt{CoT} and \texttt{RAP} demonstrate how the learning subroutine improves the planning optimality.  

\begin{figure}[H]
    \centering
    \subfloat{%
\includegraphics[width=0.4\linewidth]{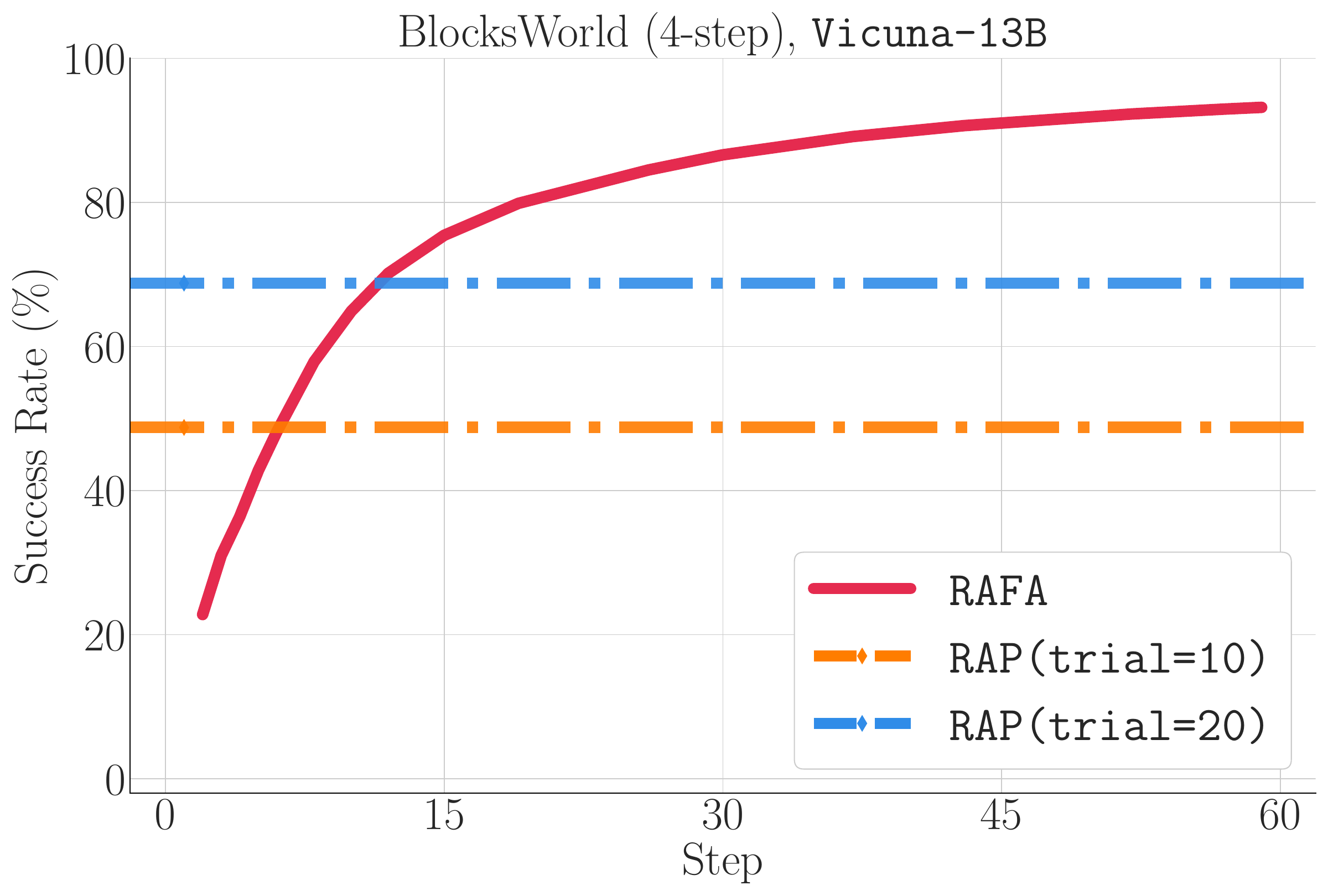}}
    \subfloat{%
        \includegraphics[width=0.4\linewidth]{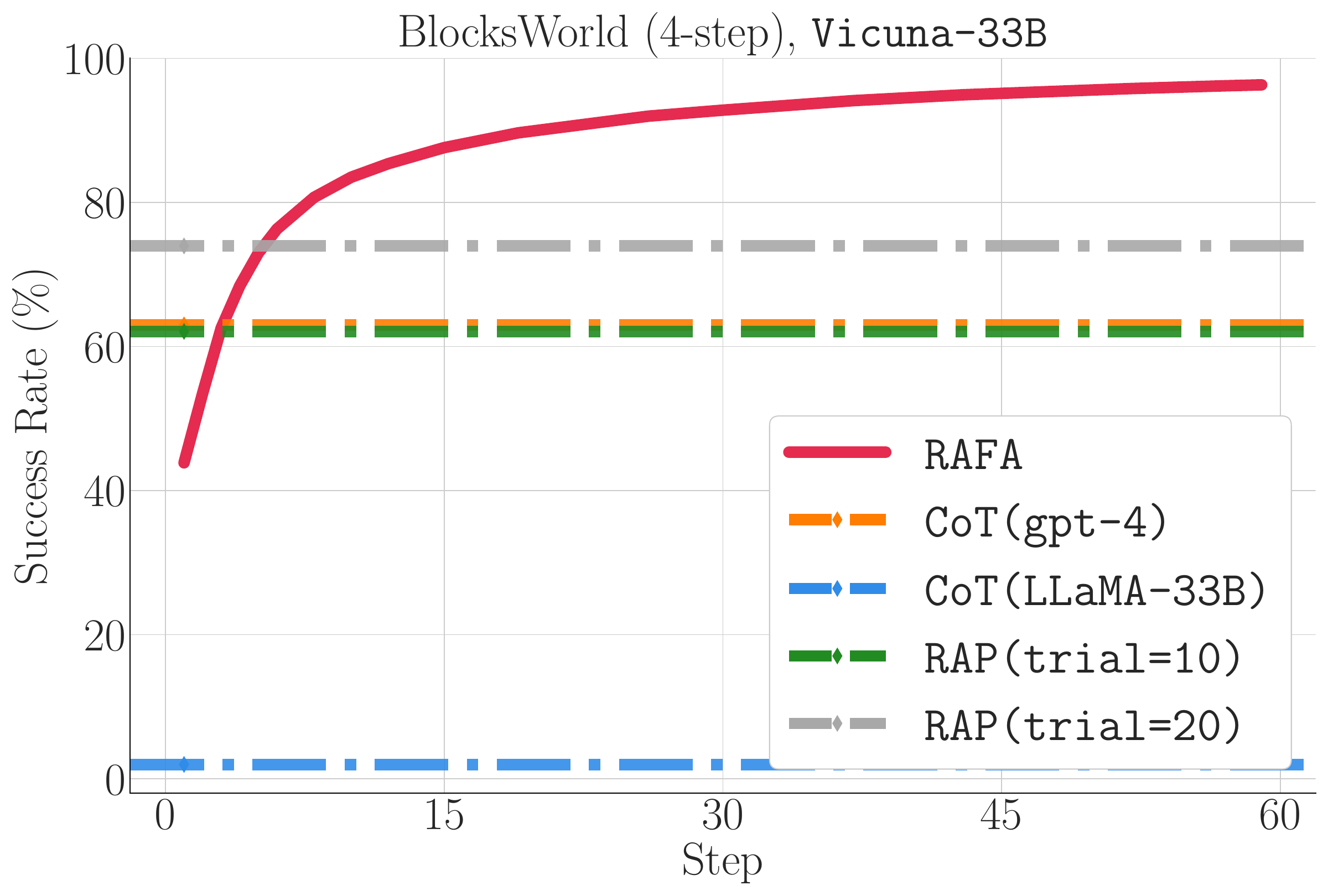}}\\
    \subfloat{%
       \includegraphics[width=0.4\linewidth]{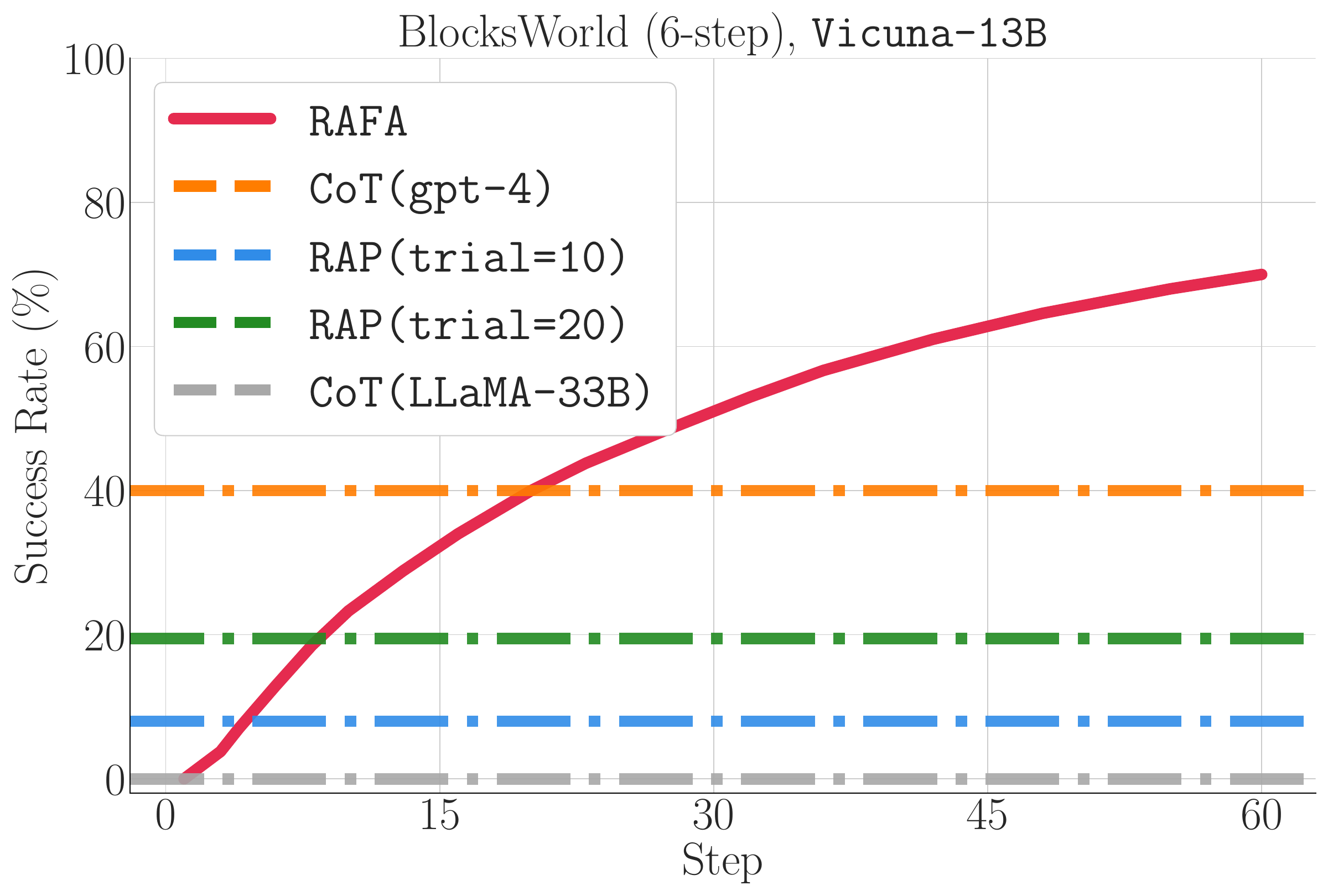}}
    \subfloat{%
        \includegraphics[width=0.4\linewidth]{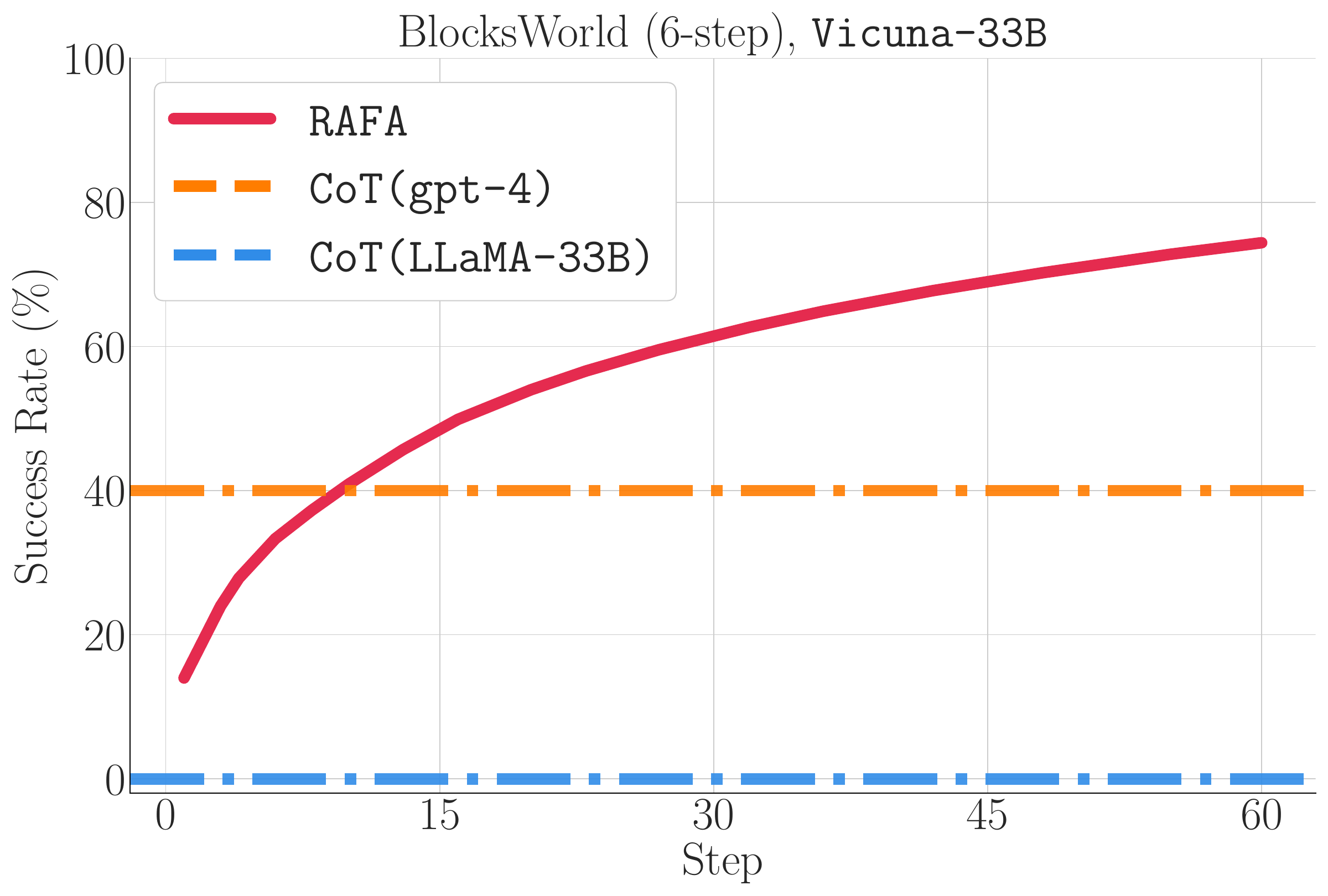}}
\setlength{\belowcaptionskip}{-10pt}
  \caption{\small Sample efficiency on BlocksWorld ($4$ and $6$ are the minimum numbers of steps for solving a specific task). \texttt{CoT} is prompted by four in-context examples.}
    \label{fig:step4}
\end{figure}

\subsection{Tic-Tac-Toe}

\begin{wrapfigure}{r}{0.5\textwidth}
  \centering
  \includegraphics[width=\linewidth]{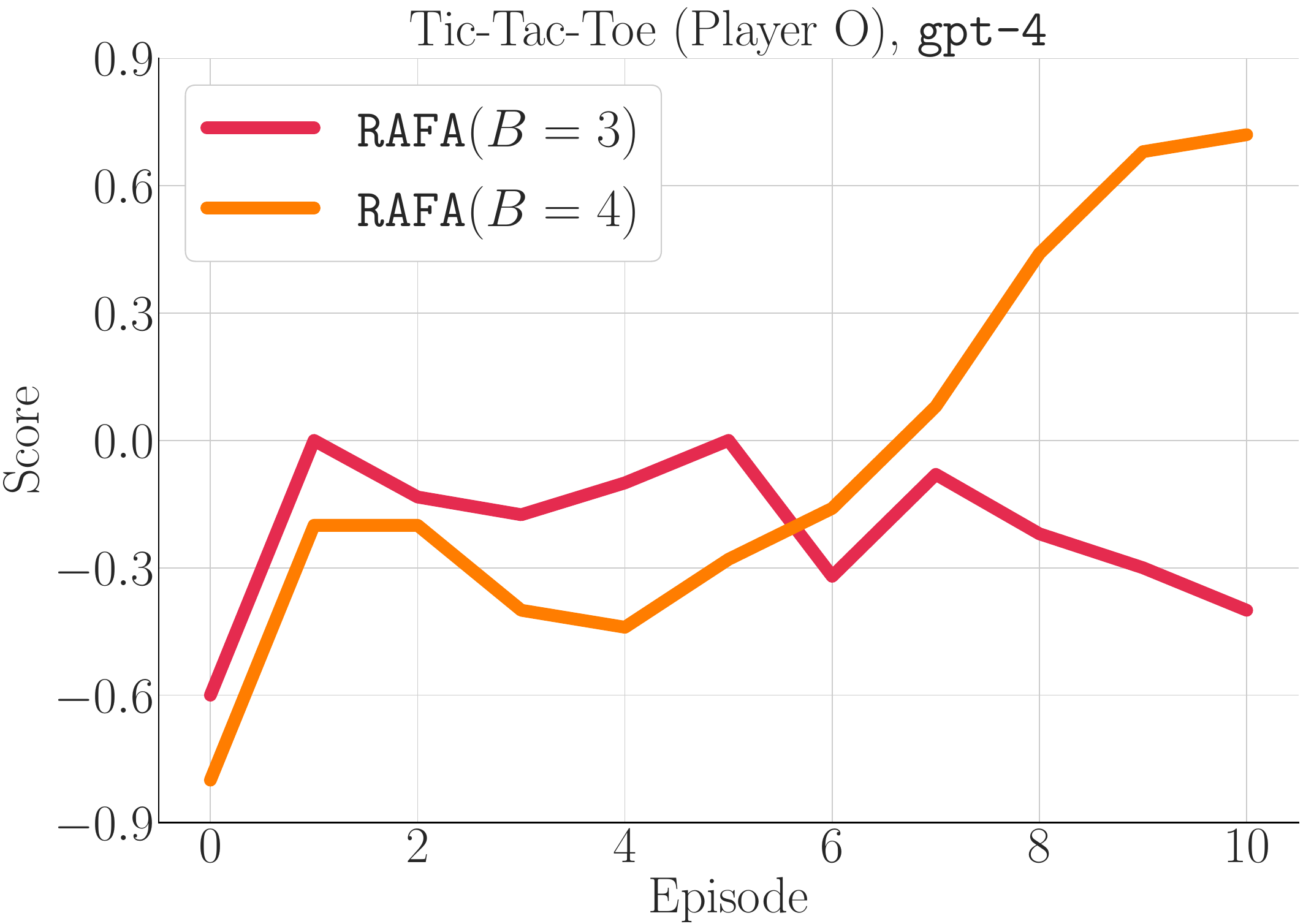}
  \setlength{\belowcaptionskip}{-10pt}
  \vspace{-0.6cm}
    \caption{Sample efficiency on Tic-Tac-Toe ($0$ means tie).}
      \vspace{-1.2cm}
    \label{fig:tictactoe_results}
\end{wrapfigure}
Tic-Tac-Toe \citep{beck2008combinatorial} is a competitive game where the X and O sides take turns to place~marks. \texttt{RAFA} invokes  \texttt{Model} to simulate the transition and opponent dynamics (Figure \ref{fig:tictactoe_rafa}).

\paragraph{Setup.}
We use \texttt{gpt-4} and emulate the tree-search algorithm to plan ($B\in \{3,4\}$). \texttt{RAFA} switches to a new policy when (a) the predicted state differs from the observed one, (2) the predicted action of opponents differs from the observed one, or (3) \texttt{Critic} gives the wrong prediction of the game status. Here, X has an asymmetric advantage (winning surely if played properly).

\paragraph{Result.} 
\texttt{RAFA} (playing O) matches and beats \texttt{gpt-4} for $T=5$ and $T=7$ (Table \ref{tab:tictactoe}), although O is destined to lose. The ablation study ($B=3$ versus $B=4$) illustrates how the planning suboptimality affects the sample efficiency (Figure \ref{fig:tictactoe_results}).

\begin{figure}[H]
\centering
\includegraphics[width=0.85\textwidth]{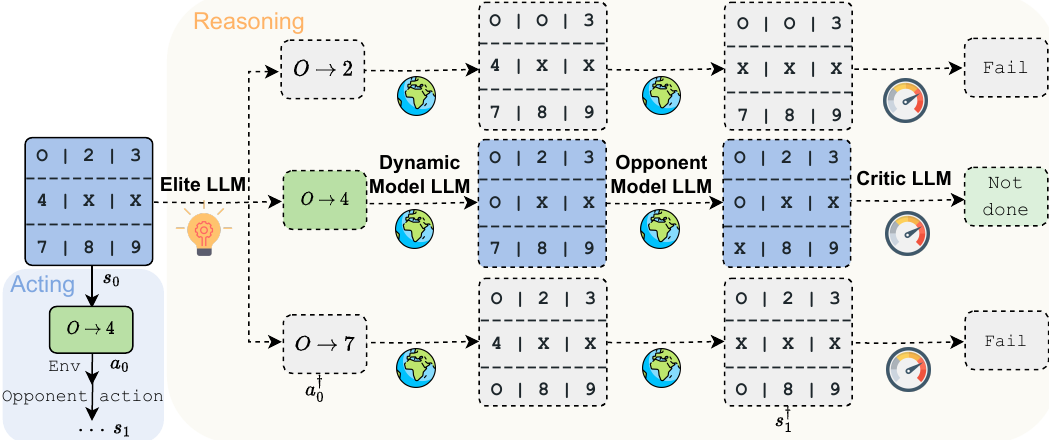}
\setlength{\belowcaptionskip}{-10pt}
\caption{\texttt{RAFA} (playing O) for Tic-Tac-Toe.}
\label{fig:tictactoe_rafa}
\end{figure}

\begin{table}[H]
    \begin{tabularx}{\linewidth}{p{2.8cm}|*{6}{p{2.8cm}}}
      \toprule
      \diagbox[height=1.35em,width=7.5em]{X}{O} & \texttt{gpt-4} & \texttt{RAFA} \!$(T\!=\!1)$ & \texttt{RAFA} \!$(T\!=\!5)$ & \texttt{RAFA} \!$(T\!=\!7)$ \\ 
      \midrule
      \texttt{gpt-4} & 90\%, 0\%, \textbf{10\%} & 90\%, 0\%, \textbf{10\%} & 50\%, 0\%, \textbf{50\%} & 0\%, 0\%, \textbf{100\%} \\
      \bottomrule
    \end{tabularx}
        \setlength{\belowcaptionskip}{-10pt}
    \captionof{table}{\small Tic-Tac-Toe Results. We set $B=4$ and report the winning rate of X, the tie rate, and the winning rate of O.}
    \label{tab:tictactoe}
\end{table}

\section{Theoretical Analysis}
\label{sec:theory}
In this section, we provide the theoretical results in this paper. In Section \ref{app:plan}, we characterize the requirement for the planning subroutine in  \texttt{RAFA} and show the value iteration algorithm with a truncated horizon can be an example of the desired planner. In Section \ref{sec:bma}, we show that the LLM with a posterior alignment performs BMA, which supports Claim \ref{claim:llm} in theory. We present the regret analysis for \texttt{RAFA} in Section \ref{app:theory} to explain its superior empirical performance, where we provide necessary assumptions and the regret bound of \texttt{RAFA}. In Section \ref{app:theory_explore}, we show that \texttt{RAFA} can be modified to encourage efficient exploration for more complex tasks such that \texttt{RAFA} is still sample-efficient without the concentrability assumption in Section \ref{app:theory}.

\subsection{Planning Optimality}\label{app:plan}


 To characterize the requirement for the planning subroutine in  \texttt{RAFA} (Algorithm \ref{alg: llm}), we define the $\epsilon$-optimal planner as follows. 
 \begin{definition}[{$\epsilon$-Optimal Planner}]\normalfont
 Denote $\{V\mid V \text{ is a value function}\}$ as  $\mathcal{V}$. 
A planning algorithm $\texttt{PL}^\epsilon:\mathcal{P}\times\mathcal{R}\mapsto\Pi\times\mathcal{V}$ is an $\epsilon$-optimal planner if $\texttt{PL}^\epsilon( P, r) = (\pi,V)$, where $ |Q(s,a) -  r(s,a) - \gamma\cdot( P V)(s,a)|\le \eps$ and $ V(s) = \max_{a}Q(s,a)=Q(s,\pi(s))$ for all $(s,a)\in\cS\times\cA.$\label{def: epsoptplan}
\end{definition} In other words, an $\eps$-optimal planner with a model (transition kernel and reward function) can generate a policy to approximately maximize the corresponding long-term value function instead of the myopic reward with an approximate error limit $\eps$. 
As an instance of the planner satisfying Definition \ref{def: epsoptplan}, we present the value iteration algorithm (Algorithm \ref{alg: fvi}) with a truncated horizon $U$, i.e., a finite length of the lookahead window as the $\epsilon$-optimal planner in Algorithm~\ref{alg: theory_gen}. The following proposition ensures that Algorithm \ref{alg: fvi} satisfies Definition \ref{def: epsoptplan}.
\begin{prop}\normalfont\label{prop:viisepsopt}
    Algorithm \ref{alg: fvi} is an $\epsilon$-optimal planner as long as we set $U\ge 1+\lceil \log_{\gamma}(\eps/L)\rceil$ and any value function is bounded by $L\ge0$.
\end{prop}
\begin{proof}
    See Appendix \ref{prop:C2} for a detailed proof.
\end{proof}
\begin{algorithm}[H]
	\caption {$\epsilon$-Optimal planner: The value iteration algorithm with a truncated horizon.}
	\label{alg: fvi}
	\begin{algorithmic}[1]
	\STATE  \textbf{input}: The model $(P,r)$ and the truncated horizon $U$. 
 \STATE  \textbf{initialization}: Set the value function $V_\theta^{(U)} (\cdot) \leftarrow 0$.
    \FOR{$u=U-1,\ldots,1$}
    \STATE  Set the value function $V^{(u)}(\cdot) \leftarrow \max_{a\in\mathcal{A}}Q^{(u)}(\cdot,a)$, where $Q^{(u)}(\cdot,\cdot) \leftarrow r(\cdot,\cdot)+\gamma (PV^{(u+1)})(\cdot,\cdot)$.
    \ENDFOR
    \STATE  \textbf{output}: The greedy policy $\pi(\cdot) = \arg\max_{a\in\mathcal{A}} Q^{(1)}(\cdot,a)$ and the value function $V^{(1)}$.
	\end{algorithmic}
\end{algorithm}
 Alternatively, we may choose to emulate the tree-search algorithm, the random shooting algorithm, or the Monte-Carlo tree-search algorithm. In the tree-search example (Lines 5-11 in Algorithm \ref{alg: example}), $\epsilon$ decreases as the search breadth $B$ and depth $U$ increase. Note that, as long as we emulate an $\epsilon$-optimal planner, we are able to establish provable sample efficiency guarantees.
 \subsection{LLMs with Posterior Alignments 
 Perform Bayesian Model Averaging (BMA)} \label{sec:bma}In the following, we analyze Claim \ref{claim:llm} from the theoretical perspective. For LLMs used in \texttt{RAFA}, we denote $P^{\texttt{LLM}}(\xi_{(s,a)}\mid \cD, s, a)$ as the probability measure of the predicted state-reward pair given the query state-action pair and the memory buffer $\mathcal{D}$ as the in-context dataset. 
 Induced by $P^{\texttt{LLM}}$, we also denote $P_{\texttt{LLM}(\mathcal{D})}$ and $r_{\texttt{LLM}(\mathcal{D})}$ as the estimated transition kernel and reward function by LLMs, respectively.

For the simplicity of analysis, we assume that all LLMs have posterior alignments in the tasks that we study, that is, their posterior distributions of the reward and the next state given the current state-action pair and any in-context dataset match the posteriors in these tasks. We formulate this assumption as follows.
\begin{assumption}
    [Posterior Alignment] We assume that LLMs are aligned with the posterior of the state and reward in the underlying MDP, which is formulated as \label{as:perfect}
    \begin{align*}
P^{\texttt{LLM}}\bigl(\xi_{(s,a)}\big|\,\cD,s,a\bigr) &= \mathbb{P}_{\text{post}}\bigl(\xi_{(s,a)}\big|\, \mathcal{D},s,a\bigr),
    \end{align*}
    for any in-context dataset $\mathcal{D}=\{(s_i,a_i,r_i,s_{i}^\prime)\}_{i=0}^I$ with size $I$, query state-action pair $(s,a)$, reward $r$, and state $s^\prime$.
    Here, the posterior $\mathbb{P}_{\text{post}}$ is defined in \eqref{eq:posterior}.
\end{assumption}
We remark that the posterior alignment in Assumption \ref{as:perfect} comes from the in-context ability of LLMs, which is widely studied in \citet{lee2023supervised,wies2024learnability,xie2021explanation}. We also remark that Assumption \ref{as:perfect} does not mean that LLMs can trivially make the optimal decision at each step in the underlying MDP: (1) Though the posterior distributions of state and reward are aligned, LLMs still need to be instructed to maximize the long-term value (via explicit planning) instead of the myopic reward. (2) LLMs still require online interactions to enlarge the in-context dataset $\cD$ such that their prediction uncertainty can be reduced from the prior uncertainty. 
In Appendix \ref{app:relax}, we also discuss how to relax Assumption \ref{as:perfect} to accommodate an additional error term in the regret bound derived by our analysis, where we assume that LLMs are maximum likelihood estimatiors (MLEs) on the pretraining dataset with uniform coverage.   Based on Assumption \ref{as:perfect}, we prove that LLMs with posterior alignments perform BMA in the model estimation in the following proposition.
\begin{prop}
    [LLMs with Posterior Alignments Perform BMA] Under Assumption \ref{as:perfect}, the LLM predictions satisfy  
   \begin{align*}
       r_{\texttt{LLM}(\mathcal{D})}(s,a)+\gamma\cdot(P_{\texttt{LLM}(\mathcal{D})}V)(s,a)  =\mathbb{E}_{\theta\sim \mathbb{P}_{\text{post}}(\cdot\mid\cD)} [(B_{\theta} V)(s,a)]
   \end{align*}
    for any dataset $\mathcal{D}=\{(s_i,a_i,r_i,s_{i}^\prime)\}_{i=0}^I$ with size $I$, value function $V$, and query state-action pair $(s,a)\in\cS \times\cA$. Here,  $\mathbb{P}_{\text{post}}(\theta\mid \mathcal{D})$ is the posterior of $\theta^\star$ given  $\mathcal{D}$ in the underlying MDP. \label{prop:llm_bma}
\end{prop}
\begin{proof}
    [Proof of Proposition \ref{prop:llm_bma}]
    See the detailed proof in Appendix \ref{pf:llm_bma}.
\end{proof}

The proof of Proposition \ref{prop:llm_bma} can be found in Appendix \ref{pf:llm_bma}. 
Some variants of Proposition \ref{prop:llm_bma} can be found in various literature \citep{lee2023supervised,zhang2022analysis,zhang2023and}. In particular, 
\citet{zhang2022analysis} establish the theoretical equivalence between BMA and the ideal attention architecture and analyze the generalization error rate of  LLMs. By Proposition \ref{prop:llm_bma}, LLMs can provide a more certain and accurate estimate for the data-generating model with more collected feedback, as the uncertainty in the posterior is reduced with more data. Thus, Proposition \ref{prop:llm_bma} supports Claim \ref{claim:llm} in theory. 

\subsection{Regret Bound of \texttt{RAFA}}\label{app:theory}
\begin{algorithm}[H]
	\caption {Reason for future, act for now (\texttt{RAFA}): The theoretical version.}
	\label{alg: theory_gen}
	\begin{algorithmic}[1]
	\STATE  \textbf{input}:  An $\epsilon$-optimal planner $\texttt{PL}^\eps$, which returns an $\eps$-optimal policy that maximizes the value function up to an $\eps$ accuracy (Definition \ref{def: epsoptplan}), and LLMs with posterior alignments. 
 \STATE  \textbf{initialization}: Sample the initial state $s_0 \sim \rho$, set $t = 0$, and initialize the memory buffer $\mathcal{D}_0 = \emptyset$.
    \FOR{$k=0,1,\ldots,$}
    \STATE  Set $t_k \leftarrow{t}$. 
    \REPEAT
    \STATE  Plan ahead with the $\eps$-optimal planner and LLMs $(\pi_t, V_t)\leftarrow \texttt{PL}^\eps(P_{\texttt{LLM}(\mathcal{D}_{t_k})},r_{\texttt{LLM}(\mathcal{D}_{t_k})})$. \label{line:pl_gen}
    \hfill (``reason for future'')
        \STATE   Execute action $a_t=\pi_t(s_t)$ to receive  reward $r_t$ and state $s_{t+1}$ from environment. \\\hfill (``act for now'')
        \STATE  Update the memory buffer  $\mathcal{D}_{t+1} \leftarrow \mathcal{D}_{t} \cup \{(s_t,a_t,s_{t+1},r_t)\}$. 
        \STATE  Set $t\leftarrow t+1.$ 
        \UNTIL{$H_{t_{k}}-H_t > \log 2$, where $H_t$ denotes posterior entropy of $\theta^\star$ conditioned on $\mathcal{D}_t$.\label{line:sw_gen}\\ \hfill (the switching condition is satisfied )}
    \ENDFOR
	\end{algorithmic}
\end{algorithm}

To analyze \texttt{RAFA} in theory, we 
 propose the theoretical version of \texttt{RAFA} in Algorithm \ref{alg: theory_gen}, where we
 instantiate the switching condition of \texttt{RAFA} in Line \ref{line:sw_gen} by measuring the reduction of the posterior entropy. At the $t$-th step and the $k$-th switching times, Algorithm \ref{alg: theory} only makes the $(k+1)$-th switch when the reduction of posterior entropy $H_{t_k}-H_{t}$ is greater than $\log 2$. 
 In Line \ref{line:pl_gen} of Algorithm \ref{alg: theory_gen}, we describe the planning subroutine in \texttt{RAFA} (Algorithm \ref{alg: llm}) by an $\eps$-planner $\texttt{PL}^\eps$ (defined in Definition \ref{def: epsoptplan}).  We specify the terminating condition for Algorithm \ref{alg: theory_gen}. Let $(K-1)$ be the total number of switches until $t$ reaches $(T-1)$. Let $t_{K}=T$. At the  $(T-1)$-th step, Algorithm \ref{alg: theory_gen} executes $a_{T-1} = \pi_{T-1}(s_{T-1})$ , where we have $\pi_{T-1} = \texttt{PL}^\eps(P_{\texttt{LLM}(\mathcal{D}_{t_{K-1}})},r_{\texttt{LLM}(\mathcal{D}_{t_{K-1}})})$. Upon receiving $r_{T-1}$ and $s_{T}$ from the external environment, Algorithm \ref{alg: theory_gen} updates $\mathcal{D}_{T} = \{(s_t,a_t,s_{t+1},r_t)\}_{t=0}^{T-1}$ and terminates. Since the agent in Algorithm \ref{alg: theory_gen} executes the same policy until making a switch, we have $\pi_t=\pi_{t_k}$ for any $t_k\le k <t_{k+1}$. We denote by $\pi^k=\pi_{t_k}$ for the notational simplicity. 
 Next, we impose a regularity assumption on the structure of MDPs to measure the learning difficulty.  Recall that we define the posterior entropy $H_t$  in \eqref{eq:ent_def}, the information gain $I(\theta;\xi\mid \cD)$, and $\xi_{(s,a)}$ as the pair of the next state and the current reward $(s^\prime,r)$ given the query state-action pair $(s,a)$ in Section \ref{sec: pre}. Define $H_t$ as the posterior entropy $H(\theta\mid \cD_t)$ given the dataset $\mathcal{D}_t = \{(s_i, a_i,r_i,s_{i+1})\}_{i=0}^{t-1}$.
\begin{assumption}
    [MDPs Regularity] We assume that there exists a coefficient $\eta>0$ such that, if $H_{t_1} - H_{t_2} \le \log 2$, then it holds that
   \begin{equation*}
        I(\theta;\xi_{(s,a)}\mid \mathcal{D}_{t_1})\le 4\eta \cdot I(\theta;\xi_{(s,a)}\mid \mathcal{D}_{t_2})
   \end{equation*}
   for any given value function $V$, $t_1 <t_2$ and $(s,a)\in\cS\times\cA$.\label{as:var}\vspace{-0.1cm}
\end{assumption}
Assumption \ref{as:var} is a regularity assumption on MDPs and is intrinsic to the agent design. In Appendix \ref{app:as:lkmdp}, we prove that  $d$-dimensional Bayesian linear kernel MDPs (defined in Definition \ref{def:lk mdp}),  satisfy Assumption \ref{as:var} with the coefficient $\eta = d/\log(1+d)$.   Intuitively, Assumption \ref{as:var} restricts the increase of the information gain given one bit ($\log 2$) reduction of the posterior entropy.


Similar to other theoretical work on deep RL \citep{lazaric2010analysis,fan2020theoretical,zhang2020generative}, we introduce the concentrability coefficient $\kappa$ to bound the distribution shift between the current policy and the optimal policy. For the simplicity of discussions, we define the optimal $\gamma$-discounted visitation measure $\nu^\star$ starting from state $s$ as 
\begin{align}
    \nu^\star(s^\prime \mid s) = \frac{1}{1-\gamma}\cdot \sum_{\tau=0}^\infty \gamma^\tau \cdot \mathbb{P}\bigl(s_\tau = s^\prime\big|\, s_0 = s, s_{i+1} \sim P_{\theta^\star}(\cdot \mid s_i, \pi^\star(s_i)) \text{ for any }0\le i < \tau\bigr), \label{eq:optimal_vis}
\end{align}
for any state $s, s^\prime\in\mathcal{S}$. 
Here, $\nu^\star(\cdot\mid s)$ describes the discounted average probability measure of the state that the optimal policy $\pi^\star$ visits starting from state $s$ in the underlying MDP. Now, we are ready to provide the full statement of the concentrability coefficient as follows.
 \begin{assumption}
        [Concentrability] For \texttt{RAFA} (Algorithm \ref{alg: theory_gen}), we assume that there exists a constant  $\kappa<\infty$ such that \label{as:conc_coef}
        \begin{equation*}
            \mathbb{E}\Biggl[\sum_{k=0}^{K-1}\mathbb{E}_{\pi^k}\Biggl[\sum_{t=t_k}^{t_{k+1}-1} \mathbb{E}_{\theta^\star\sim \mathbb{P}_{t_k}}\Biggl[\frac{\mathbb{E}_{s\sim \nu^{\star}(\cdot\mid s_t)}\bigl[\bigl((B_k-B_{\theta^\star})  V_t\bigr)^2(s,\pi^\star(s))\bigr]}{\bigl((B_k-B_{\theta^\star})  V_t\bigr)^2(s_t,\pi^k(s_t))}\Biggr|\mathcal{D}_{t_k}\Biggr]\Biggr]\Biggr] 
        \end{equation*}
        is bounded by $\kappa^2 \cdot T$, where we define $ (B_k V)(s,a) = r_{\texttt{LLM}(\mathcal{D}_{t_k})}(s,a)+\gamma\cdot (P_{\texttt{LLM}(\mathcal{D}_{t_k})}V)(s,a)$ and denote by $\mathbb{P}_{t_k}$ the posterior of $\theta^\star$ given $\cD_{t_k}$.
\end{assumption}
    Intuitively, $\kappa$ measures the hardness to generalize the low prediction error $(B_k-B_{\theta^\star})  V_t$ on the current trajectory induced by $\pi^k$ to the optimal trajectory induced by $\pi^\star$ in the underlying MDP. We remark that we can drop the dependency of the concentrability coefficient $\kappa$ (Assumption \ref{as:conc_coef}) if we modify \texttt{RAFA} to encourage efficient exploration in MDPs. We will discuss the variants of \texttt{RAFA} with efficient exploration strategies in Section \ref{app:theory_explore}. 

    In the following theorem, we give the bound of the Bayesian regret of \texttt{RAFA} (Algorithm \ref{alg: theory_gen}) as follows.
\begin{theorem}
    Under Assumptions \ref{as:perfect}, \ref{as:var} , and \ref{as:conc_coef},the Bayesian regret of \texttt{RAFA} (Algorithm \ref{alg: theory_gen}) satisfies \label{thm:reg_conc}
    \begin{equation*}
        \mathfrak{R}(T)= \mathcal{O}\Biggl(\frac{(\kappa+1)L\cdot\sqrt{\mathbb{E}[H_0-H_T]}}{1-\gamma}\cdot\sqrt{T} +\frac{\eps}{1-\gamma}\cdot T + \frac{L\cdot\mathbb{E}[H_0 - H_{T}]}{1-\gamma}\Biggr),
    \end{equation*}
    where $\kappa$ is the concentrability coefficient defined in Assumption \ref{as:conc_coef}, $H_t$ is the posterior entropy of $\theta^\star$ given the history $\mathcal{D}_t = \{(s_i, a_i,r_i,s_{i+1})\}_{i=0}^{t-1}$, and $L$ is the bound of $|r+V(s)|$ for any reward $r$, state $s$, and value function $V$.
\end{theorem}
\begin{proof}
    [Proof of Theorem \ref{thm:reg_conc}]
    See the detailed proof in Appendix \ref{pf:reg_conc}.
\end{proof}
 Theorem \ref{thm:reg} establishes the $\sqrt{T}$ regret of \texttt{RAFA} (Algorithm \ref{alg: theory_gen}) for a proper choice of the planning suboptimality $\eps$, e.g., $\eps = {\mathcal{O}}(1/\sqrt{T})$, which shows that \texttt{RAFA} is sample efficient and explains its strong empirical performance in Section \ref{sec: experiments}.  Here, the first term in the upper bound in Theorem \ref{thm:reg} is the leading term and involves several multiplicative factors, namely the effective horizon $1/(1-\gamma)$, the value bound $L$, and the cumulative posterior entropy reduction ${H_0-H_T}$ throughout the $T$ steps, which are common in the RL literature \citep{abbasi2015bayesian, osband2013more, russo2014learningx, russo2014learning, russo2016information, lu2019information}. In particular, $H_0$ highlights the prior knowledge obtained through pretraining, as $H_0$ quantifies the prior uncertainty of LLMs before incorporating any collected feedback. Hence, $H_0-H_T$ highlights the uncertainty reduction achieved by reasoning and acting, as $H_T$ quantifies the posterior uncertainty of LLMs after incorporating the collected feedback. In Appendix \ref{app:linear}, we prove that $H_0-H_T=\cO(d\cdot\log T)$ and the $1-\delta$ probability bound on value functions $L=\mathcal{O}(\sqrt{d}\cdot\log(dT/\delta))$ for the $d$-dimensional Bayesian linear kernel MDPs, which implies $\mathfrak{R}(T) = \tilde{\mathcal{O}}((1-\gamma)^{-1}(\kappa+1)\cdot\sqrt{d^3T})$ with probability at least $1-\delta$. Here $\tilde{\mathcal{O}}$ hides the logarithmic factor. 
\subsection{\texttt{RAFA} with Efficient Exploration Strategies}\label{app:theory_explore}
To drop the dependency of Assumption \ref{as:conc_coef} (Concentrability) and solve more complex tasks, we provide two variants of \texttt{RAFA} (Algorithm \ref{alg: theory_gen}): (1) \texttt{RAFA} with an optimistic bonus (Algorithm \ref{alg: theory_bonus}) and (2) \texttt{RAFA} with posterior sampling (Algorithm \ref{alg: theory}). We also prove the bound of the Bayesian regret of each variant, which demonstrates the effectiveness of these efficient exploration strategies without Assumption \ref{as:conc_coef} (Concentrability).  
\subsubsection{Optimistic Bonus}
We incorporate the \textit{Optimism in Face of Uncertainty} (OFU) principle \citep{cai2020provably,zhou2021provably,jin2020provably,liu2022learning,Wang_Pan_Zhou_Wang_2023} to encourage efficient exploration by adding an optimistic bonus on the reward function in the planning subroutine of \texttt{RAFA}. We design the optimistic bonus by the information gain and implement a variant of \texttt{RAFA} in Algorithm \ref{alg: theory_bonus}. In particular, the bonus $\Gamma_k(s,a)$ takes the following form
\begin{align}
\Gamma_k(s,a) = \sqrt{2}L\cdot\sqrt{I(\theta;\xi_{(s,a)}\mid \mathcal{D}_{t_k})}
\end{align}
 for any $(s,a)\in\cS\times\cA$ and $k<K$. In Line \ref{line:pl_bonus} of Algorithm \ref{alg: theory_bonus}, we generate the policy $\pi^t$ by $\texttt{PL}^\eps(P_{\texttt{LLM}(\mathcal{D}_{t_k})},r_{\texttt{LLM}(\mathcal{D}_{t_k})}+\Gamma_k)$ for any $t_k\le t<t_{k+1}$. Intuitively, the bonus is higher at the state-action pair with higher information gain, which incentivizes the agent to explore those less visited states (with higher information gain). In the following theorem, we prove the regret bound of \texttt{RAFA} with an optimistic bonus (Algorithm \ref{alg: theory_bonus}). 
\begin{algorithm}[t]
	\caption {Reason for future, act for now (\texttt{RAFA}): The theoretical version with an optimistic bonus.}
	\label{alg: theory_bonus}
	\begin{algorithmic}[1]
	\STATE  \textbf{input}:  An $\epsilon$-optimal planner $\texttt{PL}^\eps$, which returns an $\eps$-optimal policy that maximizes the value function up to an $\eps$ accuracy (Definition \ref{def: epsoptplan}), and LLMs with posterior alignments. 
 \STATE  \textbf{initialization}: Sample the initial state $s_0 \sim \rho$, set $t = 0$, and initialize the memory buffer $\mathcal{D}_0 = \emptyset$.
    \FOR{$k=0,1,\ldots,$}
    \STATE  Set $t_k \leftarrow{t}$. 
    \REPEAT
    \STATE  Design optimistic bonus $\Gamma_k(s,a)=\sqrt{2}L\cdot\sqrt{I(\theta;\xi_{(s,a)}\mid \mathcal{D}_{t_k})}$ for all $(s,a)\in\cS\times\cA$.
    \STATE  Plan ahead with the $\eps$-optimal planner and LLMs\\ $(\pi_t, V_t)\leftarrow \texttt{PL}^\eps(P_{\texttt{LLM}(\mathcal{D}_{t_k})},r_{\texttt{LLM}(\mathcal{D}_{t_k})}+\Gamma_k)$. \label{line:pl_bonus}
    \\\hfill (``reason for future'')
        \STATE   Execute action $a_t=\pi_t(s_t)$ to receive  reward $r_t$ and state $s_{t+1}$ from environment. \hfill (``act for now'')
        \STATE  Update memory $\mathcal{D}_{t+1} \leftarrow \mathcal{D}_{t} \cup \{(s_t,a_t,s_{t+1},r_t)\}$. 
        \STATE  Set $t\leftarrow t+1.$ 
        \UNTIL{$H_{t_{k}}-H_t > \log 2$, where $H_t$ denotes posterior entropy of $\theta^\star$ conditioned on $\mathcal{D}_t$.\\ \hfill (the switching condition is satisfied )}
    \ENDFOR
	\end{algorithmic}
\end{algorithm}

\begin{theorem}
    Under Assumptions \ref{as:perfect} and \ref{as:var}, the Bayesian regret of \texttt{RAFA} with an optimistic bonus (Algorithm \ref{alg: theory_bonus}) satisfies \label{thm:reg_bonus}
    \begin{equation*}
        \mathfrak{R}(T)= \mathcal{O}\Biggl(\frac{L\cdot\sqrt{\mathbb{E}[H_0-H_T]}}{1-\gamma}\cdot\sqrt{T} +\frac{\eps}{1-\gamma}\cdot T + \frac{L\cdot\mathbb{E}[H_0 - H_{T}]}{1-\gamma}\Biggr),
    \end{equation*}
    where all the variables have the same definitions in Theorem \ref{thm:reg_conc}.
\end{theorem}
\begin{proof}
    [Proof of Theorem \ref{thm:reg_bonus}]
    See the detailed proof in Appendix \ref{pf:thm:reg_bonus}.
\end{proof}
Compared with Theorem \ref{thm:reg_conc}, the regret bound in Theorem \ref{thm:reg_bonus} is not dependent on the concentrability coefficient $\kappa$, which demonstrates the effectiveness of the optimistic bonus in Algorithm \ref{alg: theory_bonus}. In Appendix \ref{app:linear}, we prove that $H_0-H_T=\cO(d\cdot\log T)$ and the $1-\delta$ probability bound on value functions $L=\mathcal{O}(\sqrt{d}\cdot\log(dT/\delta))$ for the $d$-dimensional Bayesian linear kernel MDPs, which implies $\mathfrak{R}(T) = \tilde{\mathcal{O}}((1-\gamma)^{-1}\cdot\sqrt{d^3T})$ with probability at least $1-\delta$. Here $\tilde{\mathcal{O}}$ hides the logarithmic factor. 
\subsubsection{Posterior Sampling}

\begin{algorithm}[H]
	\caption {Reason for future, act for now (\texttt{RAFA}): The theoretical version with posterior sampling.}
	\label{alg: theory}
	\begin{algorithmic}[1]
	\STATE  \textbf{input}:  An $\epsilon$-optimal planner $\texttt{PL}^\eps$, which returns an $\eps$-optimal policy that maximizes the value function up to an $\eps$ accuracy (Definition \ref{def: epsoptplan}), and LLMs satisfying Assumption \ref{as:ts}. 
 \STATE  \textbf{initialization}: Sample the initial state $s_0 \sim \rho$, set $t = 0$, and initialize the memory buffer $\mathcal{D}_0 = \emptyset$.
    \FOR{$k=0,1,\ldots,$}
    \STATE  Set $t_k \leftarrow{t}$. 
    \REPEAT
    \STATE  Plan ahead with the $\eps$-optimal planner and the posterior sampling mechanism of LLMs (defined in Assumption \ref{as:ts}) $(\pi_t, V_t)\leftarrow \texttt{PL}^\eps(P_{\texttt{LLM+PS}(\mathcal{D}_{t_k})},r_{\texttt{LLM+PS}(\mathcal{D}_{t_k})})$.\label{line:pl_ts}\\
    \hfill (``reason for future'')
        \STATE   Execute action $a_t=\pi_t(s_t)$ to receive  reward $r_t$ and state $s_{t+1}$ from environment. \hfill (``act for now'')
        \STATE  Update memory $\mathcal{D}_{t+1} \leftarrow \mathcal{D}_{t} \cup \{(s_t,a_t,s_{t+1},r_t)\}$. 
        \STATE  Set $t\leftarrow t+1.$ 
        \UNTIL{$H_{t_{k}}-H_t > \log 2$, where $H_t$ denotes posterior entropy of $\theta^\star$ conditioned on $\mathcal{D}_t$.\\ \hfill (the switching condition is satisfied )}
    \ENDFOR
	\end{algorithmic}
\end{algorithm}
As another method for efficient exploration, we assume that there exists a mechanism that deploys posterior sampling and we use this mechanism to encourage exploration for \texttt{RAFA}.  
\begin{assumption}
    [LLMs with Posterior Sampling Mechanism] We assume that there exists a mechanism \texttt{LLM+PS} that maps the memory buffer $\cD$ to the transition kernel and the reward function, such that $(r_{\texttt{LLM+PS}(\mathcal{D})}(s,a)+\gamma\cdot(P_{\texttt{LLM+PS}(\mathcal{D})}V)(s,a) )\mid \mathcal{D}$ and $(B_{\theta^\star}V(s,a)) \mid\mathcal{D} $ are identically independent distributed for any $(s,a)\in\cS\times\cA$, in-context dataset $\mathcal{D}$, and value function $V$. Here, $\theta^\star$ is the data-generating parameter.
    \label{as:ts}
\end{assumption}
We remark that the bootstrap method \citep{efron1982jackknife} can approximate the posterior sampling mechanism satisfying Assumption \ref{as:ts}. Widely used in applied statistics \citep{davison1997bootstrap} and the design of RL algorithms \citep{osband2016deep,hao2019bootstrapping}, the bootstrap method takes a dataset $D$ and a functional estimator $g$ as inputs. Depending on the configuration of bootstrap, we generate the bootstrapped dataset $\tilde{D}$ from $D$ by uniform sampling with replacement \citep{efron1982jackknife} or weighted sampling with replacement \citep{newton1994approximate}. Viewing the LLM as the functional estimator $g$ and the memory buffer $\mathcal{D}$ as the dataset $D$, we 
can use this bootstrap method to approximate the mechanism \texttt{LLM+PS} that is introduced in Assumption \ref{as:ts}. From the statistics literature \citep{bickel1981some,singh1981asymptotic,newton1994approximate}, we also know that bootstrap distribution recovers the posterior distribution asymptotically.

Based on the mechanism satisfying Assumption \ref{as:ts}, we propose a variant of \texttt{RAFA} in Algorithm \ref{alg: theory}, where we use the mechanism \texttt{LLM+PS} as the model estimator in the learning subroutine of \texttt{RAFA}.  In Line \ref{line:pl_bonus} of Algorithm \ref{alg: theory_bonus}, we generate the policy $\pi^t$ by $\texttt{PL}^\eps(P_{\texttt{LLM+PS}(\mathcal{D}_{t_k})},r_{\texttt{LLM+PS}(\mathcal{D}_{t_k})})$. In the following, we give a simple explanation of how this mechanism helps the agent to explore efficiently. By the Bayes' rule, we have $p(\theta\mid \mathcal{D})\propto L(\mathcal{D}\mid\theta)\mathbb{P}_0(\theta)$, where $L(\mathcal{D}\mid \theta)$ is the likelihood of $\mathcal{D}$ given $\theta$ and $\mathbb{P}_0$ is the prior of $\theta^\star$. Taking the logarithm, we have
$\log( p(\theta\mid \mathcal{D}))=c+\log(\mathbb{P}_0(\theta))+\log(L(\mathcal{D}\mid\theta))$ for some constant $c$. Hence, the uncertainty of the posterior is higher ($p(\theta\mid \mathcal{D})$ is closer to $0$) at the less visited states (the likelihood of these states is closer to $0$). Suppose we sample the model estimator from the posterior. In that case, the agent has more incentives to explore the less visited states, which explains why the mechanism \texttt{LLM+PS} encourages the efficient exploration.

In the following theorem, we prove the regret bound of \texttt{RAFA} with posterior sampling (Algorithm \ref{alg: theory}).
\begin{theorem}[Bayesian Regret]\normalfont\label{thm:reg}
    Under Assumptions  \ref{as:var} and \ref{as:ts}, the Bayesian regret of \texttt{RAFA} with posterior sampling (Algorithm \ref{alg: theory}) satisfies 
    \begin{equation*}
        \mathfrak{R}(T)= \mathcal{O}\Biggl(\frac{L\cdot\sqrt{\mathbb{E}[H_0-H_T]}}{1-\gamma}\cdot\sqrt{T} +\frac{\eps}{1-\gamma}\cdot T + \frac{L\cdot\mathbb{E}[H_0 - H_{T}]}{1-\gamma}\Biggr),
    \end{equation*}
    where all the variables have the same definitions in Theorem \ref{thm:reg_conc}.
\end{theorem}
\begin{proof}
    [Proof of Theorem \ref{thm:reg}]
    See the detailed proof in Appendix \ref{pf:reg}.
\end{proof}
Compared with Theorem \ref{thm:reg_conc}, the regret bound in Theorem \ref{thm:reg} is not dependent on the concentrability coefficient $\kappa$, which demonstrates the effectiveness of the posterior sampling mechanism in Algorithm \ref{alg: theory}. In Appendix \ref{app:linear}, we prove that $H_0-H_T=\cO(d\cdot\log T)$ and the $1-\delta$ probability bound on value functions $L=\mathcal{O}(\sqrt{d}\cdot\log(dT/\delta))$ for the $d$-dimensional Bayesian linear kernel MDPs, which implies $\mathfrak{R}(T) = \tilde{\mathcal{O}}((1-\gamma)^{-1}\cdot\sqrt{d^3T})$ with probability at least $1-\delta$. Here $\tilde{\mathcal{O}}$ hides the logarithmic factor.  
\section{Conclusions}
In this paper, we establish the LLM-RL correspondence and propose a principled framework \texttt{RAFA} for orchestrating reasoning and acting, which achieves provable sample efficiency guarantees in autonomous LLM agents for the first time.  We prove the $\sqrt{T}$ regret bound of \texttt{RAFA} to highlight the synergy between prior knowledge from pretraining and the iterative process of reasoning and acting. 
\texttt{RAFA}'s outstanding empirical performance underscores its potential for autonomous and adaptive decision-making in various complex tasks, which we remain for future work. 
\section*{Acknowledgement}
Zhaoran Wang acknowledges National Science Foundation (Awards 2048075, 2008827, 2015568, 1934931), Simons Institute (Theory of Reinforcement Learning), Amazon, J.P. Morgan, and Two Sigma for their supports.
\bibliographystyle{ims}

\bibliography{main}

\newpage
\appendix
\section{Notations}
We provide a table of notations here.\label{app:not}
\begin{table}[!h]
\small
    \centering\label{tab:no}
    \begin{tabular}{c|c}
    \hline\hline
    \textbf{Notations}  & \textbf{Explanation}\\ \hline
    \multirow{2}{*}{$\xi_{(s,a)}$} & the pair of the next state and the current reward $(s^\prime,r)$, \\ 
    & given the query state-action pair $(s,a)$ \\
    \multirow{2}{*}{$P^{\texttt{LLM}}(\xi_{(s,a)}\mid \cD ,s,a)$} & the probability measure of the LLM predicted $\xi_{(s,a)}$ \\ &  given the memory buffer $\mathcal{D}$ as contexts\\
    \multirow{2}{*}{$r_{\text{LLM}(\mathcal{D})}$}& the LLM reward estimator \\ & with the memory buffer $\mathcal{D}$ prompted as contexts\\
    \multirow{2}{*}{$P_{\text{LLM}(\mathcal{D})}$}& the LLM transition kernel estimator \\ &with the memory buffer $\mathcal{D}$ prompted as contexts\\
$\mathcal{D}_t$& the history at the $t$-th step, which includes $\{(s_i,a_i,r_i,s_{i+1})\}_{i=0}^{t-1}$ \\ 
$\mathbb{P}_0(\theta)$ & the prior of $\theta^\star$\\
$\mathbb{P}_{\text{post}}(\theta\mid\cD)$& the posterior of $\theta^\star$ conditioned on $\mathcal{D}$\\ 
$\mathbb{P}_{t}(\theta)$& the abbreviation of $\mathbb{P}_{\text{post}}(\theta\mid\mathcal{D}_t)$\\ 
   $H(\theta\mid \mathcal{D})$&the posterior entropy of the posterior of $\theta$ conditioned on $\mathcal{D}$\\
  $I(\theta;\xi\mid \cD)$ & the information gain of $\xi$, defined by $H(\theta\mid\cD) - H(\theta,\xi\mid\cD)$\\
$H_t$& the abbreviation of $H(\theta\mid \mathcal{D}_t)$\\
$t_k$& the timestep when \texttt{RAFA} switches the policy for the $k$-th time \\
$\pi^k$& the abbreviation of $\pi_{t_k}$ \\
$B_\theta$& the Bellman operator induced by $\theta$\\
$B_k$ & the Bellman operator induced by $\texttt{LLM}(\mathcal{D}_{t_k})$\\ 
    $d_{\text{TV}}(p\|q)$ & total variation (TV) between two probability measures $p$ and $q$\\
    $d_{\text{KL}}(p\|q)$& Kullback–Leibler divergence between two probability measures $p$ and $q$\\
   $ \mathbb{V}_{x\sim p}[f(x)]$ & the variance of $f(X)$, where $X$ follows the distribution $p$\\
   $ (PV)(s,a)$ &$ \mathbb{E}_{s^\prime \sim P(\cdot\mid s,a)}[V(s^\prime)]$\\
   $L$ & the bound of $|r+V(s)|$ for any $r\in\cR$, $s\in\cS$, and value $V$\\
   $\nu^\star(\cdot\mid s)$ & the optimal $\gamma$-discounted visitation measure starting from state $s$\\
   $\mathbb{N}$&  the set of  natural numbers\\
   $\textbf{1}(x=y)$& the indicator with value $1$ if $x$ equals $y$ and value $0$ otherwise  \\ 
   $\mathbb{E}$ & the expectation \\
   $\mathbb{V}$ & the variance\\
    \hline\hline
    \end{tabular}
\end{table}
\newpage
\section{More Algorithms}\label{app:omitted_alg}

Depending on the specific configuration of the state and action spaces (continuous versus discrete) and the transition and reward models (stochastic versus deterministic), we may choose to emulate the tree-search algorithm, the value iteration algorithm, the random shooting algorithm, or the MCTS algorithm. All of them allow \texttt{RAFA} to achieve provable sample efficiency guarantees as long as they satisfy a specific requirement of optimality (Definition \ref{def: epsoptplan}). For illustration, we emulate the beam-search algorithm (an advanced version of the tree-search algorithm) in Algorithm \ref{alg: beam_det} and the MCTS algorithm in Algorithm \ref{alg: example_stoc}. For the theoretical discussion, we present the value iteration algorithm in Algorithm \ref{alg: fvi}.
\begin{algorithm}[H]
	\caption{The LLM learner-planner (\texttt{LLM-LR-PL}): A beam-search example (for the deterministic case).}
	\label{alg: beam_det}
	\begin{algorithmic}[1]
	\STATE \textbf{input}: The memory buffer $\mathcal{D}$, the initial state $s$, the proposal width $L$, the search breadth $B$, and the search depth $U$.
 \STATE \textbf{initialization}: Initialize the state array $\mathcal{S}_0 \leftarrow \{s\}$ and the action array $\mathcal{A}_0 \leftarrow \emptyset$.
\\--------------------------------------- (the learning subroutine) -----------------------------------------
\STATE Set \texttt{Model} as an LLM instance prompted to use $\mathcal{D}$ as contexts to generate the next state.
\STATE Set \texttt{Critic} as an LLM instance prompted to use $\mathcal{D}$ as contexts to estimate the value function.
\\--------------------------------------- (the planning subroutine) -----------------------------------------
\STATE Set \texttt{Elite} as an LLM instance prompted to use $\mathcal{D}$ as contexts to generate multiple candidate actions. 
     \FOR{$u = 0,\ldots,U$}
\STATE For each current state $s_u$ in $\mathcal{S}_{u}$, invoke \texttt{Elite} to generate $L$ candidate actions. 
\STATE For each candidate action $a_u^{(\ell)}$, invoke \texttt{Model} to generate the next state $s_{u+1}^{(\ell)}$ and the received reward $r_{u}^{(\ell)}$. 
\STATE For each resulting tuple $(s_u, a_u^{(\ell)}, s_{u+1}^{(\ell)}, r_{u}^{(\ell)})$, invoke \texttt{Critic} to evaluate the expected cumulative~future reward $\widehat{Q}(s_u,a^{(\ell)}_u)\leftarrow r_{u}^{(\ell)}+\gamma \widehat{V}(s^{(\ell)}_{u+1})$, where $\widehat{V}$ is given by \texttt{Critic}.
\STATE Select $B$ best tuples $(s_u, a_u^{(\ell)}, s_{u+1}^{(\ell)})$ with the highest value $\widehat{Q}(s_u,a^{(\ell)}_u)$ and write them to $\mathcal{S}_{u}\times \mathcal{A}_u \times \mathcal{S}_{u+1}$.
     \ENDFOR
    \STATE For $B$ preserved rollouts in $\mathcal{S}_0\times\mathcal{A}_0\times \cdots \times \mathcal{S}_U \times \mathcal{A}_U \times  \mathcal{S}_{U+1}$, invoke \texttt{Critic} to evaluate the expected cumulative future reward $\sum_{u=0}^{U} \gamma^u r_{u}^{(b)} + \gamma^{U+1} \widehat{V}(s_{U+1}^{(b)})$ and select the best one $(s_0^\dagger,a_0^\dagger,\ldots,s_U^\dagger, a_U^\dagger, s_{U+1}^\dagger)$, where $\widehat{V}$ is given by \texttt{Critic} and $s_0^\dagger = s$.
    \STATE \textbf{output}: The initial action $a_0^\dagger$ of the selected rollout.
	\end{algorithmic}
\end{algorithm}

\begin{algorithm}[H]
	\caption{LLM learner-planner (\texttt{LLM-PL}) for \texttt{RAFA}: A Monte-Carlo tree-search example (for the stochastic case).}
	\label{alg: example_stoc}
	\begin{algorithmic}[1]
	\STATE \textbf{input}: The memory buffer $\mathcal{D}$, the initial state $s$, the proposal width $L$, $L'$, and the expansion budget $E$.
 \STATE \textbf{initialization}: Initialize the root node $n \leftarrow s$ and the child function $c(\cdot)\leftarrow\emptyset$.
\\--------------------------------------- (the learning subroutine) -----------------------------------------
\STATE Set \texttt{Model} as an LLM instance prompted to use $\mathcal{D}$ as contexts to generate the next state.
\STATE Set \texttt{Critic} as an LLM instance prompted to use $\mathcal{D}$ as contexts to estimate the value function.
\\--------------------------------------- (the planning subroutine) -----------------------------------------
\STATE Set \texttt{Elite} as an LLM instance prompted to use $\mathcal{D}$ as contexts to generate multiple candidate actions. 

     \FOR{$e = 0,\ldots,E$}
     \STATE Set $s_e\leftarrow n$.
     \WHILE{$s_e$ is not a leaf node, i.e., $c(s_e) \neq \emptyset$,}
        \STATE Invoke \texttt{Critic} to evaluate the expected cumulative future reward and select the child node $a_e$ in $c(s_e)$ with the highest value $\widehat{Q}(s_e,a_e)$.
        \STATE Set $s_e$ as a child node in $c(a_e)$.
     \ENDWHILE
\STATE For the current state $s_e$, invoke \texttt{Elite} to generate $L$ candidate actions. 
\STATE Write each candidate action $a^{(\ell)}_e$ to $c(s_e)$, i.e., $c(s_e) \leftarrow \{a^{(\ell)}_e\}_{\ell=1}^{L}$. 
\STATE For each candidate action $a^{(\ell)}_e$, invoke \texttt{Model} to sample $L'$ next states.
\STATE Write each next state $s^{(\ell,\ell')}_e$ to $c(a^{(\ell)}_e)$, i.e., $c(a^{(\ell)}_e) \leftarrow \{s^{(\ell,\ell')}_{e}\}_{\ell'=1}^{L'}$.
\STATE For each generated state $s^{(\ell,\ell')}_e$, invoke \texttt{Critic} to evaluate the expected cumulative future reward and update the estimated value $\widehat{V}$ for all ancestor nodes.\hfill (Optional)
     \ENDFOR
     \STATE Set $s^\dagger_0\leftarrow n$ and $i\leftarrow 0$.
     \WHILE{$s^\dagger_i$ is not a leaf node, i.e., $c(s^\dagger_i) \neq \emptyset$,}
        \STATE Invoke \texttt{Critic} to evaluate the expected cumulative future reward and select the child node $a^\dagger_{i+1}$ in $c(s^\dagger_i)$ with the highest value $\widehat{Q}(s^\dagger_i,a^\dagger_i)$.
        \STATE Set $s^\dagger_{i+1}$ as a child node in $c(a^\dagger_i)$ and $i\leftarrow i+1$.
    \ENDWHILE
    \STATE \textbf{output}: The initial action $a_0^\dagger$ of the selected rollout $(s_0^\dagger,a_0^\dagger,\ldots,s_i^\dagger, a_i^\dagger)$.
\end{algorithmic}
\end{algorithm}

\section{Main Proofs}\label{app:proof_main}

\subsection{Proof of Proposition \ref{prop:viisepsopt}}\label{prop:C2}
\begin{proof}[Proof of Proposition \ref{prop:viisepsopt}]
We now prove that the value iteration algorithm with a truncated horizon $U$ (Algorithm \ref{alg: fvi}) satisfies the definition of the $\eps$-optimal planner (Definition \ref{def: epsoptplan}), where $U$ is dependent on $\eps$. For notational simplicity, we denote $\max_{s\in\cS}$ and $\max_{a\in\cA}$ as $\max_{s}$ and $\max_a$.

Let 
\begin{align}
    \epsilon^\dagger = \max_{s,a}\bigl|Q^{(1)}(s,a)-r(s,a) - \gamma(P V^{(1)})(s,a)\bigr|.\label{eq:eps}
\end{align}
Note that the convergence analysis of the value iteration algorithm in  \citet{sutton2018reinforcement} gives
\begin{align*}
    \max_{s,a}\bigl|
Q^{(1)}(s,a) - Q^{(2)}(s,a)\bigr| \le \gamma^{U-2}\max_{s,a}\bigl|
Q^{(U-1)}(s,a) - Q^{(U)}(s,a)\bigr|,
\end{align*}
which implies 
\begin{align}
    \max_{s,a}\bigl|Q^{(1)}(s,a)-Q^{(2)}(s,a)\bigr|\le \gamma^{U-2}L.\label{eq:p_conv}
\end{align}
We have
\begin{align}
    \epsilon^\dagger & =  \max_{s,a}\bigl|Q_\theta^{(1)}(s,a) -r(s,a) - \gamma(P V^{(2)})(s,a)\notag\\
    &\qquad\qquad + \gamma \mathbb{E}_{s^\prime\sim P(\cdot\mid s,a)}\bigl[V^{(1)}(s^\prime)-V^{(2)}(s^\prime)\bigr]\bigr|\notag\\
    &=\gamma\cdot\max_{s,a}\bigl| \mathbb{E}_{s^\prime\sim P(\cdot\mid s,a)}\bigl[V^{(1)}(s^\prime)-V^{(2)}(s^\prime)\bigr]\bigr|\notag\\
     &=\gamma\cdot \max_{s,a}\bigl|\mathbb{E}_{s^\prime\sim P(\cdot\mid s,a)}\bigl[\max_{a}Q^{(1)}(s^\prime,a)-\max_{a}Q^{(2)}(s^\prime,a)\bigr]\bigr|\notag\\
     &\le\gamma\cdot\max_{s,a}\bigl| \mathbb{E}_{s^\prime\sim P(\cdot\mid s,a)}\bigl[\max_{a}\bigl|Q^{(1)}(s^\prime,a)-Q^{(2)}(s^\prime,a)\bigr|\bigr]\bigr|\notag\\
    &\le \gamma^{U-1}L,\label{eq:eps2}
\end{align}
where the first and third equalities are based on Algorithm \ref{alg: fvi}, the second last inequality uses the contraction property of the maximum operator, and the last inequality uses \eqref{eq:p_conv}. To let $\eps^\dagger < \eps$, it suffices to set $U\ge 1+\lceil \log_{\gamma}(\eps/L)\rceil$.  Note that the policy $\pi$ returned by Algorithm \ref{alg: fvi} satisfies $\pi(s)=\operatornamewithlimits{argmax}_a Q^{(1)}(s,a)$. 
Thus, we prove Proposition \ref{prop:viisepsopt}.
\end{proof}
\subsection{LLMs with Posterior Alignments Perform BMA}
\begin{proof}
    [Proof of Proposition \ref{prop:llm_bma}]\label{pf:llm_bma}
    Recall that $P_{\texttt{LLM}(\mathcal{D})}$ and $r_{\texttt{LLM}(\mathcal{D})}$ are the estimated transition kernel and reward function induced by $P^{\texttt{LLM}}$ that satisfies Assumption \ref{as:perfect}. 
    For any query state-action pair $(s,a)$ and in-context dataset $\cD$, it holds that
    \begin{align}
        (P_{\texttt{LLM}(\cD)}V)(s,a)& =  
        \int_\cS V(s^\prime) P_{\texttt{LLM}(\cD)}(\mathrm{d} s^\prime\mid s,a)
        \notag\\&=\int_\cS V(s^\prime)\left(\int_\Theta P_{\theta}(\mathrm{d}s^\prime\mid s,a)P_{\text{post}}(\mathrm{d}\theta\mid \cD)\right)\notag\\
        &=\int_\Theta P_{\text{post}}(\mathrm{d}\theta\mid \cD) \left(\int_\cS P_{\theta}(\mathrm{d}s^\prime\mid s,a) V(s^\prime) \right)\notag\\
        &=\mathbb{E}_{\theta\sim\mathbb{P}_{\text{post}}(\cdot\mid\cD)}[(P_\theta V)(s,a)],\label{eq:perfect1}
    \end{align}
      where the second equality uses Assumption \ref{as:perfect} (Posterior Alignment), the third equality uses Fubin's theorem, and  the last equality uses \eqref{eq:posterior}. For any query state-action pair $(s,a)$ and in-context dataset $\cD$, it holds that
      \begin{align}
        r_{\texttt{LLM}(\cD)}(s,a)& = \mathbb{E}_{P^{\texttt{LLM}}}[r\mid \cD,s,a]\notag\\
        &= \mathbb{E}_{P_{\text{post}}}[r\mid \cD,s,a]\notag\\
        &=\mathbb{E}_{\theta\sim \mathbb{P}_\text{post}(\cdot\mid \cD)} [r_{\theta}( s,a)],\label{eq:perfect2}
    \end{align}
     where the second equality uses Assumption \ref{as:perfect} (Posterior Alignment) and the last equality uses \eqref{eq:posterior}. 
By the linearity of expectation, we combine \eqref{eq:perfect1} and \eqref{eq:perfect2} to obtain
     \begin{align*}
         r_{\texttt{LLM}(\cD)}(s,a)+\gamma\cdot(P_{\texttt{LLM}(\cD)}V)(s,a) &= \mathbb{E}_{\theta\sim \mathbb{P}_\text{post}(\cdot\mid \cD)} [ r_\theta(s,a)+ \gamma\cdot(P_{\theta} V)(s,a)] \notag\\
         &=
         \mathbb{E}_{\theta\sim \mathbb{P}_\text{post}(\cdot\mid \cD)} [ (B_{\theta} V)(s,a)],
     \end{align*}
     where the last equality uses the definition of $B_\theta$.
     Thus, we finish the proof of Proposition \ref{prop:llm_bma}.
\end{proof}
\subsection{Contraction Property of the Posterior Variance}
\begin{prop}
    [Contraction Property of the Posterior Variance]  Under Assumptions   \ref{as:var},  the posterior variance in Algorithms \ref{alg: theory_gen}, \ref{alg: theory_bonus}, and \ref{alg: theory} \label{prop:pred_red} satisfies the following two properties:
    \begin{align*}\normalfont
    &\text{(i)}\qquad\mathbb{V}_{\theta\sim \mathbb{P}_{t_k}}\bigl[(B_{\theta}  V_t)(s,a)\bigr|\cD_{t_k}\bigr] \le 2L^2 \cdot  I(\theta;\xi_{(s,a)}\mid \cD_{t_k})\\
&\textrm{(ii)}\qquad\mathbb{E}\biggl[\sum_{k=0}^{K-1}\mathbb{E}_{\pi^k}\Bigl[\sum_{t=t_k}^{t_{k+1}-1} \mathbb{V}_{\theta\sim \mathbb{P}_{t_k}}\bigl[(B_{\theta}  V_t)(s_t,a_t)\bigr|\cD_{t_k}\bigr]\Bigr]\biggr]\le 8\eta L^2\cdot\mathbb{E}[H_0-H_T],
    \end{align*}
where we denote the upper bound of the sum of any value function and the reward by a positive constant $L$, that is, $|r+V(s)|\le L$ for any reward $r$, state $s$, and estimated value function $V$.
\end{prop}

\begin{proof}[Proof of Proposition \ref{prop:pred_red}]
    \label{pf:prop:pred_red}
    We begin with the proof of the first property in Proposition \ref{prop:pred_red}. 
    Recall the definition that  $\xi_{(s,a)}$ denotes random variables $(s^\prime,r)$ in the underlying MDP given the current state $s$ and action $a$. Define that $g_t(\xi_{(s,a)})= (r + V_t(s^\prime))/(2L)$. Since the sum of any reward and value function is bounded by $L$, we know that $|g_t|\le 1/2$. for any $t_k\le t<t_{k+1}$, we have
   \begin{align}
   2L\cdot \mathbb{E}[g_t(\xi_{(s,a)})\mid \theta,\cD_{t_k}]&=(B_{\theta}V_t)(s,a)\notag\\
    2L\cdot\mathbb{E}[g_t(\xi_{(s,a)})\mid\cD_{t_k}]&=\mathbb{E}_{\theta\sim \mathbb{P}_{t}}\bigl[(B_{\theta}V_t)(s,a)\bigr],\label{eq:12333}
   \end{align}
   for any query state $s$ and action $a$. 
   By the variational form of total variation (TV) distance $d_{\text{TV}}$, we have
   \begin{align}
       d^2_{\text{TV}}(\mathbb{P}(\xi_{(s,a)}\mid \theta,\mathcal{D}_t)\|\mathbb{P}(\xi_{(s,a)}\mid \mathcal{D}_t)) &= \Bigl(\sup_{g:|g|\le 1/2} \mathbb{E}[g(\xi_{(s,a)})\mid\theta,\cD_{t_k}] - \mathbb{E}[g(\xi_{(s,a)})\mid\cD_{t_k}]\Bigr)^2\notag\\
       &\ge  \bigl(\mathbb{E}[g_t(\xi_{(s,a)})\mid\theta,\cD_{t_k}] - \mathbb{E}[g_t(\xi_{(s,a)})\mid\cD_{t_k}]\bigr)^2\notag\\
       &= \frac{1}{4L^2} \cdot\bigl((B_{\theta}V_t)(s,a)-\mathbb{E}_{\theta\sim \mathbb{P}_{t}}\bigl[(B_{\theta}V_t)(s,a)\bigr]\bigr)^2,\label{eq:12334}
   \end{align}
   where the last equality is the result of \eqref{eq:12333}. By taking the expectation with respect to $\theta\sim \mathbb{P}_{t}$ on \eqref{eq:12334}, we have
   \begin{align}
       \mathbb{V}_{\theta\sim \mathbb{P}_{t_k}}\bigl[(B_{ {\theta}}V_t)(s,a)\bigr|\cD_{t_k}\bigr]&=\mathbb{E}_{\theta\sim \mathbb{P}_{t}}\Bigl[\bigl((B_{\theta}V_t)(s,a)-\mathbb{E}_{\theta\sim \mathbb{P}_{t}}\bigl[(B_{\theta}V_t)(s,a)\bigr]\bigr)^2\Bigr]\notag\\
       &\le 4L^2 \cdot \mathbb{E}_{\theta\sim \mathbb{P}_{t}}\bigl[ d^2_{\text{TV}}(\mathbb{P}(\xi_{(s,a)}\mid \theta,\mathcal{D}_{t_k})\|\mathbb{P}(\xi_{(s,a)}\mid \mathcal{D}_{t_k}))\bigr]\notag\\
       &\le 2L^2\cdot \mathbb{E}_{\theta\sim \mathbb{P}_{t}}\bigl[d_{\text{KL}}(\mathbb{P}(\xi_{(s,a)}\mid \theta,\mathcal{D}_{t_k})\|\mathbb{P}(\xi_{(s,a)}\mid \mathcal{D}_{t_k}))\bigr]\notag\\
       &=2L^2\cdot \bigl(H(\xi_{(s,a)}\mid \cD_{t_k}) - H(\xi_{(s,a)}\mid\theta,\cD_{t_k})\bigr)\notag\\
       &= 2L^2\cdot I(\xi_{(s,a)};\theta\mid \cD_{t_k})\notag\\
        &= 2L^2\cdot I(\theta;\xi_{(s,a)}\mid \cD_{t_k})
       ,\label{eq:prop_2}
   \end{align}
   where the first equality uses the definition of variance, the first inequality uses \eqref{eq:12334} by taking the expectation with respect to $\theta\sim \mathbb{P}_{t}$, and the second inequality uses Pinsker's inequality. Here, the second equality uses the definition of entropy and the second last equality uses the definition of the information gain. Here, the last equality uses the fact that $I(X;Y) = I(Y;X)$ for any two random variables $X$ and $Y$. Thus, we finish the proof of the first property in Proposition \ref{prop:pred_red}.

Next, we prove the second property in Proposition \ref{prop:pred_red}. 
   By the fact that $a_t=\pi^k(s_t)$ for any $t_k\le t <t_{k+1}$, we have
   \begin{align*}
    &\mathbb{E}\Bigl[\sum_{k=0}^{K-1}\mathbb{E}_{\pi^k}\bigl[\sum_{t=t_k}^{t_{k+1}-1} \mathbb{V}_{\theta\sim \mathbb{P}_{t_k}}\bigl[(B_{ {\theta}}V_t)(s_t,a_t)\bigr|\cD_{t_k}\bigr]\bigr]\Bigr]\notag\\
       &\qquad =\mathbb{E}\Bigl[\sum_{k=0}^{K-1}\mathbb{E}_{\pi^k}\bigl[\sum_{t=t_k}^{t_{k+1}-1} \mathbb{V}_{\theta\sim \mathbb{P}_{t_k}}\bigl[(B_{ {\theta}}V_t)(s_t,\pi^k(s_t))\bigr|\cD_{t_k}\bigr]\bigr]\Bigr]\\
        &\qquad\le 2L^2\cdot \mathbb{E}\Bigl[\sum_{k=0}^{K-1}\mathbb{E}_{\pi^k}\bigl[\sum_{t=t_k}^{t_{k+1}-1} I(\theta;\xi_{(s_t,\pi^k(s_t))}\mid \cD_{t_k})\bigr]\Bigr],
   \end{align*}
   where the inequality invokes  \eqref{eq:prop_2}. 
   Under Assumption \ref{as:var} and the same switching condition in Algorithms  \ref{alg: theory_gen}, \ref{alg: theory_bonus}, and \ref{alg: theory}, we  have 
   \begin{align*}
       & \mathbb{E}\Bigl[\sum_{k=0}^{K-1}\mathbb{E}_{\pi^k}\bigl[\sum_{t=t_k}^{t_{k+1}-1} \mathbb{V}_{\theta\sim \mathbb{P}_{t_k}}\bigl[(B_{ {\theta}}V_t)(s_t,a_t)\bigr|\cD_{t_k}\bigr]\bigr]\Bigr]\notag\\
       &\qquad \le 8\eta L^2 \cdot\mathbb{E}\biggl[\sum_{k=0}^{K-1}\mathbb{E}_{\pi^k}\Bigl[\sum_{t=t_k}^{t_{k+1}-1} I(\theta;\xi_{(s_t,\pi^k(s_t))}\mid \cD_{t})\Bigr]\biggr]\notag\\
       &\qquad \le 8\eta L^2(H_0 - H_T),
   \end{align*}
   where the last inequality uses the chain rule of information gain. Thus, we finish the proof of the second property in Proposition \ref{prop:pred_red}.
\end{proof}
\subsection{Proof of Theorem \ref{thm:reg_conc}}
\begin{proof}
    [Proof of Theorem \ref{thm:reg_conc}]\label{pf:reg_conc}
Recall that we denote by $\pi^k = \pi_{t_k}$  and $V_t$ is the estimated value function returned by the $\eps$-optimal planner $\texttt{PL}^\eps$ in Algorithm \ref{alg: theory_gen}. 
By the definition of the Bayesian regret $\mathfrak{R}(T)$ and the tower property of the conditional expectation, we have
\begin{align}
    \mathfrak{R}(T)&= \mathbb{E}\Bigl[\sum_{k=0}^{K-1}\sum_{t=t_k}^{t_{k+1}-1} V_{\theta^\star}^{\pi^\star}(s_t) - V_{\theta^\star}^{\pi_t}(s_t)\Bigr]\notag\\
    &=\mathbb{E}\Bigl[\sum_{k=0}^{K-1}\mathbb{E}_{\pi^k}\Bigl[\sum_{t=t_k}^{t_{k+1}-1}    V_{\theta^\star}^{\pi^\star}(s_t) - V_{\theta^\star}^{\pi^k}(s_t)\Bigr]\Bigr]\notag\\
     &=\underbrace{\mathbb{E}\Bigl[\sum_{k=0}^{K-1}\mathbb{E}_{\pi^k}\Bigl[\sum_{t=t_k}^{t_{k+1}-1}    V_t(s_t)- V_{\theta^\star}^{\pi^k}(s_t)\Bigr]\Bigr]}_{\text{term (a)}}+\underbrace{\mathbb{E}\Bigl[\sum_{k=0}^{K-1}\mathbb{E}_{\pi^k}\Bigl[\sum_{t=t_k}^{t_{k+1}-1} V_{\theta^\star}^{\pi^\star}(s_t)-    V_t(s_t)\Bigr]\Bigr]}_{\text{term (b)}},\label{eq:conc:0}
\end{align}
where the first equality uses the fact that $\pi_t = \pi_{t_k}$ for any $t_k\le t <t_{k+1}$ in Algorithm \ref{alg: theory_gen}. By Definition \ref{def: epsoptplan}, we know that $V_t(s) = Q_t(s,\pi^k(s))$ for any $t_k\le t <t_{k+1}$. Then, we introduce the following performance difference lemma to bound terms (a) and (b)  in \eqref{eq:conc:0}, respectively.

    \begin{lemma}[Performance Difference]\normalfont    \label{lem:reg_dec}
    For an algorithm \texttt{ALG} with switching times $K$, estimated value functions $\{(   Q_t,    V_t)\}_{t=0}^{T-1}$, and the corresponding output policy $\{\pi^{k}\}_{k=0}^{K-1}$ for $T$-steps interaction. We assume that \texttt{ALG} switches to  the policy $\pi^k$ at the $t_k$-th timestep for the $k$-th switch and $  V_{t}(s) =    Q_t(s,\pi^k(s))$ for any $s\in\mathcal{S}$ and $k<K$. 
    Then, we have two parts of performance difference results for \texttt{ALG} as follows,
   \begin{itemize}
       \item (Part I) \#\label{eq:www}
       &(1-\gamma)\cdot\mathbb{E}\Bigl[\sum_{k=0}^{K-1}\mathbb{E}_{\pi^k}\Bigl[\sum_{t=t_k}^{t_{k+1}-1}    V_t(s_t)- V_{\theta^\star}^{\pi^k}(s_t)\Bigr]\Bigr]\notag\\
       &\qquad= \underbrace{\mathbb{E}\Bigl[\sum_{k=0}^{K-1}\mathbb{E}_{\pi^k}\Bigl[\sum_{t=t_k}^{t_{k+1}-1}    Q_t(s_t,a_t) - \bigl(B_{\theta^\star}   V_t\bigr)(s_t,a_t)\Bigr]\Bigr]}_{\text{term (A): model prediction error}}\notag\\
       &\qquad +\underbrace{\mathbb{E}\Bigl[\sum_{k=0}^{K-1}\mathbb{E}_{\pi^k}\Bigl[\bigl(   V_t (s_{t_{k+1}}) - V^{\pi^k}_{\theta^\star}(s_{t_{k+1}})\bigr)-\bigl(   V_t(s_{t_k}) - V^{\pi^k}_{\theta^\star}(s_{t_k}) \bigr) \Bigr]\Bigr]}_{\text{term (B): value inconsistency}},
\#
\item (Part II) \#\label{eq:www99}
       &(1-\gamma)\cdot\mathbb{E}\Bigl[\sum_{k=0}^{K-1}\mathbb{E}_{\pi^k}\Bigl[\sum_{t=t_k}^{t_{k+1}-1} V_{\theta^\star}^{\pi^\star}(s_t)-    V_t(s_t)\Bigr]\Bigr]\notag\\
       &\qquad=   \underbrace{\mathbb{E}\Bigl[\sum_{k=0}^{K-1}\mathbb{E}_{\pi^k}\Bigl[\sum_{t=t_k}^{t_{k+1}-1} \mathbb{E}_{s\sim \nu^\star(\cdot\mid s_t)} \bigl[ ({B}_{\theta^\star}V_t)(s,\pi^\star(s)) -  Q_t(s,\pi^\star(s))\bigr]  \Bigr]\Bigr]}_{\text{term (A): model prediction error}}\notag\\
     & \qquad+  \underbrace{\mathbb{E}\Bigl[\sum_{k=0}^{K-1}\mathbb{E}_{\pi^k}\Bigl[\sum_{t=t_k}^{t_{k+1}-1}\mathbb{E}_{s^\prime\sim \nu^\star(\cdot\mid s_t)} \bigl[ (Q_t(s,\pi^\star(s)) -  Q_t(s,\pi^k(s))\bigr]  \Bigr]\Bigr]}_{\text{term (B): optimization error}}. 
\#
   \end{itemize}
where $\EE_{\pi^k}$ is taken with respect to the state-action sequence following $s_{t+1}\sim P_{\theta^\star}(\cdot \mid s_t, a_t)$ and $a_t = \pi^k(s_t)$ for any $t_k\leq t<t_{k+1}$, while $\EE$ is taken with respect  to the prior distribution $\mathbb{P}_0$ of $\theta^\star$, the iterative update of $\pi^k$, and the randomness of $\{(Q_t,V_t)\}_{t=0}^{T-1}$. Here, the optimal $\gamma$-discounted visitation measure $\nu^\star$ is defined in \eqref{eq:optimal_vis}.

\end{lemma}
\begin{proof}[Proof of Lemma \ref{lem:reg_dec}] See the detailed proof in Appendix \ref{mpf:reg_dec}.
\end{proof}

By the first part of Lemma \ref{lem:reg_dec}, we analyze term (a) as follows,
\begin{align}
    (1-\gamma)\cdot \text{term (a)}= & {\mathbb{E}\Bigl[\sum_{k=0}^{K-1}\mathbb{E}_{\pi^k}\Bigl[\sum_{t=t_k}^{t_{k+1}-1}    Q_t(s_t,a_t) - \bigl(B_{\theta^\star}   V_t\bigr)(s_t,a_t)\Bigr]\Bigr]}\notag\\
       &\qquad +{\mathbb{E}\Bigl[\sum_{k=0}^{K-1}\mathbb{E}_{\pi^k}\Bigl[\bigl(   V_t (s_{t_{k+1}}) - V^{\pi^k}_{\theta^\star}(s_{t_{k+1}})\bigr)-\bigl(   V_t(s_{t_k}) - V^{\pi^k}_{\theta^\star}(s_{t_k}) \bigr) \Bigr]\Bigr]}.\label{eq:222222_conc}
\end{align}

Recall that we define $(B_kV)(s,a) = r_{\texttt{LLM}(\cD_{t_k})}(s,a) + (P_{\texttt{LLM}(\cD_{t_k})}V)(s,a)$ for any $(s,a)$ and value $V$. 
By the definition of $\epsilon$-optimal planner (Definition \ref{def: epsoptplan}) and the planning procedure $(\pi_t, V_t)\leftarrow \texttt{PL}^\eps(P_{\texttt{LLM}(\cD_{t_k})},r_{\texttt{LLM}(\cD_{t_k})})$ in  Line \ref{line:pl_gen} of Algorithm \ref{alg: theory_gen}, we have
\begin{align}
    |Q_t(s_t,a_t) - \bigl(B_{\theta^\star}   V_t\bigr)(s_t,a_t)&|\le \eps + \bigl((B_{k}-B_{\theta^\star})   V_t\bigr)(s_t,a_t)\notag\\
    &=\eps + \bigl|\bigl((B_{k}-B_{\theta^\star})   V_t\bigr)(s_t,a_t)\bigr|\label{eq:22223_conc}
\end{align}
for any $t_k\le t<t_{k+1}$. Then, we plug \eqref{eq:22223_conc} into \eqref{eq:222222_conc} to obtain
\begin{align}
     (1-\gamma)\cdot \text{term (a)}\le & \underbrace{\mathbb{E}\Bigl[\sum_{k=0}^{K-1}\mathbb{E}_{\pi^k}\Bigl[\sum_{t=t_k}^{t_{k+1}-1}    \bigl|\bigl((B_k-B_{\theta^\star})   V_t\bigr)(s_t,a_t)\bigr|\Bigr]\Bigr]}_{\text{term (a1)}} + \ \epsilon \cdot T\notag\\
       &\qquad +\underbrace{\mathbb{E}\Bigl[\sum_{k=0}^{K-1}\mathbb{E}_{\pi^k}\Bigl[\bigl(   V_t (s_{t_{k+1}}) - V^{\pi^k}_{\theta^\star}(s_{t_{k+1}})\bigr)-\bigl(   V_t(s_{t_k}) - V^{\pi^k}_{\theta^\star}(s_{t_k}) \bigr) \Bigr]\Bigr]}_{\text{term (a2)}}.\label{eq:222224}
\end{align}
 Under Assumption \ref{as:perfect} and \ref{as:var}, we have
\begin{align}
   \text{term (a1)}&= \mathbb{E}\biggl[\sum_{k=0}^{K-1}\mathbb{E}_{\pi^k}\Bigl[\sum_{t=t_k}^{t_{k+1}-1} \mathbb{E}_{\theta\sim \mathbb{P}_{t_k}}\Bigl[   \bigl|\bigl((B_k-B_{\theta})   V_t\bigr)(s_t,a_t)\bigr|\Big|\, \cD_{t_k} \Big]\Bigr]\biggr]\notag\\
   &\le \sqrt{T}\cdot\biggl(\mathbb{E}\biggl[\sum_{k=0}^{K-1}\mathbb{E}_{\pi^k}\Bigl[\sum_{t=t_k}^{t_{k+1}-1} \mathbb{E}_{\theta\sim \mathbb{P}_{t_k}}\Bigl[   \bigl|\bigl((B_k-B_{\theta})   V_t\bigr)(s_t,a_t)\bigr|^2\Big|\, \cD_{t_k} \Big]\Bigr]\biggr]\biggr)^{1/2}\notag\\
   &= \sqrt{T}\cdot\biggl(\mathbb{E}\biggl[\sum_{k=0}^{K-1}\mathbb{E}_{\pi^k}\Bigl[\sum_{t=t_k}^{t_{k+1}-1} \mathbb{V}_{\theta\sim \mathbb{P}_{t_k}}\bigl[   (B_kV_t)(s_t,a_t)\big|\, \cD_{t_k} \big]\Bigr]\biggr]\biggr)^{1/2},\label{eq:var_E_1}
\end{align}
where the first equality uses the tower property of the conditional expectation and the definition that the posterior distribution of $\theta^\star$ given $\cD_{t_k}$ is $\mathbb{P}_{t_k}$. Here, the first inequality uses Cauchy-Schwarz inequality and the last equality invokes Proposition \ref{prop:llm_bma} and the definition of variance. Under Assumption \ref{as:var}, we apply the second property in Proposition \ref{prop:pred_red} on the right-hand  side of \eqref{eq:var_E_1} to have
\begin{align}
    \text{term (a1)}\le 2\sqrt{2\eta}L\cdot\sqrt{\mathbb{E}[H_0-H_T]} \cdot\sqrt{T}.\label{eq:conc001}
\end{align} 

Since any value function is bounded by $L$, we bound term (a2) in \eqref{eq:222224} as follows,
\begin{align}
    \text{term (a2)} \le 4L\cdot\ \mathbb{E}[K].\label{eq:a2_1}
\end{align}

To characterize the upper bound of the switching times $K$, we introduce the following lemma.
\begin{lemma}
    [Upper bound of Switching Times] If $H_{t_{k}}-H_{t_{k+1}}\ge \log2$ for any $k<K$, then it holds that \label{lem:swtich_times}
    \begin{align*}
        K-1\le {(H_0 - H_{t_{K-1}})}/{\log2}\le (H_0 - H_{T})/\log2.
    \end{align*}
\end{lemma}
\begin{proof}[Proof of Lemma \ref{lem:swtich_times}]
   Since $H_{t_{k}}-H_{t_{k+1}}\ge \log2$ , we have 
\begin{align*}
H_0 - H_{t_{K-1}}=\sum_{k=0}^{K-2}H_{t_{k}}-H_{t_{k+1}} \ge (K-1)\cdot \log2,
\end{align*}
which implies 
\begin{align*}
K-1\le {(H_0 - H_{t_{K-1}})}/{\log2}\le (H_0 - H_{T})/\log2.
\end{align*}
Thus, we finish the proof of Lemma \ref{lem:swtich_times}.
\end{proof}
As the switching condition in Algorithm \ref{alg: theory} implies $H_{t_{k}}-H_{t_{k+1}}\ge \log2$, we apply Lemma \ref{lem:swtich_times} to have 
\begin{align}
    \mathbb{E}[K]\le 1+{\mathbb{E}[H_0 - H_{t_{K-1}}]}/{\log2}\le (H_0 - H_{T})/\log2.\label{eq:conc002}
\end{align}
Combining \eqref{eq:a2_1} and \eqref{eq:conc002}, we upper bound 
  term (a2) in \eqref{eq:222224} as
  \begin{align}
      \text{ term (a2)}\le 4L\cdot\mathbb{E}[K]\le 4L+\frac{4L\cdot\mathbb{E}[H_0 - H_{T}]}{\log2}.\label{eq:a2_2}
  \end{align}
Plugging \eqref{eq:conc001} and \eqref{eq:a2_2} into \eqref{eq:222224}, we have
\begin{align}
    (1-\gamma)\cdot \text{term (a)}\le 2\sqrt{2\eta}L\cdot\sqrt{\mathbb{E}[H_0-H_T]} \cdot\sqrt{T}+ \eps\cdot T+4L+\frac{4L\cdot\mathbb{E}[H_0 - H_{T}]}{\log2} .\label{eq:a_end}
\end{align}
Then, we invoke the second part of Lemma \ref{lem:reg_dec} to decompose term (b) in \eqref{eq:conc:0} as follows,
\begin{align}
    (1-\gamma)\cdot\text{term (b)}&={\mathbb{E}\Bigl[\sum_{k=0}^{K-1}\mathbb{E}_{\pi^k}\Bigl[\sum_{t=t_k}^{t_{k+1}-1}\mathbb{E}_{s\sim \nu^\star(\cdot\mid s_t)}\bigl[ \bigl(B_{\theta^\star}  V_t\bigr)(s,\pi^\star(s)) -     Q_t(s,\pi^\star(s)) \bigr] \Bigr]\Bigr]}\notag\\
       &\qquad+
       {\mathbb{E}\Bigl[\sum_{k=0}^{K-1}\mathbb{E}_{\pi^k}\Bigl[\sum_{t=t_k}^{t_{k+1}-1}  \mathbb{E}_{s\sim \nu^\star(\cdot\mid s_t)}\bigl[  Q_t(s,\pi^\star(s)) -     Q_t(s,\pi^k(s))\bigr]\Bigr]\Bigr]}\notag\\
       &\le \underbrace{\mathbb{E}\Bigl[\sum_{k=0}^{K-1}\mathbb{E}_{\pi^k}\Bigl[\sum_{t=t_k}^{t_{k+1}-1}\mathbb{E}_{s\sim \nu^\star(\cdot\mid s_t)}\bigl[\bigl|\bigl((B_{\theta^\star}-B_k)   V_t\bigr)(s,\pi^\star(s))\bigr| \bigr] \Bigr]\Bigr]}_{\text{term (b1)}}+\ \eps\cdot T\notag\\
      &\qquad+
\underbrace{\mathbb{E}\Bigl[\sum_{k=0}^{K-1}\mathbb{E}_{\pi^k}\Bigl[\sum_{t=t_k}^{t_{k+1}-1}  \mathbb{E}_{s\sim \nu^\star(\cdot\mid s_t)}\bigl[  Q_t(s,\pi^\star(s)) -     Q_t(s,\pi^k(s))\bigr]\Bigr]\Bigr]}_{\text{term (b2)}},\label{eq:conc_b}
\end{align}
where the inequality uses \eqref{eq:22223_conc}. For term (b1) in \eqref{eq:conc_b}, we use the tower property of the conditional expectation to obtain
\begin{align}
    \text{term (b1)}& = 
   \mathbb{E}\Biggl[\sum_{k=0}^{K-1}\mathbb{E}_{\pi^k}\Biggl[\sum_{t=t_k}^{t_{k+1}-1} \mathbb{E}_{\theta^\star\sim \mathbb{P}_{t_k}}\biggl[\frac{\mathbb{E}_{s\sim\nu^\star(\cdot\mid s_t)}\bigl[\big|\bigl((B_{k}-B_{\theta^\star})  V_t\bigr)(s,\pi^\star(s))\big|\bigr]}{\bigl((B_{k}-B_{\theta^\star})  V_t\bigr)(s_t,\pi^k(s_t))}
   \\\notag &\qquad\cdot \bigl((B_{k}-B_{\theta^\star})  V_t\bigr)(s_t,\pi^k(s_t))\Bigr|\cD_{t_k}\biggr]\Biggr]\Biggr]\notag\\
   &= {\mathbb{E}\biggl[\sum_{k=0}^{K-1}\mathbb{E}_{\pi^k}\biggl[\sum_{t=t_k}^{t_{k+1}-1} \mathbb{E}_{\theta^\star\sim \mathbb{P}_{t_k}}\Bigl[G_{t,k}(\theta^\star)\cdot \bigl((B_{k}-B_{\theta^\star})  V_t\bigr)(s_t,\pi^k(s_t))\Big|\,\cD_{t_k}\Bigr]\biggr]\biggr]},\label{eq:conc_ratio}
\end{align}
where we define 
\begin{align*}
    G_{t,k}(\theta^\star)=\frac{ \mathbb{E}_{s\sim\nu^\star(\cdot\mid s_t)}\bigl[\big|\bigl((B_{k}-B_{\theta^\star})  V_t\bigr)(s,\pi^\star(s))\big|\bigr]}{\bigl((B_{k}-B_{\theta^\star})  V_t\bigr)(s_t,\pi^k(s_t))}
\end{align*}
and $\mathbb{E}_{\theta^\star\mid \cD_{t_k}}[\cdot]=\mathbb{E}_{\theta^\star\sim \mathbb{P}_{t_k}}[\cdot\mid \cD_{t_k}]$ for notational simplicity.

Applying Cauchy Schwartz inequality on the left-hand side of \eqref{eq:conc_ratio} several times, we have 
\begin{align}
    \text{term (b1)}& \le \mathbb{E}\biggl[\sum_{k=0}^{K-1}\mathbb{E}_{\pi^k}\Bigl[\sum_{t=t_k}^{t_{k+1}-1} \Bigl({\mathbb{E}_{\theta^\star\mid \cD_{t_k}}\bigl[G_{t,k}^2(\theta^\star)\bigr]}\Bigr)^{1/2}\\\notag &\qquad\cdot
    \Bigl({\mathbb{E}_{\theta^\star\mid \cD_{t_k}}\Bigl[ \bigl|\bigl((B_{k}-B_{\theta^\star})  V_t\bigr)(s_t,\pi^k(s_t))\bigr|^2\Bigr]}\Bigl)^{1/2}\Bigr]\biggr]\notag\\
    &\le \mathbb{E}\biggl[\sum_{k=0}^{K-1}\mathbb{E}_{\pi^k}\Bigl[ \Bigl({\sum_{t=t_k}^{t_{k+1}-1}\mathbb{E}_{\theta^\star\mid \cD_{t_k}}\bigl[G_{t,k}^2(\theta^\star)\bigr]\Bigr)^{1/2}}\\\notag &\qquad\cdot
    \Bigl({\sum_{t=t_k}^{t_{k+1}-1}\mathbb{E}_{\theta^\star\mid \cD_{t_k}}\Bigl[ \bigl|\bigl((B_{k}-B_{\theta^\star})  V_t\bigr)(s_t,\pi^k(s_t))\bigr|^2\Bigr]\Bigr)^{1/2}}\Bigr]\biggr]\notag\\
    &\le  \biggl({{\mathbb{E}\Bigl[\sum_{k=0}^{K-1}\mathbb{E}_{\pi^k}\Bigl[\sum_{t=t_k}^{t_{k+1}-1}\mathbb{E}_{\theta^\star\mid \cD_{t_k}}\bigl[G_{t,k}^2(\theta^\star)\bigr]\Bigr]\Bigr]}\biggr)^{1/2}}\notag\\&\qquad\cdot
\biggl({\mathbb{E}\biggl[\sum_{k=0}^{K-1}\mathbb{E}_{\pi^k}\Bigl[\sum_{t=t_k}^{t_{k+1}-1} \mathbb{E}_{\theta^\star\mid \cD_{t_k}}\Bigl[ \bigl|\bigl((B_{k}-B_{\theta^\star})  V_t\bigr)(s_t,\pi^k(s_t))\bigr|^2\Bigr]\Bigr]\biggr] \biggr)^{1/2}}\notag\\
    &\le \sqrt{\kappa^2 \cdot T} \cdot\biggl({\mathbb{E}\biggl[\sum_{k=0}^{K-1}\mathbb{E}_{\pi^k}\Bigl[\sum_{t=t_k}^{t_{k+1}-1} \mathbb{E}_{\theta^\star\mid \cD_{t_k}}\Bigl[ \bigl|\bigl((B_{k}-B_{\theta^\star})  V_t\bigr)(s_t,\pi^k(s_t))\bigr|^2\Bigr]\Bigr]\biggr] \biggr)^{1/2}},\notag
\end{align}
where the first three inequalities are all based on Cauchy Schwartz inequality and the last inequality uses the definition of $\kappa$ in Assumption \ref{as:conc_coef}. 
Under Assumptions \ref{as:perfect} and \ref{as:var},  we have
\begin{align}
     \text{term (b1)}& \le\sqrt{\kappa^2 \cdot T} \cdot\biggl(\mathbb{E}\biggl[\sum_{k=0}^{K-1}\mathbb{E}_{\pi^k}\Bigl[\sum_{t=t_k}^{t_{k+1}-1} \mathbb{V}_{\theta\sim \mathbb{P}_{t_k}}\bigl[   (B_\theta V_t)(s_t,a_t)\big|\, \cD_{t_k} \big]\Bigr]\biggr] \biggr)^{1/2}\notag\\
     &\le 2\sqrt{2\eta}L\kappa\cdot\sqrt{\mathbb{E}[H_0-H_T]}\cdot\sqrt{T},\label{eq:b11}
\end{align}
where the first inequality invokes Proposition \ref{prop:llm_bma} and the definition of variance. Here, the second inequality invokes Proposition \ref{prop:pred_red}. 
By the definition of $\eps$-optimal planner (Definition \ref{def: epsoptplan}), we know term (b2) in \eqref{eq:conc_b} is non-positive. 
Then, plugging \eqref{eq:b11} into \eqref{eq:conc_b}, we have
\begin{align}
    (1-\gamma)\cdot\text{term (b)}&\le 2\sqrt{2\eta}L\kappa\cdot\sqrt{\mathbb{E}[H_0-H_T]}\cdot\sqrt{T}+\eps\cdot T.\label{eq:b_end}
\end{align}
Combining \eqref{eq:conc:0}, \eqref{eq:a_end}, and \eqref{eq:b_end}, we have
\begin{align*}
    \mathfrak{R}(T)&=\frac{1}{1-\gamma}\cdot\bigl(\text{term (a)} +\text{term (b)}\bigr)\\
    &\le \frac{2\sqrt{2}(\kappa+1)L\cdot\sqrt{\mathbb{E}[H_0-H_T]}}{1-\gamma}\cdot\sqrt{T} +\frac{2\eps}{1-\gamma}\cdot T + \frac{4L}{1-\gamma} +\frac{4L\cdot\mathbb{E}[H_0 - H_{T}]}{(1-\gamma)\log2}\\
&=\mathcal{O}\Biggl(\frac{(\kappa+1)L\cdot\sqrt{\mathbb{E}[H_0-H_T]}}{1-\gamma}\cdot\sqrt{T} +\frac{\eps}{1-\gamma}\cdot T + \frac{L\cdot\mathbb{E}[H_0 - H_{T}]}{1-\gamma}\Biggr).
\end{align*}
Thus, we finish the proof of Theorem \ref{thm:reg_conc}.
\end{proof}

\subsection{Proof of Theorem \ref{thm:reg_bonus}}\label{pf:thm:reg_bonus}
\begin{proof}
    [Proof of Theorem \ref{thm:reg_bonus}]
    Recall that we denote by $\pi^k = \pi_{t_k}$, and $V_t$ is the estimated value function returned by the $\eps$-optimal planner $\texttt{PL}^\eps$ in Algorithm \ref{alg: theory_gen}. 
By the definition of the Bayesian regret $\mathfrak{R}(T)$ and the tower property of the conditional expectation, we have
\begin{align}
    \mathfrak{R}(T)&= \mathbb{E}\Bigl[\sum_{k=0}^{K-1}\sum_{t=t_k}^{t_{k+1}-1} V_{\theta^\star}^{\pi^\star}(s_t) - V_{\theta^\star}^{\pi_t}(s_t)\Bigr]\notag\\
    &=\mathbb{E}\Bigl[\sum_{k=0}^{K-1}\mathbb{E}_{\pi^k}\Bigl[\sum_{t=t_k}^{t_{k+1}-1}    V_{\theta^\star}^{\pi^\star}(s_t) - V_{\theta^\star}^{\pi^k}(s_t)\Bigr]\Bigr]\notag\\
     &=\underbrace{\mathbb{E}\Bigl[\sum_{k=0}^{K-1}\mathbb{E}_{\pi^k}\Bigl[\sum_{t=t_k}^{t_{k+1}-1}    V_t(s_t)- V_{\theta^\star}^{\pi^k}(s_t)\Bigr]\Bigr]}_{\text{term (a)}}+\underbrace{\mathbb{E}\Bigl[\sum_{k=0}^{K-1}\mathbb{E}_{\pi^k}\Bigl[\sum_{t=t_k}^{t_{k+1}-1} V_{\theta^\star}^{\pi^\star}(s_t)-    V_t(s_t)\Bigr]\Bigr]}_{\text{term (b)}},\label{eq:conc:0_bonus}
\end{align}
where the first equality uses the fact that $\pi_t = \pi_{t_k}$ for any $t_k\le t <t_{k+1}$ in Algorithm \ref{alg: theory_gen}. By Definition \ref{def: epsoptplan}, we know that $V_t(s) = Q_t(s,\pi^k(s))$ for any $t_k\le t <t_{k+1}$. Then, we apply the first part of Lemma \ref{lem:reg_dec} to analyze term (a) as follows,
\begin{align*}
    (1-\gamma)\cdot \text{term (a)}= & {\mathbb{E}\Bigl[\sum_{k=0}^{K-1}\mathbb{E}_{\pi^k}\Bigl[\sum_{t=t_k}^{t_{k+1}-1}    Q_t(s_t,a_t) - \bigl(B_{\theta^\star}   V_t\bigr)(s_t,a_t)\Bigr]\Bigr]}\notag\\
       &\qquad +{\mathbb{E}\Bigl[\sum_{k=0}^{K-1}\mathbb{E}_{\pi^k}\Bigl[\bigl(   V_t (s_{t_{k+1}}) - V^{\pi^k}_{\theta^\star}(s_{t_{k+1}})\bigr)-\bigl(   V_t(s_{t_k}) - V^{\pi^k}_{\theta^\star}(s_{t_k}) \bigr) \Bigr]\Bigr]}.
\end{align*}
Recall that we define $(B_kV)(s,a) = r_{\texttt{LLM}(\cD_{t_k})}(s,a) + (P_{\texttt{LLM}(\cD_{t_k})}V)(s,a)$ for any $(s,a)$ and value function $V$. 
By the definition of $\epsilon$-optimal planner (Definition \ref{def: epsoptplan}) and the planning procedure $(\pi_t, V_t)\leftarrow \texttt{PL}^\eps(P_{\texttt{LLM}(\cD_{t_k})},r_{\texttt{LLM}(\cD_{t_k}) }+\Gamma_k)$ in  Algorithm \ref{alg: theory_bonus}, we have
\begin{align}
    |Q_t(s,a) - \bigl(B_{\theta^\star}   V_t\bigr)(s_t,a_t)&|\le \eps + \bigl((B_{k}-B_{\theta^\star})   V_t\bigr)(s,a)+\Gamma_k(s,a)\label{eq:22224}
\end{align}
for any $t_k\le t<t_{k+1}$ and any $(s,a)\in\cS\times\cA$. Then, we plug \eqref{eq:22224} into \eqref{eq:222224} to obtain 
\begin{align}
     (1-\gamma)\cdot \text{term (a)}\le & \underbrace{\mathbb{E}\Bigl[\sum_{k=0}^{K-1}\mathbb{E}_{\pi^k}\Bigl[\sum_{t=t_k}^{t_{k+1}-1}    \bigl((B_k-B_{\theta^\star})   V_t\bigr)(s_t,a_t)\Bigr]\Bigr]}_{\text{term (a1)}} + \\&\qquad +\epsilon \cdot T+\underbrace{\mathbb{E}\Bigl[\sum_{k=0}^{K-1}\mathbb{E}_{\pi^k}\Bigl[\sum_{t=t_k}^{t_{k+1}-1}    \Gamma_k(s_t,a_t)\bigr|\Bigr]\Bigr]}_{\text{term (a2)}} \notag\\
       &\qquad +\underbrace{\mathbb{E}\Bigl[\sum_{k=0}^{K-1}\mathbb{E}_{\pi^k}\Bigl[\bigl(   V_t (s_{t_{k+1}}) - V^{\pi^k}_{\theta^\star}(s_{t_{k+1}})\bigr)-\bigl(   V_t(s_{t_k}) - V^{\pi^k}_{\theta^\star}(s_{t_k}) \bigr) \Bigr]\Bigr]}_{\text{term (a3)}}\label{eq:bonus_a}
\end{align}
Under Assumption \ref{as:perfect}, we have
\begin{align}
   \text{term (a1)}&= \mathbb{E}\biggl[\sum_{k=0}^{K-1}\mathbb{E}_{\pi^k}\Bigl[\sum_{t=t_k}^{t_{k+1}-1} \mathbb{E}_{\theta\sim \mathbb{P}_{t_k}}\Bigl[   \bigl((B_k-B_{\theta})   V_t\bigr)(s_t,a_t)\Big|\, \cD_{t_k} \Big]\Bigr]\biggr]\notag\\
   &\le \sqrt{T}\cdot\biggl(\mathbb{E}\biggl[\sum_{k=0}^{K-1}\mathbb{E}_{\pi^k}\Bigl[\sum_{t=t_k}^{t_{k+1}-1} \mathbb{E}_{\theta\sim \mathbb{P}_{t_k}}\Bigl[  \bigl| \bigl((B_k-B_{\theta})   V_t\bigr)(s_t,a_t)\bigr|^2\Big|\, \cD_{t_k} \Big]\Bigr]\biggr]\biggr)^{1/2}\notag\\
   &= \sqrt{T}\cdot\biggl(\mathbb{E}\biggl[\sum_{k=0}^{K-1}\mathbb{E}_{\pi^k}\Bigl[\sum_{t=t_k}^{t_{k+1}-1} \mathbb{V}_{\theta\sim \mathbb{P}_{t_k}}\bigl[   (B_\theta V_t)(s_t,a_t)\big|\, \cD_{t_k} \big]\Bigr]\biggr]\biggr)^{1/2},\label{eq:var_E_12aa}
\end{align}
where the first equality uses the tower property of the conditional expectation and the definition that the posterior distribution of $\theta^\star$ given $\cD_{t_k}$ is $\mathbb{P}_{t_k}$. Here, the first inequality uses Cauchy-Schwarz inequality and the last equality invokes Proposition \ref{prop:llm_bma} and the definition of variance. Under Assumption \ref{as:var}, we apply the second property in Proposition \ref{prop:pred_red} on the right-hand side of \eqref{eq:var_E_12aa} to have
\begin{align}
    \text{term (a1)}\le 2\sqrt{2\eta}L\cdot\sqrt{\mathbb{E}[H_0-H_T]} \cdot\sqrt{T}.\label{eq:bonus_a1}
\end{align} 

 Recall that the bonus $\Gamma_k$ used  in Algorithm \ref{alg: theory_bonus} is defined by $\Gamma_k(s,a)=\sqrt{2}L\cdot\sqrt{I(\theta;\xi_{(s,a)}\mid \cD_{t_k})}$ . For term (a2) in \eqref{eq:bonus_a}, we have
 \begin{align}
     \text{term (a2)}&\le \sqrt{T}\cdot\biggl(\mathbb{E}\Bigl[\sum_{k=0}^{K-1}\mathbb{E}_{\pi^k}\bigl[\sum_{t=t_k}^{t_{k+1}-1}    \Gamma_k^2(s_t,a_t)\bigr]\Bigr]\biggr)^{1/2}\notag\\&=\sqrt{2}L\cdot\sqrt{T}\cdot\biggl(\mathbb{E}\Bigl[\sum_{k=0}^{K-1}\mathbb{E}_{\pi^k}\bigl[\sum_{t=t_k}^{t_{k+1}-1}    I(\theta;\xi_{(s_t,a_t)}\mid\cD_{t_k})\bigr]\Bigr]\biggr)^{1/2}\notag\\
     &\le 2\sqrt{2\eta}L\cdot\sqrt{T}\cdot\biggl(\mathbb{E}\Bigl[\sum_{k=0}^{K-1}\mathbb{E}_{\pi^k}\bigl[\sum_{t=t_k}^{t_{k+1}-1}    I(\theta;\xi_{(s_t,a_t)}\mid\cD_{t})\bigr]\Bigr]\biggr)^{1/2},\label{eq:bonus_a21}
 \end{align}
 where the equality uses the definition of $\Gamma_k$ in Algorithm \ref{alg: theory_bonus}. Here, the last inequality invokes Assumption \ref{as:var} and the switching condition in Algorithm \ref{alg: theory_bonus}. 

 As $a_t = \pi^k(s_t)$, we further bound the right-hand side of \eqref{eq:bonus_a21} as follows,
 \begin{align}
     \text{term (a2)}&\le 2\sqrt{2\eta}L\cdot\sqrt{T}\cdot\biggl(\mathbb{E}\Bigl[\sum_{k=0}^{K-1}\mathbb{E}_{\pi^k}\bigl[\sum_{t=t_k}^{t_{k+1}-1}    I(\theta;\xi_{(s_t,\pi^k(s_t))}\mid\cD_{t})\bigr]\Bigr]\biggr)^{1/2}\notag\\
&=2\sqrt{2\eta}L\cdot\sqrt{\mathbb{E}[H_0-H_T]}\cdot \sqrt{T},\label{eq:bonus_a2}
 \end{align}
 where the last inequality uses the chain rule of the information gain. 
Using the fact that any value function is bounded by $L$, we bound term (a3) in \eqref{eq:222224} as
\begin{align*}
    \text{term (a3)}\le 4L\cdot\ \mathbb{E}[K].
\end{align*}
As the switching condition in Algorithm \ref{alg: theory_bonus} implies $H_{t_{k}}-H_{t_{k+1}}\ge \log2$, we apply Lemma \ref{lem:swtich_times} to have 
\begin{align}
    \text{term (a3)}\le 4L+\frac{4L\cdot\mathbb{E}[H_0 - H_{T}]}{\log2} .\label{eq:bonus_a3}
\end{align}

Plugging \eqref{eq:bonus_a1}, \eqref{eq:bonus_a2} and \eqref{eq:bonus_a3} into \eqref{eq:222224}, we have
\begin{align}
    (1-\gamma)\cdot \text{term (a)}\le 4\sqrt{2}L\cdot\sqrt{\mathbb{E}[H_0-H_T]} \cdot\sqrt{T}+ \eps\cdot T+4L+\frac{4L\cdot\mathbb{E}[H_0 - H_{T}]}{\log2} .\label{eq:bonus_a_end}
\end{align}
Then, we invoke the second part of Lemma \ref{lem:reg_dec} to decompose term (b) in \eqref{eq:conc:0_bonus} as follows,
\begin{align}
    (1-\gamma)\cdot\text{term (b)}&={\mathbb{E}\Bigl[\sum_{k=0}^{K-1}\mathbb{E}_{\pi^k}\Bigl[\sum_{t=t_k}^{t_{k+1}-1}\mathbb{E}_{s\sim\nu^\star(\cdot\mid s_t)}\bigl[\bigl(B_{\theta^\star}   V_t\bigr)(s,\pi^\star(s)) -     Q_t(s,\pi^\star(s))  \bigr]\Bigr]\Bigr]}\notag\\
      &\qquad+
       {\mathbb{E}\Bigl[\sum_{k=0}^{K-1}\mathbb{E}_{\pi^k}\Bigl[\sum_{t=t_k}^{t_{k+1}-1}  \mathbb{E}_{s\sim\nu^\star(\cdot\mid s_t)}\bigl[  Q_t(s,\pi^\star(s)) -     Q_t(s,\pi^k(s))\bigr]\Bigr]\Bigr]}\notag\\
       &\le \underbrace{\mathbb{E}\biggl[\sum_{k=0}^{K-1}\mathbb{E}_{\pi^k}\Bigl[\sum_{t=t_k}^{t_{k+1}-1}\mathbb{E}_{s\sim\nu^\star(\cdot\mid s_t)}\bigl[\bigl((B_{\theta^\star}-B_k)   V_t\bigr)(s,\pi^\star(s))-\Gamma_k(s,\pi^\star(s))  \bigr]\Bigr]\biggr]}_{\text{term (b1)}}\notag\\
     &\qquad+\eps\cdot T+
\underbrace{\mathbb{E}\biggl[\sum_{k=0}^{K-1}\mathbb{E}_{\pi^k}\Bigl[\sum_{t=t_k}^{t_{k+1}-1}  \mathbb{E}_{s\sim\nu^\star(\cdot\mid s_t)}\bigl[  Q_t(s,\pi^\star(s)) -     Q_t(s,\pi^k(s))\bigr]\Bigr]\biggr]}_{\text{term (b2)}},\label{eq:bonus_b}
\end{align}
where the inequality uses \eqref{eq:22224}. Under Assumption \ref{as:perfect}, we bound term (b1) in \eqref{eq:bonus_b} as follows,
\begin{align}
   \text{term (b1)}&= \mathbb{E}\Biggl[\sum_{k=0}^{K-1}\mathbb{E}_{\pi^k}\biggl[\sum_{t=t_k}^{t_{k+1}-1}\mathbb{E}_{s\sim\nu^\star(\cdot\mid s_t)}\biggl[ \mathbb{E}_{\theta\sim \mathbb{P}_{t_k}}\Bigl[   \bigl((B_k-B_{\theta})   V_t\bigr)(s,\pi^\star(s))\Big|\, \cD_{t_k} \Bigr]-\Gamma_k(s,\pi^\star(s))\biggr]\biggr]\Biggr]\notag\\
   &\le \mathbb{E}\Biggl[\sum_{k=0}^{K-1}\mathbb{E}_{\pi^k}\biggl[\sum_{t=t_k}^{t_{k+1}-1}\mathbb{E}_{s\sim\nu^\star(\cdot\mid s_t)}\biggl[ \sqrt{\mathbb{E}_{\theta\sim \mathbb{P}_{t_k}}\Bigl[  \bigl| \bigl((B_k-B_{\theta})   V_t\bigr)(s,\pi^\star(s))\bigr|^2\Big|\, \cD_{t_k} \Big]}\notag\\
   &\qquad
   -\Gamma_k(s,\pi^\star(s)) \biggr]\biggr]\Biggr]\notag\\
   &= \mathbb{E}\Biggl[\sum_{k=0}^{K-1}\mathbb{E}_{\pi^k}\biggl[\sum_{t=t_k}^{t_{k+1}-1} \mathbb{E}_{s\sim\nu^\star(\cdot\mid s_t)}\biggl[\sqrt{\mathbb{V}_{\theta\sim \mathbb{P}_{t_k}}\bigl[   (B_\theta V_t)(s,\pi^\star(s))\big|\, \cD_{t_k} \big]}-\Gamma_k(s,\pi^\star(s))\biggr]\biggr]\Biggr],\label{eq:var_E_12bb}
\end{align}
where the first equality uses the tower property of the conditional expectation and the definition that the posterior distribution of $\theta^\star$ given $\cD_{t_k}$ is $\mathbb{P}_{t_k}$. Here, the first inequality uses Cauchy-Schwarz inequality and the last equality invokes Proposition \ref{prop:llm_bma} and the definition of variance. Under Assumption \ref{as:var}, we apply the first property in Proposition \ref{prop:pred_red}  to have
\begin{align}
    \sqrt{\mathbb{V}_{\theta\sim \mathbb{P}_{t_k}}\bigl[   (B_\theta V_t)(s,\pi^\star(s))\big|\, \cD_{t_k} \big]}-\Gamma_k(s,\pi^\star(s))&\le \sqrt{2}L\cdot\sqrt{I(\theta;\xi_{(s,\pi^\star(s))})}-\Gamma_k(s,\pi^\star(s))\notag\\&=\sqrt{2}L\cdot\sqrt{I(\theta;\xi_{(s,\pi^\star(s))})}-\sqrt{2}L\cdot\sqrt{I(\theta;\xi_{(s,\pi^\star(s))})}\notag\\
    &\le 0,\label{eq:optimistic}
\end{align} 
for any $t_k\le t<t_{k+1}$, $k<K$, and state $s\in\mathcal{S}$. Here, the equality uses the definition of $\Gamma_k$ in Algorithm \ref{alg: theory_bonus}. Plugging \eqref{eq:optimistic} into \eqref{eq:var_E_12bb}, we have
\begin{align}
    \text{term (b1)}\le 0.\label{eq:bonus_b1}
\end{align}
By the definition of $\eps$-optimal planner (Definition \ref{def: epsoptplan}), we know term (b2) in \eqref{eq:conc_b} is non-positive. 
Then, plugging \eqref{eq:bonus_b} into \eqref{eq:bonus_b}, we have
\begin{align}
    (1-\gamma)\cdot\text{term (b)}&\le \eps\cdot T.\label{eq:bonus_b_end}
\end{align}
Combining \eqref{eq:conc:0_bonus}, \eqref{eq:bonus_a_end}, and \eqref{eq:bonus_b_end}, we have
\begin{align*}
    \mathfrak{R}(T)&=\frac{1}{1-\gamma}\cdot\bigl(\text{term (a)} +\text{term (b)}\bigr)\\
    &\le \frac{4\sqrt{2}L\cdot\sqrt{\mathbb{E}[H_0-H_T]}}{1-\gamma}\cdot\sqrt{T} +\frac{2\eps}{1-\gamma}\cdot T + \frac{4L}{1-\gamma} +\frac{4L\cdot\mathbb{E}[H_0 - H_{T}]}{(1-\gamma)\log2}\\
&=\mathcal{O}\Biggl(\frac{L\cdot\sqrt{\mathbb{E}[H_0-H_T]}}{1-\gamma}\cdot\sqrt{T} +\frac{\eps}{1-\gamma}\cdot T + \frac{L\cdot\mathbb{E}[H_0 - H_{T}]}{1-\gamma}\Biggr).
\end{align*}
Thus, we finish the proof of Theorem \ref{thm:reg_bonus}.
\end{proof}
\subsection{Proof of Theorem \ref{thm:reg}} 
\label{pf:reg}
\begin{proof}[Proof of Theorem \ref{thm:reg}]

 For notational simplicity, we denote by  $\theta^k$ the corresponding parameter for the mechanism \texttt{LLM+PS} given $\cD_{t_k}$ in Algorithm \ref{alg: theory}, which satisfies
\begin{align}
    (B_{\theta^k}V)(s,a)=
    r_{\texttt{LLM+PS}(\cD_{t_k})}(s,a)+\gamma\cdot(P_{\texttt{LLM+PS}(\cD_{t_k})}V)(s,a), \label{theta^k}
\end{align}
 for any $k<K$, $(s,a)\in\cS\times\cA$, and value function $V$.  
 Recall the definition of optimal value $V_\theta^\star$ given the parameter $\theta$ in \eqref{eq:bellman_opt}.
  Under Assumption \ref{as:ts}, we know that $(B_{\theta^k}V)(s,a)\mid\cD_{t_k}$ and $(B_{\theta^\star}V)(s,a)\mid\cD_{t_k}$ follows the same distribution for any $k<K$, $(s,a)\in\cS\times\cA$, and value function $V$. By Bellman optimality equation in \eqref{eq:bellman_opt}, we have that $V_{\theta^k}^{\star}(s,a)\mid\cD_{t_k}$ and $V_{\theta^\star}^{\star}(s,a)\mid\cD_{t_k}$ follows the same distribution for any $k<K$, $(s,a)\in\cS\times\cA$, and value function $V$.  Recall that we denote by $\pi^k = \pi_{t_k}$ and $V_t$ is the estimated value function returned by the $\eps$-optimal planner $\texttt{PL}^\eps$ in Algorithm \ref{alg: theory}. 
By the definition of the Bayesian regret $\mathfrak{R}(T)$ and $\pi_t=\pi^k$ for any $t_k\le t<t_{k+1}$ and $k<K$  , we have
\begin{align}\label{eq_11}
    \mathfrak{R}(T)&= \mathbb{E}\Bigl[\sum_{k=0}^{K-1}\mathbb{E}_{\pi^k}\Bigl[\sum_{t=t_k}^{t_{k+1}-1} V_{\theta^\star}^{\pi^\star}(s_t) - V_{\theta^\star}^{\pi^k}(s_t)\Bigr]\Bigr]\notag\\
    & =\mathbb{E}\Bigl[\sum_{k=0}^{K-1}\mathbb{E}_{\pi^k}\Bigl[\sum_{t=t_k}^{t_{k+1}-1} V_{\theta^\star}^\star(s_t) - V_{\theta^\star}^{\pi^k}(s_t)\Bigr]\Bigr]\notag\\ &=\mathbb{E}\biggl[\sum_{k=0}^{K-1}\mathbb{E}_{\pi^k}\Bigl[\sum_{t=t_k}^{t_{k+1}-1} \mathbb{E}_{\theta^\star\sim \mathbb{P}_{t_k}}\bigl[V_{\theta^\star}^\star(s_t) - V_{\theta^\star}^{\pi^k}(s_t)\mid \cD_{t_k}\bigr]\Bigr]\biggr],
\end{align}
where the second equality uses the definition of optimal policy and \eqref{eq:bellman_opt}. Here, the last equality uses the tower property of the conditional expectation. Recall that $\theta^\star\mid\cD_{t_k}$ and ${\theta^k}\mid\cD_{t_k}$ follows the same distribution $\mathbb{P}_{t_k}$ for any $t_k\le t < t_{k+1}$, which implies 
\begin{align}
\mathbb{E}_{\theta^\star\sim \mathbb{P}_{t_k}}\bigl[V_{\theta^\star}^\star(s_t) - V_{\theta^\star}^{\pi^k}(s_t)\mid \cD_{t_k}\bigr]&=\mathbb{E}_{\theta^\star\sim \mathbb{P}_{t_k}}\bigl[V_{\theta^\star}^\star(s_t)\mid \cD_{t_k}\bigr] - \mathbb{E}_{\theta^\star\sim \mathbb{P}_{t_k}}\bigl[ V_{\theta^\star}^{\pi^k}(s_t)\mid \cD_{t_k}\bigr]\notag\\
&=\mathbb{E}_{\theta^k\sim \mathbb{P}_{t_k}}\bigl[V_{\theta^k}^\star(s_t)\mid \cD_{t_k}\bigr] - \mathbb{E}_{\theta^\star\sim \mathbb{P}_{t_k}}\bigl[ V_{\theta^\star}^{\pi^k}(s_t)\mid \cD_{t_k}\bigr]\notag\\
&=\mathbb{E}_{\theta^\star,\theta^k\sim \mathbb{P}_{t_k}}\bigl[V_{\theta^k}^\star(s_t) - V_{\theta^\star}^{\pi^k}(s_t)\mid \cD_{t_k}\bigr],\label{eq:eq_V}
\end{align}
where the first and the second inequalities use the linear property of the conditional expectation. Plugging \eqref{eq:eq_V} into \eqref{eq_11}, we have
\begin{align}\label{eq0}
        \mathfrak{R}(T) & =\mathbb{E}\biggl[\sum_{k=0}^{K-1}\mathbb{E}_{\pi^k}\Bigl[\sum_{t=t_k}^{t_{k+1}-1} \mathbb{E}_{\theta^\star,\theta^k\sim \mathbb{P}_{t_k}}\bigl[V_{\theta^k}^\star(s_t) - V_{\theta^\star}^{\pi^k}(s_t)\mid \cD_{t_k}\bigr]\Bigr]\biggr]\notag\\
       & =\mathbb{E}\Bigl[\sum_{k=0}^{K-1}\mathbb{E}_{\pi^k}\Bigl[\sum_{t=t_k}^{t_{k+1}-1} V_{\theta^k}^\star(s_t) - V_{\theta^\star}^{\pi^k}(s_t)\Bigr]\Bigr],
\end{align}
where the last equality uses the tower property of the conditional expectation.

Meanwhile, by Definition \ref{def: epsoptplan}, we have
\begin{align}
    \max_{s\in\cS }|V^{\star}_{\theta^k}(s_t)- V_t(s_t)|&=\max_{s\in\cS }|\max_{a\in\cA}Q^{\star}_{{\theta^k}}(s,a)- \max_a Q_t(s,a)|\notag\\
    & \le \max_{(s,a)\in\cS\times\cA}|Q^{\star}_{{\theta^k}}(s,a)-Q_t(s,a)|\notag\\
&=\max_{(s,a)\in\cS\times\cA}\bigl|\bigl((B_{\theta^k}V_{\theta^k}^\star)(s,a)-(B_{\theta^k}V_{t})(s,a)\bigr)+\bigl((B_{\theta^k}V_{t})(s,a)-Q_t(s,a)\bigr)\bigr|\notag\\
&\le \gamma\cdot\max_{s\in\cS} |V^{\star}_{\theta^k}(s_t)- V_t(s_t)| + \max_{(s,a)\in\cS\times\cA}|(B_{\theta^k}V_{t})(s,a)-Q_t(s,a)|,
    \label{eq:221}
\end{align}
where the equality and the second equality uses the definitions of $(Q_\theta^\star,V_\theta^\star)$ in \eqref{eq:bellman_opt}. Here, the first inequality uses the fact that the maximum operator is a contraction map, and the last inequality uses triangle inequality and \eqref{theta^k}. Rearranging \eqref{eq:221}, we have
\begin{align}
    \max_{s\in\cS }|V^{\star}_{\theta^k}(s_t)- V_t(s_t)|&\le \frac{1}{1-\gamma}\cdot\max_{(s,a)\in\cS\times\cA}|(B_{\theta^k}V_{t})(s,a)-Q_t(s,a)|\notag\\
    &\le \frac{\eps}{1-\gamma},\label{eq:999}
\end{align}
where the last inequality uses the definition of $\eps$-optimal planner (Definition \ref{def: epsoptplan}), the planning procedure $(\pi_t, V_t)\leftarrow \texttt{PL}^\eps(P_{\texttt{LLM+PS}(\cD_{t_k})},r_{\texttt{LLM+PS}(\cD_{t_k}))}$ in  Algorithm \ref{alg: theory} and the definition of $\theta^k$ in \eqref{theta^k}. 
Then, we upper bound the right-hand side of \eqref{eq0} as
\begin{align}
 \mathfrak{R}(T) \le \frac{\eps }{1-\gamma}\cdot T+\underbrace{\mathbb{E}\Bigl[\sum_{k=0}^{K-1}\mathbb{E}_{\pi^k}\Bigl[\sum_{t=t_k}^{t_{k+1}-1} V_t(s_t) - V_{\theta^\star}^{\pi^k}(s_t)\Bigr]\Bigr]}_{\text{term (a)}}.\label{eq1}
\end{align}

By the first part of Lemma \ref{lem:reg_dec}, we analyze term (a) in \eqref{eq1} as follows,
\begin{align}
    (1-\gamma)\cdot \text{term (a)}= & {\mathbb{E}\Bigl[\sum_{k=0}^{K-1}\mathbb{E}_{\pi^k}\Bigl[\sum_{t=t_k}^{t_{k+1}-1}    Q_t(s_t,a_t) - \bigl(B_{\theta^\star}   V_t\bigr)(s_t,a_t)\Bigr]\Bigr]}\notag\\
       &\qquad +{\mathbb{E}\Bigl[\sum_{k=0}^{K-1}\mathbb{E}_{\pi^k}\Bigl[\bigl(   V_t (s_{t_{k+1}}) - V^{\pi^k}_{\theta^\star}(s_{t_{k+1}})\bigr)-\bigl(   V_t(s_{t_k}) - V^{\pi^k}_{\theta^\star}(s_{t_k}) \bigr) \Bigr]\Bigr]}.\label{eq:222222}
\end{align}

By the last inequality in \eqref{eq:999}, we have
\begin{align}
    |Q_t(s_t,a_t) - \bigl(B_{\theta^\star}   V_t\bigr)(s_t,a_t)&|\le \eps + \bigl((B_{k}-B_{\theta^\star})   V_t\bigr)(s_t,a_t)\notag\\
    &=\eps + \bigl|\bigl((B_{k}-B_{\theta^\star})   V_t\bigr)(s_t,a_t)\bigr|\label{eq:22223_ps}
\end{align}
for any $t_k\le t<t_{k+1}$. Then, we plug \eqref{eq:22223_ps} into \eqref{eq:222222} to obtain
\begin{align}
     (1-\gamma)\cdot \text{term (a)}\le & \underbrace{\mathbb{E}\Bigl[\sum_{k=0}^{K-1}\mathbb{E}_{\pi^k}\Bigl[\sum_{t=t_k}^{t_{k+1}-1}    \bigl|\bigl((B_k-B_{\theta^\star})   V_t\bigr)(s_t,a_t)\bigr|\Bigr]\Bigr]}_{\text{term (a1)}} + \ \epsilon \cdot T\notag\\
       &\qquad +\underbrace{\mathbb{E}\Bigl[\sum_{k=0}^{K-1}\mathbb{E}_{\pi^k}\Bigl[\bigl(   V_t (s_{t_{k+1}}) - V^{\pi^k}_{\theta^\star}(s_{t_{k+1}})\bigr)-\bigl(   V_t(s_{t_k}) - V^{\pi^k}_{\theta^\star}(s_{t_k}) \bigr) \Bigr]\Bigr]}_{\text{term (a2)}}.\label{eq:ts_a}
\end{align}
 Recall that $\mathbb{P}_{t_k}$ denotes the posterior distribution of $\theta^\star$ given $\cD_{t_k}$. Under Assumption \ref{as:ts}, we know that the distribution of $\theta^k\mid \cD_{t_k}$ is also  $\mathbb{P}_{t_k}$, where $\theta^k$ is defined in \eqref{theta^k}. Then, we have
\begin{align}
   \text{term (a1)}&= \mathbb{E}\biggl[\sum_{k=0}^{K-1}\mathbb{E}_{\pi^k}\Bigl[\sum_{t=t_k}^{t_{k+1}-1} \mathbb{E}_{\theta^\star,\theta^k\sim \mathbb{P}_{t_k}}\Bigl[   \bigl|\bigl((B_{\theta^k}-B_{\theta^\star})   V_t\bigr)(s_t,a_t)\bigr|\Big|\, \cD_{t_k} \Big]\Bigr]\biggr]\notag\\
   &\le \sqrt{T}\cdot\biggl(\mathbb{E}\biggl[\sum_{k=0}^{K-1}\mathbb{E}_{\pi^k}\Bigl[\sum_{t=t_k}^{t_{k+1}-1} \mathbb{E}_{\theta^\star,\theta^k\sim \mathbb{P}_{t_k}}\Bigl[   \bigl|\bigl((B_{\theta^k}-B_{\theta^\star})   V_t\bigr)(s_t,a_t)\bigr|^2\Big|\, \cD_{t_k} \Big]\Bigr]\biggr]\biggr)^{1/2},\label{eq:var_E_1_ts}
\end{align}
where the first equality uses the tower property of the conditional expectation and the first inequality uses Cauchy-Schwarz inequality. 
Note that
$\mathbb{E}[|X-X^\prime|^2] = 2 \mathbb{V}[X]$, 
 if $X$ and $X^\prime$ are two identically independently distributed variables. Recall that  Assumption \ref{as:ts} tells that $\theta^{k}$ (defined in \eqref{theta^k}) and the data-generating parameter $\theta^\star$ are identically independently distributed given $\cD_{t_k}$, which implies
\begin{align}
    \text{term (a1)}&\le \sqrt{2T}\cdot\biggl(\mathbb{E}\biggl[\sum_{k=0}^{K-1}\mathbb{E}_{\pi^k}\Bigl[\sum_{t=t_k}^{t_{k+1}-1} \mathbb{V}_{\theta\sim \mathbb{P}_{t_k}}\bigl[   (B_{\theta^k}V_t)(s_t,a_t)\big|\, \cD_{t_k} \big]\Bigr]\biggr]\biggr)^{1/2}.\label{eq:var_E_2_ts}
\end{align}
Under Assumption \ref{as:var}, we apply the second property in Proposition \ref{prop:pred_red} on the right-hand  side of \eqref{eq:var_E_2_ts} to have
\begin{align}
    \text{term (a1)}\le 4L\cdot\sqrt{\mathbb{E}[H_0-H_T]} \cdot\sqrt{T}.\label{eq:ts_a1}
\end{align} 

Using the fact that any value function is bounded by $L$, we bound term (a2) in \eqref{eq:222224} as
\begin{align*}
    \text{term (a3)}\le 4L\cdot\ \mathbb{E}[K].
\end{align*}
As the switching condition in Algorithm \ref{alg: theory} implies $H_{t_{k}}-H_{t_{k+1}}\ge \log2$, we apply Lemma \ref{lem:swtich_times} to have 
\begin{align}
    \text{term (a2)}\le 4L+\frac{4L\cdot\mathbb{E}[H_0 - H_{T}]}{\log2} .\label{eq:ts_a2}
\end{align}
Plugging \eqref{eq:ts_a1} and \eqref{eq:ts_a2} into \eqref{eq:ts_a}, we have
\begin{align}
    (1-\gamma)\cdot \text{term (a)}\le 4L\cdot\sqrt{\mathbb{E}[H_0-H_T]} \cdot\sqrt{T}+\eps\cdot T+4L+\frac{4L\cdot\mathbb{E}[H_0 - H_{T}]}.\label{eq:a_end_ts}
\end{align}
Combining \eqref{eq1} and \eqref{eq:a_end_ts}, we obtain
\begin{align*}
    \mathfrak{R}(T)&\le \frac{4\sqrt{2}L\cdot\sqrt{\mathbb{E}[H_0-H_T]}}{1-\gamma}\cdot\sqrt{T} +\frac{2\eps}{1-\gamma}\cdot T + \frac{4L}{1-\gamma} +\frac{4L\cdot\mathbb{E}[H_0 - H_{T}]}{(1-\gamma)\log2}\\
&=\mathcal{O}\Bigl(\frac{L\cdot\sqrt{\mathbb{E}[H_0-H_T]}}{1-\gamma}\cdot\sqrt{T} +\frac{\eps}{1-\gamma}\cdot T + \frac{L\cdot\mathbb{E}[H_0 - H_{T}]}{1-\gamma}\Bigr).
\end{align*}
Thus, we finish the proof of Theorem \ref{thm:reg}.
\end{proof}
\subsection{Relaxing Assumption \ref{as:perfect} for Theorem \ref{thm:reg_conc}}\label{app:relax}
In this section, we show that Assumption \ref{as:perfect} (Posterior Alignment) can be relaxed for Theorem \ref{thm:reg_conc} to accommodate a generalization error. We remove the dependency on Assumption \ref{as:perfect} (Posterior Alignment) by introducing the following assumptions.

\begin{algorithm}[h] 
	\caption{The data collection process for the pretraining dataset.}
	\label{alg:pre}
	\begin{algorithmic}[1]
	\STATE \textbf{input}: Some (mixed) data collection policy $\pi_{\text{collect}}$.
 \STATE \textbf{initialization}: Initialize the pretraining dataset $\mathcal{D}_\text{pre} = \emptyset$.
\FOR{$n = 1, \ldots,N$}
\STATE Reset the environment such that  $\theta^\star\sim \mathbb{P}_0$.
\STATE Receives $s_0$ from the environment. 
\STATE Initialize the memory buffer $\mathcal{D}_0 = \emptyset$.
\FOR{$t= 0, \ldots, T$}
\STATE Execute action $a_t\sim \pi_{\text{collect}}(s_t)$ to receive  reward $r_t = r_{\theta^\star}(s_t,a_t)$ and state $s_{t+1}\sim P_{\theta^\star}(\cdot\mid s_t,a_t)$ from the environment. 
 \STATE  Update memory buffer $\mathcal{D}_{t+1} \leftarrow \mathcal{D}_{t} \cup \{(s_t,a_t,s_{t+1},r_t)\}$.
\ENDFOR
\STATE Uniformly sample $t_0\in\{0,\ldots, T\}$ and update the pretraining dataset $\mathcal{D}_\text{pre} = \mathcal{D}_\text{pre} \cup \{(s_{t_0 +1}, r_{t_0}, \cD_{t_0}, s_{t_0},a_{t_0})\}$
\ENDFOR
    \STATE \textbf{output}: The pretraining dataset $\cD_{\text{pre}}$.
	\end{algorithmic}
\end{algorithm}
First, we characterize the data generation process for the pretraining dataset in the following assumption.
\begin{assumption}
    [Pretraining Dataset Generation]
    We assume that the pretraining dataset $\cD_{\text{pre}}$  consists $N$ i.i.d. tuples of $(s^\prime, r , \mathcal{D},s,a)$ generated by Algorithm \ref{alg:pre}. \label{as:pre_gen}
\end{assumption}
For the pretraining dataset $\cD_{\text{pre}}$, we denote by $\mathbb{P}_{\text{pre}}$ the conditional distribution of $(s^\prime, r)$ given $(\cD, s, a)$. We show that $\mathbb{P}_{\text{pre}}$ is equivalent to $\mathbb{P}_{\text{post}}$ as follows,
\begin{align}
    \mathbb{P}_{\text{pre}}(s^\prime,r\mid \cD, s, a) 
    &= \int_\theta \mathbb{P}_\text{pre}(s^\prime,r\mid s,a,\theta)\mathbb{P}_{\text{pre}}(\mathrm{d}\theta\mid \cD, s, a)\notag\\
    &= \int_\Theta P_\theta(s^\prime\mid s,a)\textbf{1}(r=r_\theta(s,a))\mathbb{P}_{\text{pre}}(\mathrm{d}\theta\mid \cD, s, a)\notag\\
    &=\int_\Theta P_\theta(s^\prime\mid s,a)\textbf{1}(r=r_\theta(s,a))\mathbb{P}_{\text{post}}(\mathrm{d}\theta\mid\cD)\notag\\
    &=\mathbb{P}_{\text{post}}(\xi_{(s,a)}\mid \cD, s, a),
\end{align}
where the first equality uses Line 8 in Algorithm \ref{alg:pre}, the second equality uses the definition of the posterior of $\theta^\star$ in \eqref{eq:post_theta} and the fact that $\theta$ and $(s,a)$ are conditionally independent given $\cD$, and the last equality uses \eqref{eq:posterior}.
Denote by $\mathcal{F}_{\text{LLM}}$ the function class of LLMs. In the next assumption, we assume that the function class $\mathcal{F}_{\text{LLM}}$ contains the posterior of $\xi_{(s,a)}$ in the underlying MDP, which is the conditional distribution of $(s^\prime, r)$ given $(\cD, s, a)$ from $\cD_{\text{pre}}$.
\begin{assumption}
    [Realizability] We assume that there exists a LLM \texttt{LLMPA} with a posterior alignment, that is, there exists $P^{\texttt{LLMPA}}\in\mathcal{F}_{\text{LLM}}$, such that  $P^{\texttt{LLMPA}}\bigl(\xi_{(s,a)}\big|\,\cD,s,a\bigr) = \mathbb{P}_{\text{post}}\bigl(\xi_{(s,a)}\big|\, \mathcal{D},s,a\bigr)$, for any query state-action pair $(s,a)$ and in-context dataset $\cD$. \label{as:realize}
\end{assumption} 
We introduce the following assumption to require that LLMs are MLEs in the pretraining dataset with uniform coverage. 
\begin{assumption}
    We assume that LLMs used in \texttt{RAFA} are Maximum Likelihood Estimators (MLEs) in the pretraining dataset $\cD_{\text{pre}}$ satisfying Assumption \ref{as:pre_gen}, that is,
    \begin{align}
        P^{\texttt{LLM}} = \operatornamewithlimits{argmax}_{\hat P\in\mathcal{F}_{\text{LLM}}} \sum_{(s^\prime,r,\cD, s, a)\in\cD_{\text{pre}}}\log\hat  P(s^\prime, r\mid \cD, s, a).\notag
    \end{align}
   Denote by $\rho_\text{pre}$ the marginal population distribution of $(\cD, s, a)$ from $\cD_{\text{pre}}$.  We also assume that the pretraining dataset satisfies the following coverage condition: \label{as:relax}
    \begin{align}
       \zeta= \sup_{t <T} \left\{\left\|\frac{\mu_t}{\rho_{\text{pre}}}\right\|_\infty + \left\|\frac{\mu_t^\star}{\rho_{\text{pre}}}\right\|_\infty\right\}<\infty.\label{eq:zeta}
    \end{align}Here, $\mu_t$ is the marginal distribution of $(\mathcal{D}_{t_k}, s_t,\pi^k(s_t))$ and $\mu_t^\star$ is the marginal distribution of $(\mathcal{D}_{t_k}, s_t,\pi^\star(s_t))$ with $s\sim\nu^\star(\cdot\mid s_t)$ and $(s_t,\mathcal{D}_{t_k})$ following the trajectory distribution of \texttt{RAFA} (Algorithm \ref{alg: theory_gen}), where $\nu^\star$ is defined in \eqref{eq:optimal_vis}. 
\end{assumption}
We provide the generalization of Theorem \ref{thm:reg_conc} in the following corollary, which removes the dependency on Assumption \ref{as:perfect} (Posterior Alignment). 
\begin{corollary}[Generalization of Theorem \ref{thm:reg_conc}]\label{cor:relax}
     Under Assumptions \ref{as:realize}, \ref{as:relax}, \ref{as:var} , and \ref{as:conc_coef},the Bayesian regret of \texttt{RAFA} (Algorithm \ref{alg: theory_gen}) satisfies 
    \begin{align*}
        \mathfrak{R}(T)&= \mathcal{O}\Biggl(\frac{(\kappa+1)L\cdot\sqrt{\mathbb{E}[H_0-H_T]}}{1-\gamma}\cdot\sqrt{T} +\frac{\eps}{1-\gamma}\cdot T \\
        &\qquad+ \frac{L\cdot\mathbb{E}[H_0 - H_{T}]}{1-\gamma} +\underbrace{\zeta L\cdot\sqrt{\frac{\log(|\mathcal{F}_{\mathrm{LLM}}|/\delta)}{N}}\cdot  T}_{\text{Additional Regret Compared with Theorem \ref{thm:reg_conc}}}\Biggr),
    \end{align*}
    with probability at least $1-\delta$.
    Here, $\zeta$ is defined in \eqref{eq:zeta}, $|\mathcal{F}_{\text{LLM}}|$ is the cardinality of the function class for LLMs, and $N$ is the size of the pretraining dataset. 
\end{corollary}
Comparing Corollary \ref{cor:relax} with Theorem \ref{thm:reg_conc}, we remark that the additional regret decays to zero if $N$ tends to infinity. Hence, we can recover the regret bound based on Assumption \ref{as:perfect} (posterior alignment) approximately if the pretraining dataset has uniform coverage and is large enough.
\begin{proof}[Proof of Corollary \ref{cor:relax}]
  We start with a standard concentration result for the maximum-likelihood estimator (MLE). 
  \begin{lemma}
      [MLE Concentration] Let $\mathcal{F}$ be a finite function class used to model a conditional distribution $\mathbb{P}_{Y\mid X}(y |\, x)$ for $x \in \mathcal{X}$ and $y \in \mathcal{Y}$. Assume there is $f^{\star} \in \mathcal{F}$ such that $\mathbb{P}(y |\, x)=f^{\star}(y |\, x)$ (realizablility condition). Let $\left\{(x_i, y_i)\right\}_{i =1}^N$ denote a dataset of i.i.d. samples where $x_i \sim \mathbb{P}_X(x)$ and $y_i \sim \mathbb{P}_{Y |\, X}\left(\cdot |\, x_i\right)$. Let $\hat f$ be the MLE, which satisfies
\begin{align}
    \hat{f}=\underset{f \in \mathcal{F}}{\operatorname{argmax}} \sum_{i =1}^N \log f\left(y_i |\, x_i\right).\notag
\end{align}
Then, it holds that
\begin{align}
    \mathbb{E}_{x \sim \mathbb{P}_X}\left[d_{\mathrm{TV}}\left(\hat{f}(\cdot |\, x), p_{Y |\, X}(\cdot |\, x)\right)\right] \leq \frac{8 \log (|\mathcal{F}| / \delta)}{N},\notag
\end{align}
 with probability at least $1-\delta$.\label{lem:mle}
  \end{lemma}
  \begin{proof}
      [Proof of Lemma \ref{lem:mle}] See the proof of Theorem 21 of \citet{agarwal2020flambe}.
  \end{proof}
   Under Assumptions \ref{as:realize} and \ref{as:relax}, we apply Lemma \ref{lem:mle} to show that
   \begin{align*}
      \mathbb{E}_{(\mathcal{D},s,a)\sim \rho_{\text{pre}}}\bigl[d_{\mathrm{TV}}(P^{{\texttt{LLMPA}}}(\cdot|\, \mathcal{D},s,a)\|P^{\texttt{LLM}}(\cdot|\, \mathcal{D},s,a) )\bigr]\le \sqrt{\frac{8 \log (|\mathcal{F}_{\mathrm{LLM}}| / \delta)}{N}}
   \end{align*}
   holds with probability at least $1-\delta$.
  
For any fixed distribution $\mu$ of $(\mathcal{D},s,a)$ satisfying $\|\mu/\rho_{\text{pre}}\|_\infty <\infty$, we use Hölder's inequality to know that
\begin{align}
&\mathbb{E}_{(\mathcal{D},s,a)\sim \mu}\bigl[d_{\mathrm{TV}}(P^{{\texttt{LLMPA}}}(\cdot|\, \mathcal{D},s,a)\|P^{\texttt{LLM}}(\cdot|\, \mathcal{D},s,a) )\bigr]\\
&\qquad\le \left\|\frac{\mu}{\rho_{\text{pre}}}\right\|_{\infty}\cdot \mathbb{E}_{(\mathcal{D},s,a)\sim \rho_{\text{pre}}}\bigl[d_{\mathrm{TV}}(P^{{\texttt{LLMPA}}}(\cdot|\, \mathcal{D},s,a)\|P^{\texttt{LLM}}(\cdot|\, \mathcal{D},s,a) )\bigr]\notag\\
&\qquad\le \left\|\frac{\mu}{\rho_{\text{pre}}}\right\|_{\infty} \cdot \sqrt{\frac{8 \log (|\mathcal{F}_{\mathrm{LLM}}| / \delta)}{N}}\label{eq:mle_ge}
\end{align}
holds with probability at least $1-\delta$. 
Here, $\|\cdot\|_{\infty}$ denotes the infinity norm.  
We denote the Bellman operator induced by \texttt{LLMPA} (the LLM with a posterior alignment) and $\mathcal{D}_{t_k}$ as $\tilde{B}_k$, which is defined as $(\tilde{B}_k V)(s,a) = r_{\texttt{LLMPA}(\mathcal{D}_{t_k})}(s,a)+({P}_{\texttt{LLMPA}(\mathcal{D}_{t_k})} V)(s,a)$ for any $s,a$, and value function $V$. Then, by the definition of $B_k$, we have
\begin{align}
\left|\bigl((\tilde{B}_k-B_{k})V_t\bigr)(s,a)\right|&= \left|\mathbb{E}_{ {P}^{\texttt{LLMPA}}}[r+\gamma\cdot V(s^\prime)] - \mathbb{E}_{ P^{\texttt{LLM}}}[r+\gamma\cdot V(s^\prime)]\right|\notag\\
&\le 2L\cdot d_{\mathrm{TV}}({P}^{\texttt{LLMPA}}(\cdot|\, \mathcal{D}_{t_k}, s, a) \|P^{\texttt{LLM}}(\cdot|\, \mathcal{D}_{t_k}, s, a) ),
\label{eq:appendix_gen_1}
\end{align}
where the first inequality uses the definition of $L$ (recall that $L$ is the bound of $|\, r+V(s)|\,$ for any reward $r$, state $s$, and value $V$) and Hölder's inequality.  
In the proof of Theorem \ref{thm:reg_conc} (the analysis of the regret of \texttt{RAFA}), we need to modify \eqref{eq:222224} and \eqref{eq:conc_b} with the following inequality
\begin{align}
\big|\bigl((B_k-B_{\theta^\star})V_t\bigr)(s,a)\big|&= \big|\bigl((\tilde{B}_k-B_{\theta^\star})V_t\bigr)(s,a)\notag\\
&\qquad+\bigl((\tilde{B}_k-B_{k})V_t\bigr)(s,a)\big|\notag\\
&\le  \big|\bigl((\tilde{B}_k-B_{\theta^\star})V_t\bigr)(s,a)\big| \notag\\
&\qquad + 2L\cdot d_{\mathrm{TV}}({P}^{\texttt{LLMPA}}(\cdot|\, \mathcal{D}_{t_k}, s, a)\|P^{\texttt{LLM}}(\cdot|\, \mathcal{D}_{t_k}, s, a) )\notag,
\end{align}
which holds for any state $s\in\cS$ and action $a\in\cA$. 
By Proposition \ref{prop:llm_bma} (LLMs with posterior alignments perform BMA) and the fact that $\tilde{B}_k$ is the Bellman operator induced by  \texttt{LLMPA} (the LLM with a posterior alignment) and $\mathcal{D}_{t_k}$, we can analyze $[(\tilde{B}_k-B_{\theta^\star})V_t)(s,a)]$ in the same way as in the previous proof of Theorem \ref{thm:reg_conc} (the analysis of the regret of \texttt{RAFA}). 
It is clear that the additional regret  by relaxing  Assumption \ref{as:perfect} (Posterior Alignment) is less than 
\begin{align}
&\frac{1}{1-\gamma}\cdot\mathbb{E}\Bigl[\sum_{k=0}^{K-1}\mathbb{E}_{\pi^k}\Bigl[\sum_{t=t_k}^{t_{k+1}-1} 2L\cdot d_{\mathrm{TV}}({P}^{\texttt{LLMPA}}(\cdot|\, \mathcal{D}_{t_k}, s, a)\|P^{\texttt{LLM}}(\cdot|\, \mathcal{D}_{t_k}, s, a) )\Bigr]\Bigr]\notag\\
&\qquad+ \frac{1}{1-\gamma}\cdot\mathbb{E}\biggl[\sum_{k=0}^{K-1}\mathbb{E}_{\pi^k}\biggl[\sum_{t=t_k}^{t_{k+1}-1} \mathbb{E}_{s\sim \nu^\star(\cdot\mid s_t)}\Bigl[2L\cdot d_{\mathrm{TV}}({P}^{\texttt{LLMPA}}(\cdot|\, \mathcal{D}_{t_k}, s, a)\|P^{\texttt{LLM}}(\cdot|\, \mathcal{D}_{t_k}, s, a) )\Bigr]\biggr]\biggr]\notag\\
&\qquad\le
8L\cdot\sqrt{\frac{2\log(|\mathcal{F}_{\mathrm{LLM}}|/\delta)}{N}}\cdot \zeta\cdot T,\label{eq:additional_reg}
\end{align}
with probability at least $1-\delta$, where the inequality uses \eqref{eq:mle_ge} and the definition of $\zeta$ in \eqref{eq:zeta}. Combining \eqref{eq:additional_reg} and Theorem \ref{thm:reg_conc}, we conclude the proof of Corollary \ref{cor:relax}. 
\end{proof}
\section{Missing Proofs in Appendix \ref{app:proof_main}}
\subsection{Proof of Lemma \ref{lem:reg_dec}}
\begin{proof}[Proof of Lemma \ref{lem:reg_dec}]
\label{mpf:reg_dec}
We prove the first part as follows. The Bellman equation \citep{sutton2018reinforcement} connects $Q_\theta^\pi(s,a)$ and $V^\pi_\theta(s)$ by
\begin{align}
    Q^\pi_\theta\left(s, a\right)=r_\theta\left(s, a\right)+\gamma\left(P_\theta V_{\theta}^\pi\right)\left(s, a\right),\quad V_\theta^\pi = Q_\theta^\pi\bigl(s,a\bigr).\label{eq:bellman}
\end{align}
By the definition of $B_\theta$, we rewrite \eqref{eq:bellman} as
$
Q^\pi_\theta\left(s, a\right)=(B_{\theta}V_\theta^\pi)(s,a)
$.
For the left-hand side of \eqref{eq:www} in the first part of Lemma \ref{lem:reg_dec}, we have
\begin{align}
    &\mathbb{E}\Bigl[\sum_{k=0}^{K-1}\mathbb{E}_{\pi^k}\Bigl[\sum_{t=t_k}^{t_{k+1}-1}  V_t(s_t)- V_{\theta^\star}^{\pi^k}(s_t)\Bigr]\Bigr]\notag\\
    &\quad= \mathbb{E}\Bigl[\sum_{k=0}^{K-1}\mathbb{E}_{\pi^k}\Bigl[\sum_{t=t_k}^{t_{k+1}-1}  Q_t(s_t,a_t) - ({B}_{\theta^\star}V_{\theta^\star}^{\pi^k})(s_t,a_t)\Bigr]\Bigr]\notag\\
    &\quad= \underbrace{\mathbb{E}\Bigl[\sum_{k=0}^{K-1}\mathbb{E}_{\pi^k}\Bigl[\sum_{t=t_k}^{t_{k+1}-1}  Q_t(s_t,a_t) - ({B}_{\theta^\star} V_t)(s_t,a_t)\Bigr]\Bigr]}_{\text{term (A)}}\notag\\
    &\quad\qquad+ \underbrace{\mathbb{E}\Bigl[\sum_{k=0}^{K-1}\mathbb{E}_{\pi^k}\Bigl[\sum_{t=t_k}^{t_{k+1}-1} ({B}_{\theta^\star} V_t)(s_t,a_t)-r_{\theta^\star}(s_t,a_t)-\gamma\cdot V_t(s_{t+1}) \Bigr]\Bigr]}_{\text{term (C1)}}\notag\\
     &\quad\qquad+  \underbrace{\mathbb{E}\Bigl[\sum_{k=0}^{K-1}\mathbb{E}_{\pi^k}\Bigl[\sum_{t=t_k}^{t_{k+1}-1} r_{\theta^\star}(s_t,a_t) + \gamma \cdot V_{\theta^\star}^{\pi^k}(s_{t+1})-({B}_{\theta^\star}V^{\pi^k}_{\theta^\star})(s_t,a_t) \Bigr]\Bigr]}_{\text{term (C2)}}\notag\\
        &\quad\qquad + \gamma \cdot \underbrace{\mathbb{E}\Bigl[\sum_{k=0}^{K-1}\mathbb{E}_{\pi^k}\Bigl[\sum_{t=t_k}^{t_{k+1}-1}  V_t(s_{t+1}) - V_{\theta^\star}^{\pi^k}(s_{t+1})\Bigr]\Bigr]}_{\text{term (D)}},\label{eq:decom}
\end{align}
where the first equality uses $a_t = \pi^k(s_t)$, the condition $Q(s,\pi^k(s))=V_t(s)$ for any $t_{k}\le t<t_{k+1}$ and $k<K$ in Lemma \ref{lem:reg_dec}, and \eqref{eq:bellman} . Since we have 
\begin{align*}
(B_{\theta^\star}V)(s_t,a_t) =r_{\theta^\star}(s,a)+\gamma\cdot\mathbb{E}_{s_{t+1} \sim P_{\theta^{\star}}(\cdot\mid s_t,a_t)}\bigl[ V(s_{t+1})\bigr],
\end{align*} 
terms (C1) and (C2) in \eqref{eq:decom} are zero. 
Meanwhile, term (D) in \eqref{eq:decom} satisfies
\begin{align}
    \text{term (D)}& = {\mathbb{E}\Bigl[\sum_{k=0}^{K-1}\mathbb{E}_{\pi^k}\Bigl[\sum_{t=t_k}^{t_{k+1}-1}  \bigl( V_t(s_{t}) - V_{\theta^\star}^{\pi^k}(s_{t})\bigr)\Bigr]\Bigr]}\notag\\
&\qquad+\underbrace{\mathbb{E}\Bigl[\sum_{k=0}^{K-1}\mathbb{E}_{\pi^k}\Bigl[\bigl( V_t(s_{t_{k+1}}) - V^{\pi^k}_{\theta^\star}(s_{t_{k+1}})\bigr) -\bigl( V_t(s_{t_k}) - V_{\theta^\star}^{\pi^k}(s_{t_k})\bigr) \Bigr]\Bigr]}_{\text{term (B)}}\label{eq:c},
\end{align}
where term (B) is defined in the first part of Lemma \ref{lem:reg_dec}. 
Rearranging \eqref{eq:decom} and \eqref{eq:c}, we prove the first part of Lemma \ref{lem:reg_dec}.

Next, we show the proof of the second part of Lemma \ref{lem:reg_dec}, we. For the left-hand side of \eqref{eq:www99}, we have
\begin{align}
    &\mathbb{E}\Bigl[\sum_{k=0}^{K-1}\mathbb{E}_{\pi^k}\Bigl[\sum_{t=t_k}^{t_{k+1}-1}  V^{\pi^\star}_{\theta^\star}(s_t)-  V_t(s_t)\Bigr]\Bigr]\notag\\
    &\quad= \mathbb{E}\Bigl[\sum_{k=0}^{K-1}\mathbb{E}_{\pi^k}\Bigl[\sum_{t=t_k}^{t_{k+1}-1} ({B}_{\theta^\star}V_{\theta^\star}^{\pi^\star})(s_t,\pi^\star(s_t)) -  Q_t(s_t,\pi^k(s_t))  \Bigr]\Bigr]\notag\\
     &\quad = {\mathbb{E}\Bigl[\sum_{k=0}^{K-1}\mathbb{E}_{\pi^k}\Bigl[\sum_{t=t_k}^{t_{k+1}-1} ({B}_{\theta^\star}V_{\theta^\star}^{\pi^\star})(s_t,\pi^\star(s_t)) -  ({B}_{\theta^\star}V_t)(s_t,\pi^\star(s_t)) \Bigr]\Bigr]} \notag\\
    &\quad\qquad+ {\mathbb{E}\Bigl[\sum_{k=0}^{K-1}\mathbb{E}_{\pi^k}\Bigl[\sum_{t=t_k}^{t_{k+1}-1}  ({B}_{\theta^\star}V_t)(s_t,\pi^\star(s_t)) -  Q_t(s_t,\pi^k(s_t))  \Bigr]\Bigr]}\\
    &\quad = \mathbb{E}\Bigl[\sum_{k=0}^{K-1}\mathbb{E}_{\pi^k}\Bigl[\sum_{t=t_k}^{t_{k+1}-1} \gamma\cdot\mathbb{E}_{s^\prime \sim P_{\theta^\star}(\cdot\mid s_t,\pi^\star(s_t))}\bigl[V_{\theta^\star}^{\pi^\star}(s^\prime) -  V_t(s^\prime)\bigr] \Bigr]\Bigr]\notag\\
    &\quad\qquad+ \mathbb{E}\Bigl[\sum_{k=0}^{K-1}\mathbb{E}_{\pi^k}\Bigl[\sum_{t=t_k}^{t_{k+1}-1}  ({B}_{\theta^\star}V_t)(s_t,\pi^\star(s_t)) -  Q_t(s_t,\pi^k(s_t))  \Bigr]\Bigr],\label{eq:decomp_mid1}
\end{align}
where the first equality uses the Bellman optimality equation in \eqref{eq:bellman_opt}, the condition $Q_t(s,\pi^k(s))=V_t(s)$, and $Q^\star_{\theta^\star}(s,\pi^\star(s)) = V_{\theta^\star}^\star(s)$ for any state $s$,  $t_{k}\le t<t_{k+1}$, and $k<K$. Here, the last equality uses the definition of $B_{\theta^\star}$. For the simplicity of discussions, we define functions $F_t,M_t \in \{\mathcal{S}\mapsto \mathbb{R}\}$ and the linear operator $\mathcal{T}\in\bigl\{\{\mathcal{S}\mapsto \mathbb{R}\}\mapsto \{\mathcal{S}\mapsto \mathbb{R}\}\bigr\}$  as
\begin{align}
    F_t(s) &= V_{\theta^\star}^{\pi^\star}(s) - V_t(s),\notag\\
    M_t(s) &= ({B}_{\theta^\star}V_t)(s,\pi^\star(s)) -  Q_t(s_t,\pi^k(s)),\notag\\
    (\mathcal{T}f)(s) &= \mathbb{E}_{s^\prime \sim P_{\theta^\star}(\cdot\mid s,\pi^\star(s))}[f(s^\prime)]\label{eq:def:T},
\end{align}
for any state $s$ and function $f\in\{\mathcal{S}\mapsto\mathbb{R}\}$. Here, we denote by $\{\mathcal{S}\mapsto\mathbb{R}\}$ the class of all the functions defined on $\mathcal{S}$. By the definitions of $F_t$, $M_t$, and $\mathcal{T}$ in \eqref{eq:def:T}, it is clear that
\begin{align}
    F_t(s) = M_t(s) + \gamma\cdot \bigl(\mathcal{T}F_t\bigr)(s),\label{eq:condition_gemma}
\end{align}
for any state $s\in\mathcal{S}$.
Then, we introduce the following lemma to bound $F_t$ by   \eqref{eq:condition_gemma}.
\begin{lemma}
    For the operator $\mathcal{T}$ defined in \eqref{eq:def:T}, two arbitrary bounded functions $f,m$ defined on the state space $\mathcal{S}$, and any $\gamma\in[0,1)$, if \label{lem:1-gamma}
    \begin{align}
        f(s) = m(s) + \gamma \cdot (\mathcal{T}f)(s)\label{eq:lem:gemma:condition}
    \end{align}
    holds for any state $s\in\mathcal{S}$, then it holds that for any state $s\in\mathcal{S}$,
    \begin{align}
        f(s) = \sum_{\tau = 0}^\infty \gamma^\tau \cdot\bigl(\underbrace{(\mathcal{T}\circ\ldots\circ \mathcal{T})}_{\text{$\tau$ times}} m\bigr)(s).
    \end{align}
\end{lemma}
\begin{proof}
    [Proof of Lemma \ref{lem:1-gamma}]
    By the condition in \eqref{eq:lem:gemma:condition}, we have
    \begin{align}
        f(s)& =  m(s) + \gamma \cdot \Bigl(\mathcal{T}\bigl(m + \gamma \cdot (\mathcal{T}f)\bigr)\Bigr)(s)\notag\\
        &= m(s) + \gamma\cdot (\mathcal{T}m)(s) + \gamma^2 \cdot \bigl((\mathcal{T}\circ\mathcal{T})f\bigr)(s),\label{eq:lem:iter}
    \end{align}
    where the last equality relies on the linearity of the operator $\mathcal{T}$. Repeating the process in \eqref{eq:lem:iter} for $N\in\mathbb{N}$ times, we have that 
\begin{align}
    f(s)& =  \gamma^{N+1}\cdot\bigl(\underbrace{(\mathcal{T}\circ\ldots\circ \mathcal{T})}_{\text{$(N+1)$ times}} f\bigr)(s)  + \sum_{\tau = 0}^N \gamma^\tau\cdot \bigl(\underbrace{(\mathcal{T}\circ\ldots\circ \mathcal{T})}_{\text{$\tau$ times}} m\bigr)(s). \label{eq:lem:sum}
\end{align}
Since both $f$ and $m$  are bounded functions, we use \eqref{eq:def:T} to know that $\underbrace{(\mathcal{T}\circ\ldots\circ \mathcal{T})}_{\text{$\tau$ times}} f$ and $\underbrace{(\mathcal{T}\circ\ldots\circ \mathcal{T})}_{\text{$\tau$ times}} m$ are also bounded for any $\tau\in\mathbb{N}$. As $\gamma\in[0,1)$, we let $N$ tend to the infinity to transform \eqref{eq:lem:sum} to 
\begin{align*}
    f(s) = \sum_{\tau = 0}^\infty \gamma^\tau\cdot \bigl(\underbrace{(\mathcal{T}\circ\ldots\circ \mathcal{T})}_{\text{$\tau$ times}} m\bigr)(s),
\end{align*}
for any state $s\in\mathcal{S}$. Then, we 
conclude the proof for Lemma \ref{lem:1-gamma}.
\end{proof}
By \eqref{eq:condition_gemma} and Lemma \ref{lem:1-gamma},  we have
\begin{align*}
    F_t(s)  = \sum_{\tau}^\infty\gamma^\tau \cdot \bigl(\underbrace{(\mathcal{T}\circ\ldots\circ \mathcal{T})}_{\text{$\tau$ times}} M_t\bigr)(s)
\end{align*}
Recalling the definition of the optimal $\gamma$-discounted visitation measure in \eqref{eq:optimal_vis}, we further have 
\begin{align}
    F_t(s)  = \frac{1}{1-\gamma}\cdot \mathbb{E}_{s^\prime\sim \nu^\star(\cdot\mid s)}[M_t(s^\prime)].\label{eq:result_F}
\end{align}
Plugging \eqref{eq:result_F} and \eqref{eq:def:T} into \eqref{eq:decomp_mid1}, we have
\begin{align}
     &\mathbb{E}\Bigl[\sum_{k=0}^{K-1}\mathbb{E}_{\pi^k}\Bigl[\sum_{t=t_k}^{t_{k+1}-1}  V^{\pi^\star}_{\theta^\star}(s_t)-  V_t(s_t)\Bigr]\Bigr]\notag\\
    &\quad= \mathbb{E}\Bigl[\sum_{k=0}^{K-1}\mathbb{E}_{\pi^k}\Bigl[\sum_{t=t_k}^{t_{k+1}-1} \frac{1}{1-\gamma}\cdot\mathbb{E}_{s^\prime\sim \nu^\star(\cdot\mid s)} \bigl[ ({B}_{\theta^\star}V_t)(s_t,\pi^\star(s_t)) -  Q_t(s_t,\pi^k(s_t))\bigr]  \Bigr]\Bigr]\notag\\
     &\quad=  \frac{1}{1-\gamma}\cdot\mathbb{E}\Bigl[\sum_{k=0}^{K-1}\mathbb{E}_{\pi^k}\Bigl[\sum_{t=t_k}^{t_{k+1}-1} \mathbb{E}_{s\sim \nu^\star(\cdot\mid s)} \bigl[ ({B}_{\theta^\star}V_t)(s,\pi^\star(s)) -  Q_t(s,\pi^\star(s))\bigr]  \Bigr]\Bigr]\notag\\
     & \quad\qquad+ \frac{1}{1-\gamma}\cdot \mathbb{E}\Bigl[\sum_{k=0}^{K-1}\mathbb{E}_{\pi^k}\Bigl[\sum_{t=t_k}^{t_{k+1}-1}\mathbb{E}_{s\sim \nu^\star(\cdot\mid s)} \bigl[ (Q_t(s,\pi^\star(s)) -  Q_t(s,\pi^k(s))\bigr]  \Bigr]\Bigr]. \label{eq:end_decomp_2}
\end{align}
Multiplying $(1-\gamma)$ on the two sides of \eqref{eq:end_decomp_2}, we prove the second part of Lemma \ref{lem:reg_dec}.
\end{proof}

\section{Linear Special Case}\label{app:linear}
We specialize \texttt{RAFA}  to a linear setting and characterize the Bayesian regret. In particular, we define a Bayesian variant of linear kernel MDPs \citep{yang2020reinforcement,yang2019sample,cai2020provably,zhou2021provably}. Here, $\mathbb{E}_{s^\prime\sim P_{\theta}(\cdot\mid s,a)}[V(s^\prime)]$ is linear in a feature $\psi_V(s,a)\in\mathbb{R}^d$  for an arbitrary parameter $\theta\in\mathbb{R}^d$, while the prior and posterior distributions of the data-generating parameter $\theta^\star\in\mathbb{R}^d$ are Gaussian.  Specifically, $\psi_V(s,a)$ maps the value function $V$ and the state-action pair $(s,a)$ to a $d$-dimensional vector. 
Recall that $\rho$ is the initial distribution of states, $t$ is the step index, and $T$ is the total number of steps. Also, $\mathbb{P}_{t}$ is the posterior distribution at the $t$-th step. 
\begin{definition}[Bayesian Linear Kernel MDP \citep{ghavamzadeh2015bayesian,yang2020reinforcement,yang2019sample,cai2020provably,zhou2021provably}]\normalfont
   A Bayesian linear kernel MDP $M$ satisfies
    \begin{align*}
    V(s_{t+1})\mid s_t,a_t\sim \mathcal{N}(\psi_{V}(s_t,a_t)^\top\theta,1)
\end{align*}
for all $t\ge 0$,  $(s_t,a_t)\in\mathcal{S}\times\mathcal{A}$, $s_{t+1}\sim P_{\theta}(\cdot\mid s_t,a_t)$,  $\theta\in\mathbb{R}^d$, as well as all value function $V$.  Here, $\psi_V(s,a)$ maps $V$ and $(s,a)$  to a $d$-dimensional vector, which satisfies $\|\psi_V(s,a)\|_2 \le R$ for all  $(s,a)\in\cS\times\cA$ and all $V$. Also,  $M$ also satisfies $|\EE_{s_0\sim\rho}V(s_0)|\le R$ for all $V$.   Here, $R$ is a positive constant that is independent of $t$ and $T$. The prior distribution of the data-generating parameter $\theta^\star\in\mathbb{R}^d$ is $ \mathcal{N}(0,\lambda I_d)$, where $\lambda$ is a positive constant. Here, $\psi_V$ is known and $\theta^\star$ is unknown. Without loss of generality, we assume that the reward function is deterministic and known, i.e.,  
$(B_{\theta}V)(\cdot,\cdot) = r(\cdot,\cdot)+\gamma\cdot(P_\theta V)(\cdot,\cdot)$ for a known reward function $r$ and any $\theta$.
\label{def:lk mdp} 
\end{definition}

By Definition \ref{def:lk mdp}, we  
obtain the closed form of the posterior $\mathbb{P}_{t}$ as  follows,
\begin{align*}
    \theta\mid \mathcal{D}_t \sim \mathcal{N}(\hat \theta_t;\Sigma_t^{-1}),
\end{align*}
where 
\begin{align}
    \hat \theta_t = \Bigl(\lambda I_d+\sum_{i=0}^{t-1} \psi_{V_i}(s_i, a_i) \psi_{V_i}(s_i, a_i)^{\top}\Bigr)^{-1}\Bigl(\sum_{i=0}^{t-1}\psi_{V_i}(s_i, a_i) V_i(s_{i+1})\Bigr)\label{eq:ls estimator}
\end{align}
and 
\begin{align}
    \Sigma_t=\lambda I_d+\sum_{i=0}^{t-1} \psi_{V_i}\left(s_i, a_i\right) \psi_{V_i}\left(s_i, a_i\right)^{\top}.\label{eq:sigma}
\end{align}
Hence, the posterior entropy is
\begin{align}
    H_t =H(\mathbb{P}_{t})= 1/2 \cdot\log (\det(\Sigma_t))+d/2\cdot(1+\log (2 \pi)).\label{eq:linear_closed_ent}
\end{align}
We specialize the switching condition in Algorithm \ref{alg: theory} as follows,
\begin{align}
     H_{t_k} - H_t = 1/2\cdot\log(\det(\Sigma_{t_{k}})) - 1/2\cdot\log(\det(\Sigma_{t})) > \log 2,\label{eq:l_ent}
\end{align}
which is equivalent to 
$
    \det(\Sigma_{t_{k}}) > 4\cdot\det(\Sigma_{t})
$. This switching condition is also similarly adopted in work for RL \citep{zhou2021provably,abbasi2015bayesian}. As a result, we have
\begin{align}
    \det(\Sigma_{t_{k}}) \le 4\cdot\det(\Sigma_{t})\label{eq:linear_switch_2}
\end{align}
 for all $t_k\le t<t_{k+1}$ and $k<K$. 
\paragraph{Verification of  Assumption \ref{as:var}}
We verify the regularity assumption (Assumption \ref{as:var}.) on Bayesian linear kernel MDPs as follows.  By \eqref{eq:l_ent}, the condition 
\begin{align*}
    H_{t_1}-H_{t_2} = 1/2\cdot\log(\det(\Sigma_{t_{1}})) - 1/2\cdot\log(\det(\Sigma_{t_2})) \le \log 2
\end{align*}
is equivalent to $\det(\Sigma_{t_1})\le 4 \det(\Sigma_{t_2})$. Since $t_1<t_2$ and the posterior variance matrix is positive definite, we have $\Sigma_{t_1}^{-1}\succeq\Sigma_{t_1}^{-1}$ and $\det(\Sigma_{t_2}^{-1})\le 4 \det(\Sigma_{t_1}^{-1})$. By the definition of the information gain and \eqref{eq:l_ent}, we have
\begin{align}
    I(\theta;\xi_{(s,a)}\mid \cD_t)&=H(\theta\mid \cD_t)-H(\theta\mid\xi_{(s,a)}, \cD_t)\notag\\
    &= 1/2\cdot \log\Biggl(\frac{\det(\psi_{V_{t_2}}(s,a)\psi_{V_{t_2}}^\top(s,a)+\Sigma_t)}{\det(\Sigma_t)}\Biggr)\notag\\
    &=1/2\cdot\log\bigl(1+\psi_{V_{t_2}}(s,a)^\top \Sigma_t^{-1}\psi_{V_{t_2}}(s,a)\bigr),\label{eq:linear_I}
\end{align}
for $t=t_1$ and $t_2$. Here, the last equality uses the matrix determinant lemma. 

Plugging \eqref{eq:psi} into \eqref{eq:linear_I}, we have
     \begin{align}
    I(\theta;\xi_{(s,a)}\mid \cD_{t_2}) &= 1/2\cdot\log(1+\psi_{V_{t_2}}(s,a)^\top \Sigma_t^{-1}\psi_{V_{t_2}}(s,a)) \notag\\
        &=1/2\cdot\log(1+\psi_{V_{t_2}}(s,a)^\top \Sigma_t^{-1}\psi_{V_{t_2}}(s,a))\notag\\
        &\ge\log(1+d)/(2d)\cdot\|\psi_{V_{t_2}}(s,a)\|_{\Sigma_{t_2}^{-1}}^2,\label{eq:linear_information_ratio2}
    \end{align}
    where the second equality uses the matrix determinant lemma and the first inequality uses the fact that $\log(1+x)/ x$ is an increasing function for $x\ge 0$ and \begin{align}
        0&\le\psi_{V_{t_2}}(s,a)^\top \Sigma_t^{-1}\psi_{V_{t_2}}(s,a)\notag\\&\le\psi_{V_{t_2}}(s,a)^\top\big( {\psi_{V_{t_2}}(s,a)\psi_{V_{t_2}}(s,a)^\top}\big)^{-1}\psi_{V_{t_2}}(s,a)\notag\\
        &=\tr\Bigl(\psi_{V_{t_2}}(s,a)\psi_{V_{t_2}}(s,a)^\top\bigl( {\psi_{V_{t_2}}(s,a)\psi_{V_{t_2}}(s,a)^\top}\bigr)^{-1}\Bigr)\notag\\
        &=d.\label{eq:I_and_sigma}
    \end{align}
    Here, the first inequality uses the nonnegativity of a quadratic form, the first equality uses $\tr(a^\top b)=\tr(b a^\top)$ for two arbitrary vectors $a$ and $b$, and the second inequality uses  \eqref{eq:sigma}. 
    By \eqref{eq:sigma},  we know that 
\begin{align} \|\psi_{V_{t_2}}(s,a)\|^2_{\Sigma_{t_1}^{-1}}&\le4\cdot\|\psi_{V_{t_2}}(s,a)\|^2_{\Sigma_{t_2}^{-1}},\label{eq:psi}
\end{align}
where the inequality invokes the following lemma (Lemma \ref{lem:matrix}). 
\begin{lemma}[Lemma 12 in \citet{abbasi2011improved}]
 Suppose $A, D \in \mathbb{R}^{d \times d}$ are two positive definite matrices satisfying that $A \succeq D$, then for any $\mathbf{x} \in \mathbb{R}^d,\|\mathbf{x}\|_{A} \leq\|\mathbf{x}\|_{D} \cdot \sqrt{\operatorname{det}(A) / \operatorname{det}(D)}$.\label{lem:matrix}
\end{lemma}
 Rearranging \eqref{eq:linear_information_ratio2} , we have
\begin{align}
    \frac{8d}{\log(1+d)}\cdot I(\theta;\xi_{(s,a)}\mid \cD_{t_2}) &\ge 4\cdot\|\psi_{V_{t_2}}(s,a)\|_{\Sigma_{t_2}^{-1}}^2\notag\\&\ge \|\psi_{V_{t_2}}(s,a)\|_{\Sigma_{t_1}^{-1}}^2\notag\\&\ge \log(1+\psi_{V_{t_2}}(s,a)^\top \Sigma_{t_1}^{-1}\psi_{V_{t_2}}(s,a))\notag\\
    &= 2\cdot I(\theta;\xi_{(s,a)}\mid \cD_{t_1}),\label{eq:linear_coeff}
\end{align}
where the second inequality uses \eqref{eq:psi}, the last inequality uses the fact that $x\ge \log(1+x)$ for any $x\ge 0$, and the last equality use \eqref{eq:linear_I}. By \eqref{eq:linear_coeff}, we know that Bayesian linear kernel MDPs (Definition \ref{def:lk mdp}) satisfy Assumption \ref{as:var} with the coefficient $\eta = d/\log(1+d)$.
\label{app:as:lkmdp}

\paragraph{Analysis of the Cumulative Posterior Entropy $H_0-H_T$.} Next, we study the upper bound of the cumulative Information gain $H_0-H_T$ in Bayesian linear kernel MDPs. 
By the definition of $\Sigma_{t}$ in \eqref{eq:sigma}, we have $\log \det(\Sigma_{0}) = d\cdot\log\lambda$ and 
\begin{align}
   \log \det(\Sigma_{T})&= \log\det\Bigl(\lambda I_d + \sum_{t=0}^{T-1} \psi_{V_t}(s_t,a_t) \psi^\top_{V_t}(s_t,a_t)\Bigr)\notag\\
   &\le d\cdot\log\Bigl(1/d\cdot\tr\bigl(\lambda I_d + \sum_{t=0}^{T-1} \psi_{V_t}(s_t,a_t) \psi^\top_{V_t}(s_t,a_t)\bigr)\Bigr)\notag\\
   &= d\cdot\log\Bigl(1/d\cdot\bigl(\lambda d + \sum_{t=0}^{T-1}\|\psi_{V_t}(s_t,a_t)\|_2^2\bigr)\Bigr)\notag\\
   &\le d\cdot\log(\lambda + TR^2/d)
\end{align}
almost surely. Here, the first inequality uses the relationship between the trace and the determinant of a square matrix, the second equality uses  $\tr (a^\top b) = \tr(ba^\top)$ for two arbitrary vectors $a$ and $b$, and the last inequality uses the fact that 
$\|\psi_V(s,a)\|_2$ is upper bounded by $R$ for all  $(s,a)\in\cS\times\cA$ and $V$. 
Hence, we have 
\begin{align}
    H_0 - H_T= \mathcal{O}\big(d\cdot\log(1+TR^2/(d\lambda))\big) \label{eq:linear_ent}
\end{align}
almost surely. 
\paragraph{Regret Bounds.} With the verification of Assumption \ref{as:var} and the analysis of the upper bound of the cumulative Information gain $H_0-H_T$ in Bayesian linear kernel MDPs, we are ready to specialize the theorems in Appendix \ref{app:theory} in  Bayesian linear kernel MDPs if we determine an appropriate $L$ (the bound of the value).   We analyze the bound of $V(s^\prime)$ for $s^\prime\sim P_{\theta}(s,a)$ and  $\theta\sim \mathcal{N}(\mu,\lambda I_d)$. Define that $\Tilde{\eps}\sim \mathcal{N}(0,1)$. 
By Definition \ref{def:lk mdp}, we have
\begin{align}
    |V(s^\prime)|&= |\Tilde{\eps}+ \psi_V(s,a)^\top\theta | \notag\\
    &\le |\Tilde{\eps}| + |\psi_V(s,a)^\top (\theta-\mu) | + |\psi_V(s,a)^\top \mu |\notag\\
    & \le |\Tilde{\eps}| + \|\psi_V(s,a)\|_2\cdot \|\theta-\mu\|_2+\|\psi_V(s,a)\|_2\cdot \|\mu\|_2\notag\\
    &\le |\Tilde{\eps}| + R\cdot\|\theta\|_2+R\cdot\|\mu\|_2,\label{eq:bound_V}
\end{align}
where the first inequality uses the triangle inequality, the second inequality uses the Cauchy-Schwartz inequality, and the last inequality uses the definition of $R$ in Definition \ref{def:lk mdp}. 
By Definition \ref{def:lk mdp}, \eqref{eq:bound_V}, and the tail behavior of the Gaussian distribution \citep{ghosh2021exponential}, we have
\begin{align*}
    |V(s)|\le \sqrt{2\cdot\log(2/\delta)}+R\cdot\|\mu\|_2+R\cdot\sqrt{2\lambda d\cdot\log(2d\delta)}
\end{align*}
for any $s\in\cS$ and value function $V$ with probability at least $1-\delta$. 
Since the prior distribution of $\theta$ is $\mathcal{N}(0,\lambda I_d)$, it is natural to restrict $\mu$ such that $\|\mu\|_2 \le c d\cdot\log(2d))$ for some absolute constant $c$. Then, we apply the union bound of all $T$ value functions in \texttt{RAFA} and the variants (Algorithms \ref{alg: theory_gen}, \ref{alg: theory_bonus}, and \ref{alg: theory}) to have
\begin{align}
    |V_t(s)|&\le \sqrt{2\cdot\log(2T/\delta)}+R\cdot\|\mu\|_2+R\cdot\sqrt{2\lambda d\cdot\log(2dT\delta)}\notag\\
    &\le (c+1)R\cdot\sqrt{2\lambda d\log(2dT/\delta)})
    \label{eq:before_union2}
\end{align}
for any $t<T$, $s\in\cS$, and value function $V$ with probability at least $1-\delta$. Hence, we can select $L=(c+1)R\cdot\sqrt{2\lambda d\log(2dT/\delta)}$ in Theorems \ref{thm:reg_conc}, \ref{thm:reg_bonus}, and \ref{thm:reg}.
By specializing Theorems \ref{thm:reg_conc}, \ref{thm:reg_bonus}, and \ref{thm:reg}, we summarize the corresponding regret bounds in Table \ref{tab:reg_linear} 
for Algorithms \ref{alg: theory_gen}, \ref{alg: theory_bonus}, and \ref{alg: theory}, respectively. Here, we choose the planning suboptimality of $\texttt{PL}^\eps$ to be $\eps = \mathcal{O}(1/\sqrt{T})$ and all the bounds hold with probability at least $1-\delta$.
\begin{table}[H]
    \centering
 \begin{tabular}{l|c}
\toprule
 \makecell{Algorithm}& Bayesian Regret \\
\midrule
\texttt{RAFA} (Algorithm \ref{alg: theory_gen}) & $\mathcal{O}((1-\gamma)^{-1}(\kappa+1)\sqrt{d^3T}\cdot \log(dT/\delta))$ \\
\texttt{RAFA} with Optimistic Bonus (Algorithm \ref{alg: theory_bonus})  & $((1-\gamma)^{-1}\sqrt{d^3T}\cdot \log(dT/\delta))$  \\
\texttt{RAFA} with Posterior Sampling (Algorithm \ref{alg: theory}) & $\mathcal{O}((1-\gamma)^{-1}\sqrt{d^3T}\cdot \log(dT/\delta))$  \\
\bottomrule
\end{tabular}
    \caption{Bayesian regret of variants of  \texttt{RAFA} in Bayesian linear kernel MDPs (see Definition \ref{def:lk mdp}). Here, we choose the planning suboptimality of $\texttt{PL}^\eps$ to be $\eps = \mathcal{O}(1/\sqrt{T})$ and all the bounds hold with probability at least $1-\delta$. }
    \label{tab:reg_linear}
\end{table}
\section{More Experiments}
\label{sec:add_exp}
In what follows, we provide the detailed setups and additional results of our experiments.
\subsection{Game of 24}
\label{app_game24}

\paragraph{Task Setup.} Figure~\ref{fig:game24} gives an illustrative example for Game of 24. 

\begin{figure}[H]
    \centering
       \fbox{    
	\parbox{0.8\textwidth}{
 \small
 \textbf{\color{ultramarine} [Illustrative example for Game of 24]} 
 \begin{center}
\begin{itemize}
  \item Numbers: [2, 5, 8, 11]
  \item Arithmetic Operations: [$+$, $-$, $\times$, $/$, $($, $)$]
  
  \item \textbf{Solution}: 
  \begin{equation*}
    \begin{aligned}
        (11 - 5) \times 8 / 2 = 24 
    \end{aligned}
  \end{equation*}

\end{itemize}
\end{center}
}}
  \caption{An illustrative example of the Game of 24. The player uses combinations of basic arithmetic operations with four given numbers to get 24.} \label{fig:game24}
\end{figure}

\begin{figure}[H]
\centering
\includegraphics[width=0.47\textwidth]{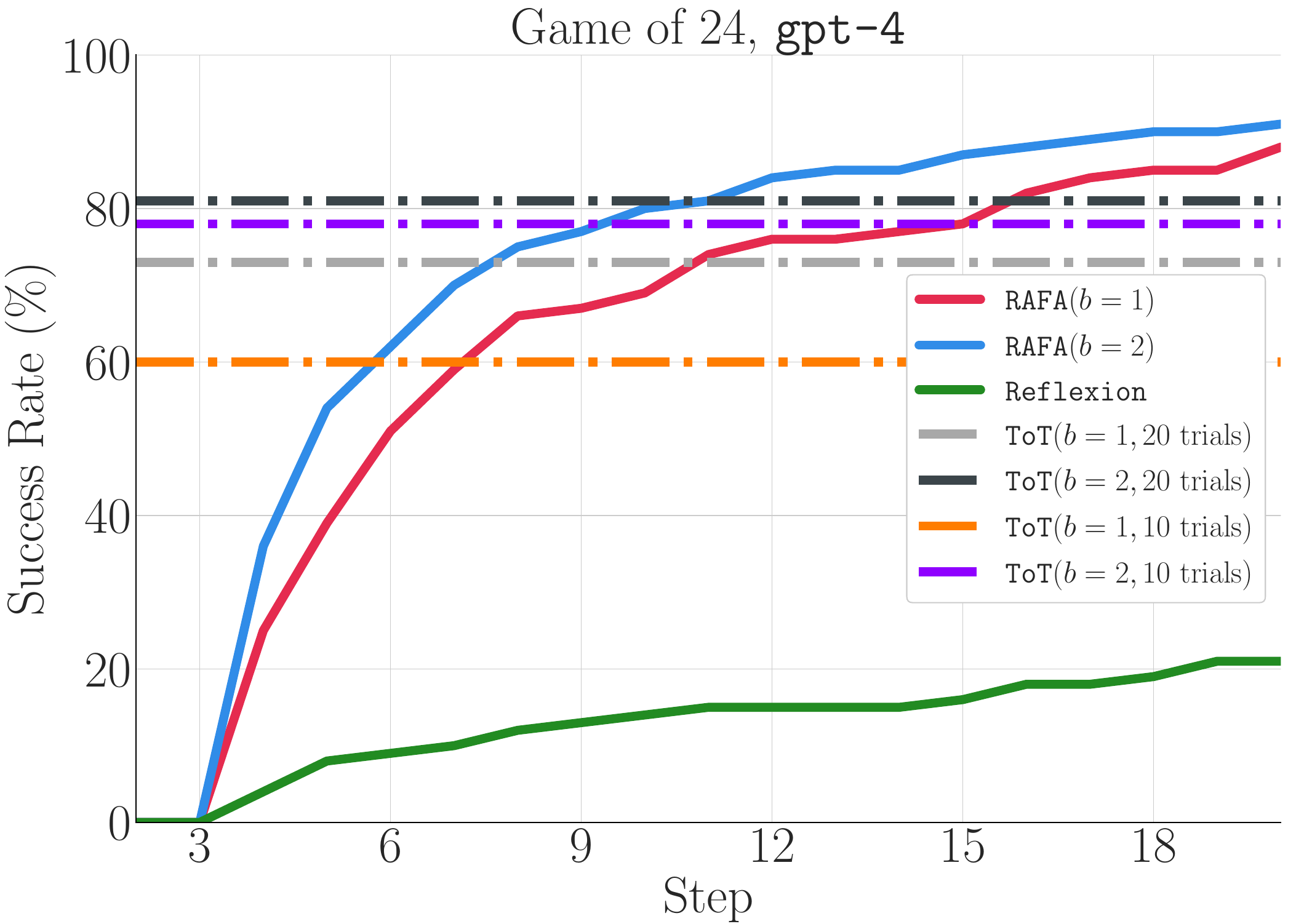}
\hfill
\includegraphics[width=0.47\textwidth]{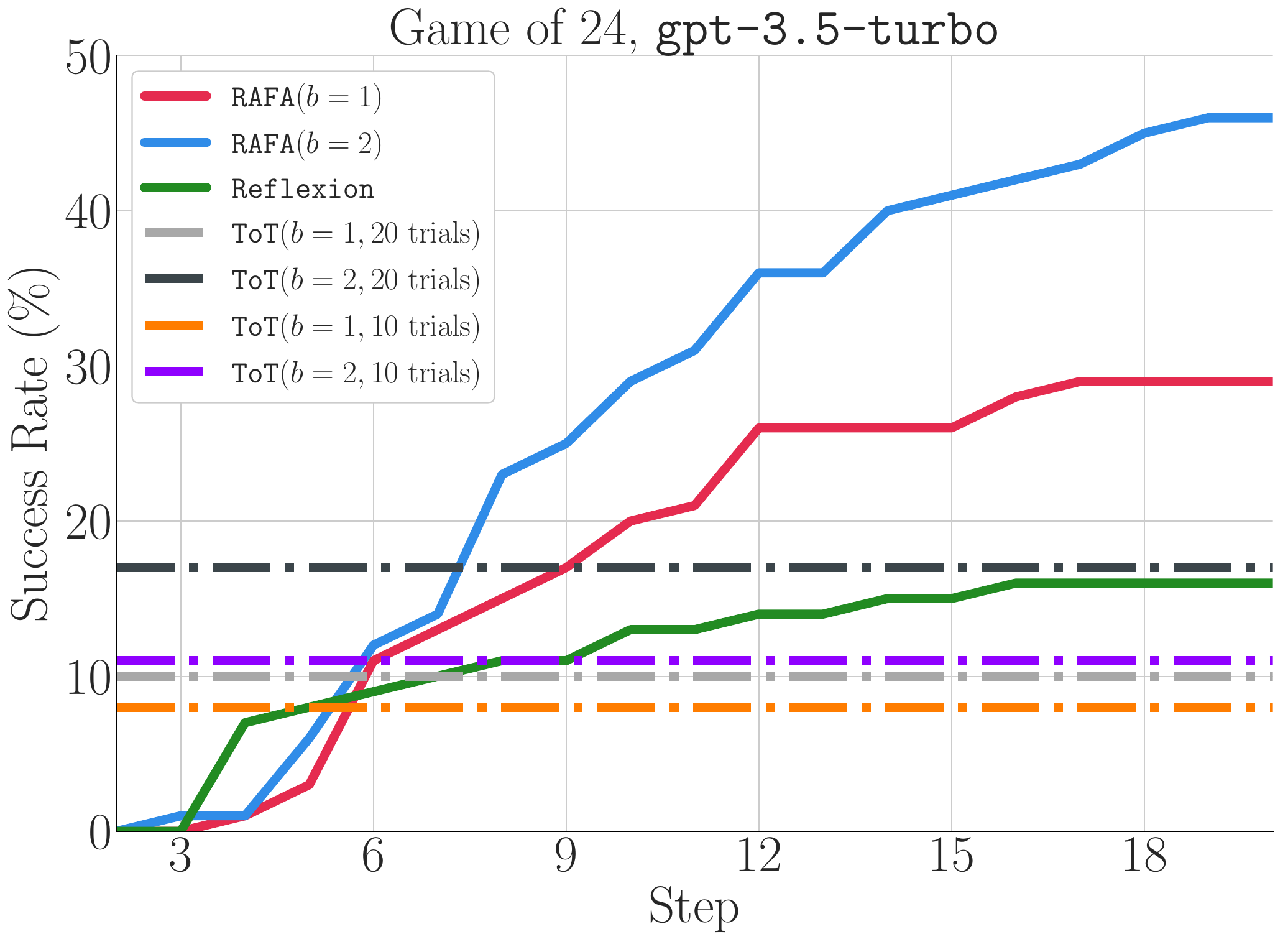}
\caption{ Sample efficiency on Game of 24. \texttt{RAFA} agent achieves strong performance due to an orchestration of reasoning and acting. The success rate at a given step is the number of tasks that is solved within the given step.}
\label{fig:game24_sample_efficiency_2}
\end{figure}

Following \citet{tree_of_thought}, we use the same subset indexed 901-1,000 from a total of 1,362 tasks collected from \texttt{4nums.com}. The index is arranged from easy to hard by human solving time so the subset is relatively challenging. The agent receives a reward of 1 if the proposed formula is correct and the proposed formula is accepted and concatenated into the state; if the final result is exactly 24, the agent receives a reward of 10, and the episode terminates. Otherwise, the agent receives a reward of 0, and the proposed formula is not accepted. We limit the maximum trials for each task to 20 to avoid meaningless retries. The task is successful if the agent receives a return larger than 10 \footnote{For $\texttt{gpt-3.5-turbo}$, we report the success rate when the agent receives a return no less than 3 (i.e., find all sub-steps to get 24 but not necessarily generate a whole correct formula). This is because $\texttt{ToT}$ with $\texttt{gpt-3.5-turbo}$ is known to suffer from correctly get a whole formula due to limited reasoning ability and non-perfect prompts. See \href{https://github.com/princeton-nlp/tree-of-thought-llm/issues/24}{https://github.com/princeton-nlp/tree-of-thought-llm/issues/24} for more details.} (i.e., find a valid solution within 20 steps). We report the final success rate and sample efficiency for each method on the subset of 100 tasks.  Notably, a task is considered successful if the \texttt{RAFA} agent returns one and only one correct formula, which is more strictly evaluated than Tree of Thoughts (\texttt{ToT}, \citet{tree_of_thought}): we allow open-loop agents like \texttt{ToT} to retry 20 times and consider them successful if they generate a valid solution in any of the 20 trials. For \texttt{CoT}~\citep{chain_of_thought} and \texttt{Reflexion}~\citep{shinn2023reflexion} agents, we allow them to reflect on the environment's feedback but require them to generate a plan immediately without sophisticated reasoning.

\paragraph{RAFA Setup.} In the Game of 24, the \texttt{RAFA} agent uses ToT as the planner, regenerates a plan when the agent receives a zero reward and continues acting according to the previous plan when the agent receives a positive reward. We set the base ToT planner with beam search width $b=1,2$ and use both \texttt{gpt-3.5-turbo} and \texttt{gpt-4} to test the RAFA's boost-up over LLM agents with different reasoning abilities. We set the temperature $t = 0.2$ by default to favor rigorous reasoning and $t = 0.7$ for majority voting.

\paragraph{Reduced Hallucination Through Interaction.} A comprehensive review of various method proposals revealed significant hallucination, especially with \texttt{gpt-3.5-turbo}. A common hallucination is that the agent believes she can reuse the same number (e.g. using the number $2$ twice as illustrated in Figure~\ref{fig:game24_illustration}). \texttt{RAFA} efficiently mitigates such hallucination by actively interacting with the environment, displaying exceptional hallucination resistance and improved performance.

\paragraph{Enhanced Efficiency Through Planning.} Evidenced in Figure~\ref{fig:game24_sample_efficiency}, the \texttt{RAFA} agent substantially surpasses the \texttt{Reflexion} baseline, reflecting heightened efficiency and minimized regret by negating careless trials. For example, without carefully planning, agent may give negative answers, e.g,``Impossible to obtain 24 with the given numbers, or unchecked answers, e.g.,``Answer: 6 * 9 / (3 - 2) = 24". This reduction of careless trails is especially achieved when a strong backbone LLMs (e.g., $\texttt{gpt-4}$) is used, even with a basic planning method, such as BFS with $B=1$.

\paragraph{Ablation Study.} The \texttt{RAFA} agent’s performance is dissected by individually examining its components: (1) Planning modules or model/elite LLM, (2) Reflection modules or critic LLM, and (3) Different LLMs. Results, displayed in Table~\ref{tab:game24_results} and Figure~\ref{fig:game24_sample_efficiency}, affirm the substantial contribution of each segment to the aggregate performance. Compared to absent or rudimentary zero-shot planning, a basic planner markedly enhances overall performance. However, augmenting planner strength only offers marginal performance enhancements. Both critic LLM and robust LLM usage emerge as pivotal for optimal performance.

\subsection{ALFWorld}
\label{sec:add_exp_alfw}

\begin{figure}[h]
\centering
\includegraphics[width=1\textwidth]{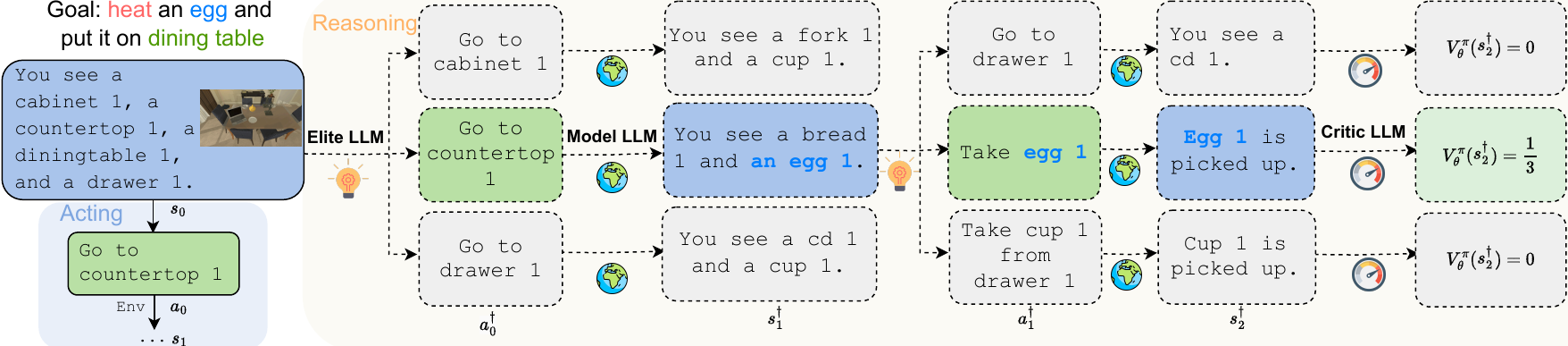}  
\caption{An illustration of \texttt{RAFA} in the ALFWorld environment.\vspace{-0.1cm}}
\label{fig:alf_demo}
\end{figure}

\paragraph{Task Setup.} The action space of ALFWorld consists of high-level actions such as “heat
a potato with a microwave”, which is executed in the underlying embodied simulator through low-level
action primitives. The egocentric visual observations of the simulator are translated into natural language before being provided to the agent. The state is the history of the observations. If a task goal can be precisely achieved by the agent, it will be counted as a success.

\paragraph{RAFA Setup.} In the ALFWorld environment, the \texttt{RAFA} planner is instantiated as Breadth First
Search (BFS). Specifically, $B$ and $U$ are both set to 2, and we use \texttt{gpt-3} (\texttt{text-davinci-003})
for the \texttt{Critic}, \texttt{Model}, and \texttt{Elite} modules. Besides, since it is challenging to prompt the LLM with the stored full trajectories in the memory buffer due to the token limit, we make the following modifications: the \texttt{Model} LLM instance uses only the partial trajectory executed so far in the current episode, and the \texttt{Elite} LLM instance uses the same partial executed trajectory with additional model-generated state-action pairs during the planning subroutine. When switching is triggered after $20$ failed timesteps (i.e., an episode), a summary from the failure trajectory is generated by \texttt{gpt-4} and added to the \texttt{Critic} prompt.

\paragraph{Reduced Hallucination Through Interaction.} The baselines are more likely to hallucinate when the target object is not found after exploring many locations. On the other hand, the critic LLM used in \texttt{RAFA} is able to probe the hallucination by generating the summary ``In this environment, my critic assigned a 1/3 value after taking a knife. However, the task is to take and cool a tomato." and avoid it in the next episode. Therefore, \texttt{RAFA} is more sample-efficient due to an orchestration of reasoning and acting and the ability to mitigate hallucination through interaction.

\paragraph{Ablation Study.} To better understand the role that the planning subroutine plays in the \texttt{RAFA} algorithm, we conduct ablation studies on the search depth $U$ and search breadth $B$. The results are shown in Figure \ref{fig:alf_curve_abl} and \ref{fig:alf_curve_abl_B}, respectively. We observe that when setting the search depth to $B=U=2$, the success rate is higher than when setting the search depth to $U=1$ or setting the search breadth $B=1$, especially at the initial episode. This indicates that the reasoning ability of \texttt{RAFA} is enhanced through the planning subroutine. Besides, the algorithm is also more sample-efficient when setting $B=U=2$, indicating a better capacity for learning and planning through interaction and reasoning.

\begin{figure}[htbp]
  \centering
  \begin{minipage}[t]{0.42\textwidth}
    \centering
    \includegraphics[width=\linewidth]{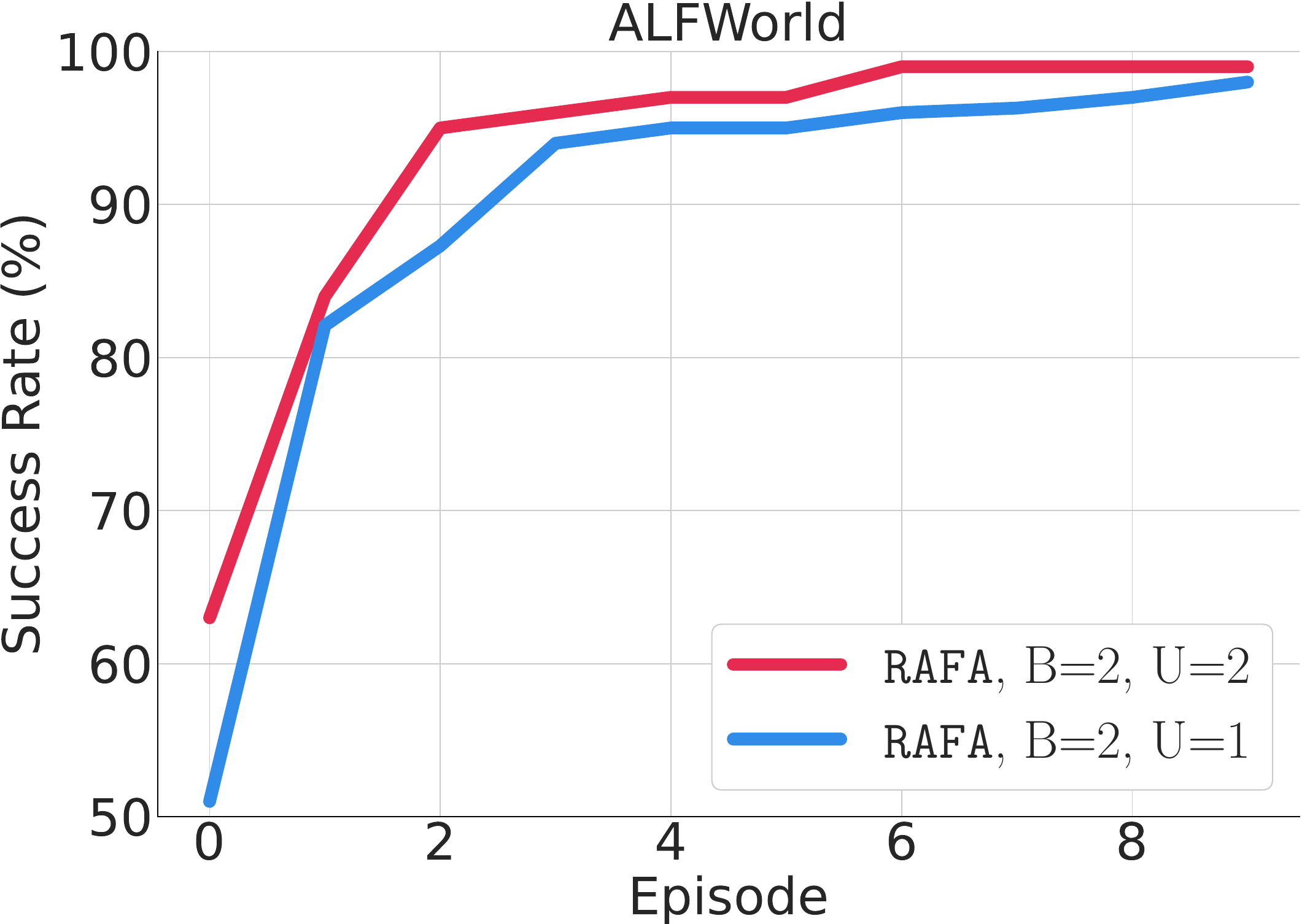}
    \caption{Ablation on the search depth $U$ in the ALFWorld environment.}
    \label{fig:alf_curve_abl}
  \end{minipage}%
  \hspace{0.5cm}
  \begin{minipage}[t]{0.42\textwidth}
    \centering
    \includegraphics[width=\linewidth]{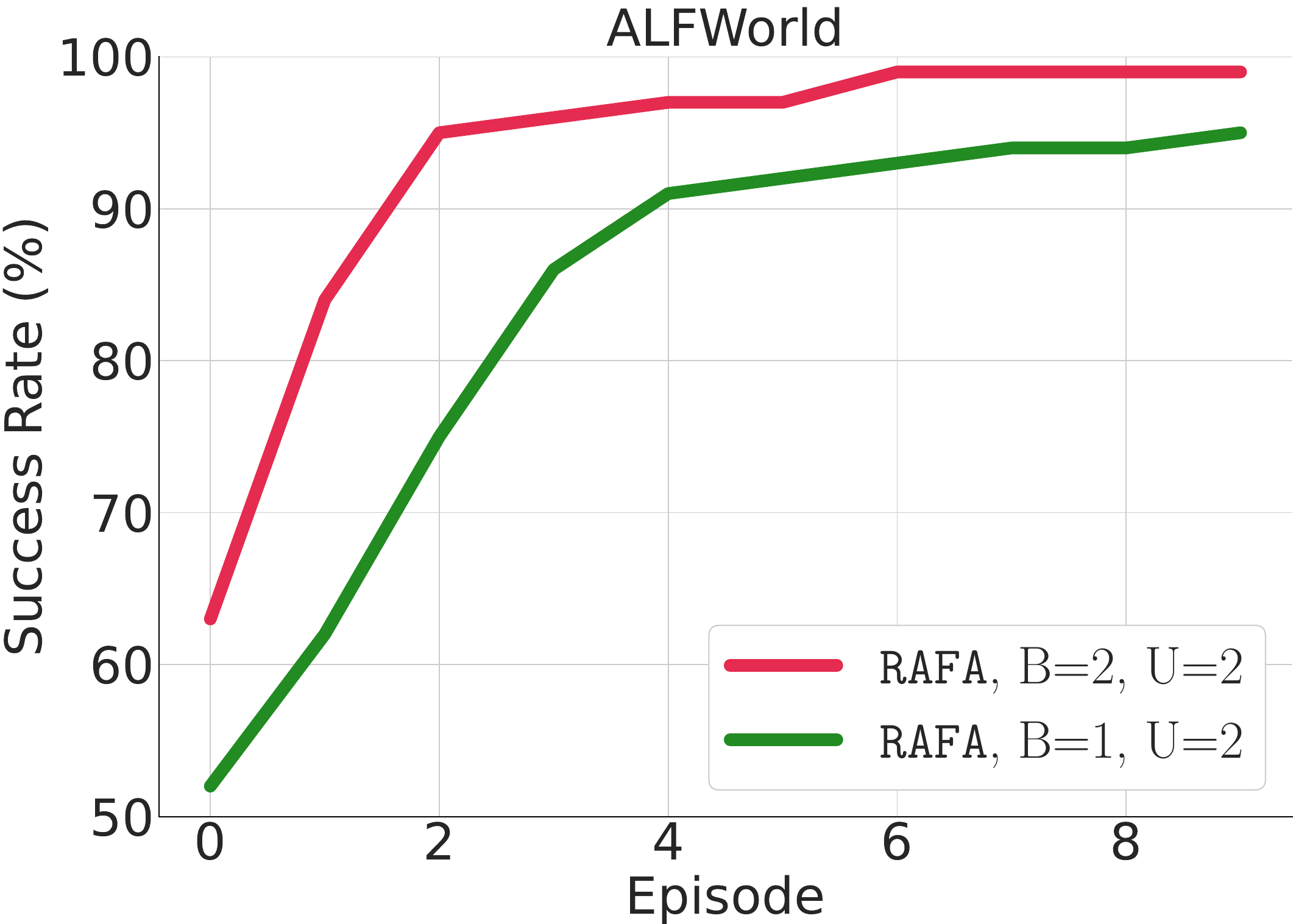}
    \caption{Ablation on the search breadth $B$ in the ALFWorld environment.}
    \label{fig:alf_curve_abl_B}
  \end{minipage}
\end{figure}

\subsection{BlocksWorld}
\paragraph{Task Setup.} The reported success rates are averaged in tasks that require different minimum steps. Specifically, the evaluation is conducted in $57$ 4-step tasks and $114$ 6-step tasks. We set the state as the current arrangement of the blocks and the actions contain Stack, Unstack, Put, and Pickup, coupled with a block being operated.

\begin{figure}[htbp]
  \centering
  \vspace{-0.1cm}
  \includegraphics[width=\linewidth]{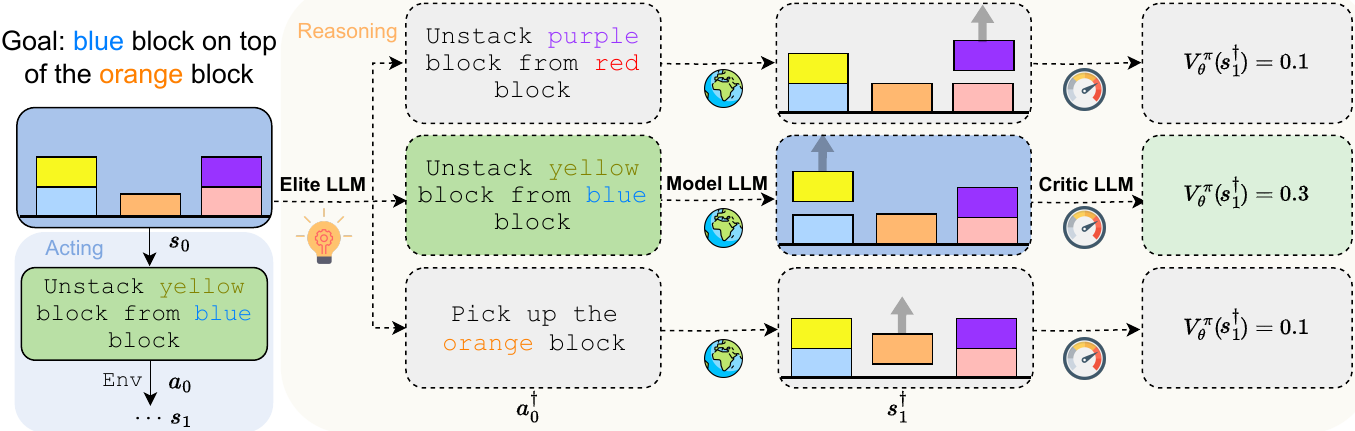}
  \vspace{-0.1cm}
  \setlength{\belowcaptionskip}{-5pt}
\caption{Illustration of \texttt{RAFA} in the BlocksWorld environment.}
\label{fig:blocksworld-example}
\end{figure}

\paragraph{RAFA Setup.} The search space is up to $5^4$ for a 4-step task and is up to $5^6$ for a 6-step task. For 4-step tasks, \texttt{RAFA} can achieve over 50\% success rate within 8 learning steps with \texttt{Vicuna-13B(v1.3)} and achieve over 80\% success rate within 8 learning steps with \texttt{Vicuna-33B(v1.3)}. For 6-step tasks, \texttt{RAFA} can achieve over 40\% success rate within 20 learning steps with \texttt{Vicuna-13B(v1.3)} and achieve over 50\% success rate within 20 learning steps with \texttt{Vicuna-33B(v1.3)}. Empirical results show that \texttt{Vicuna} could produce wrong state transition in the planning phase. \texttt{RAFA} can mitigate hallucination with feedback from failure trajectories and active exploration. One can draw such a conclusion by comparing \texttt{RAFA} with \texttt{RAP} as \texttt{RAP} does not receive feedback from the real environment.


\subsection{Tic-Tac-Toe}
\paragraph{Task Setup.} Tic-Tac-Toe \citep{beck2008combinatorial} is a competitive game in which two players take turns to mark a three-by-three grid with X or O, and a player succeeds when their marks occupy a diagonal, horizontal, or vertical line. 

\begin{figure}[htbp]
  \centering
  \vspace{-0.1cm}
  \includegraphics[width=\linewidth]{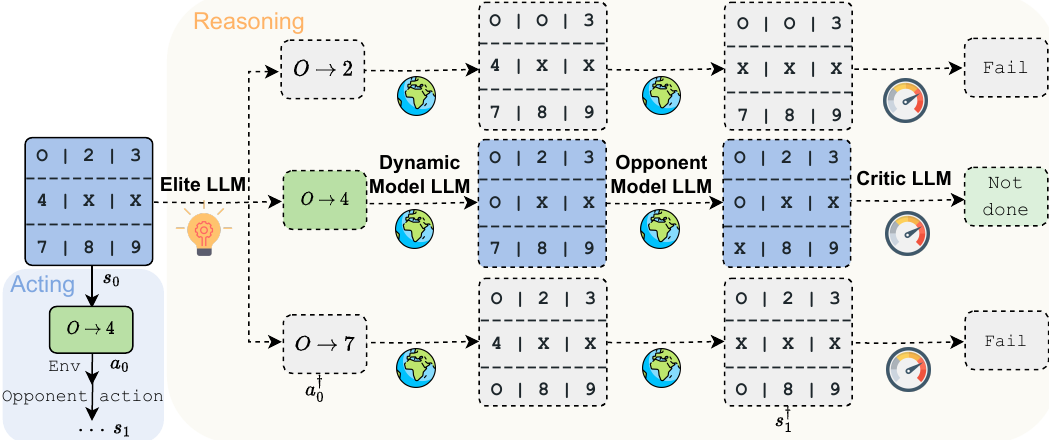}
  \vspace{-0.1cm}
  \setlength{\belowcaptionskip}{-5pt}
\caption{Illustration of \texttt{RAFA} (playing O) in the Tic-Tac-Toe game.  States are represented by a numbered $3 \times 3$ grid and actions are represented by a number between 1-9. The opponent is considered part of the environment.}
\label{fig:tictactoe_rafa}
\end{figure}

We adopt the convention that X plays first. 
As illustrated below in Figure \ref{fig:tictactoe_rafa}, we use a numbered $3 \times 3$ grid to represent a state and a number between 1 and 9 to represent an action, which also illustrates the transition and reward function.
Although Tic-Tac-Toe is a solved game with a forced draw assuming the best play from both players, it remains a challenge for LLMs to accomplish this task even when prompted to play only the optimal moves.
We collected the battle outcomes between different LLM models in Table \ref{tab:tictactoe_initial}, where we notice that \texttt{gpt-4} performs worse when playing as ``O''. Thus, in our experiments, we let \texttt{RAFA} play as ``O'' and let baseline LLM models play as ``X''. 

\begin{table}[ht]
\centering
\begin{tabular}{cc|cc}
\toprule
\multicolumn{2}{c|}{\multirow{2}{*}{X wins : Tie : O wins}} & \multicolumn{2}{c}{O}                     \\ 
\cline{3-4} 
\multicolumn{2}{c|}{}                                       & \multicolumn{1}{c}{\texttt{gpt-3.5}} & \texttt{gpt-4} \\ 
\midrule
\multicolumn{1}{c|}{\multirow{2}{*}{X}}   & \texttt{gpt-3.5}   & \multicolumn{1}{c|}{$55\% : 35\% : 10\%$}              &    $90\% : 0\% : 10\%$   \\ 
\multicolumn{1}{c|}{}                     & \texttt{gpt-4}           & \multicolumn{1}{c|}{$65\% : 15\% : 20\%$}              &    $90\% : 0\% : 10\%$   \\ 
\bottomrule
\end{tabular}
\caption{Probability of ``X wins,'' ``Tie,'' and ``O wins'' in Tic-Tac-Toe. The results are obtained by averaging over 20 simulated games.}
\label{tab:tictactoe_initial}
\end{table}

\paragraph{RAFA Setup.} For implementation, we set $B = 3$ and adopt MCTS to evaluate the proposed actions. We set $U = 4$ which is the maximum game depth. We set a prediction-based switching condition triggered when the prediction does not agree with the observation. Specifically, policy switches when one of the following events occurs:
\begin{itemize}
    \item The \texttt{RAFA} agent takes an action and predicts the next state, which is different from the observed next state.
    \item Before the opponent takes an action, the \texttt{RAFA} agent tries to predict such an action, which is different from the actual action that the opponent takes.
    \item After the opponent takes an action, \texttt{RAFA} agent predicts the next state, which is different from the observed next state.
    \item The \texttt{RAFA} agent predicts the current game status (X wins, O wins, Tie, Not finished), which is different from the environment’s feedback.
\end{itemize}

Besides, we use the ground truth of those predictions to update the agent's belief of the world, which also implicitly affects the agent's policy.

We define a discrete reward function with $r = -1, 0, 1$ corresponding to lose, tie, and win. The agent only gets rewards when the current episode is completed. We define the score of an agent as its expected reward which can be approximated by simulation.
The empirical results are shown in figure \ref{fig:tictactoe_results_gpt4}. We conduct experiments using both \texttt{gpt-4} as the backend. The score of \texttt{RAFA} $(B=4)$ increases as it interacts more with the environment.
By analyzing the generated trajectories, we also notice that although \texttt{RAFA} agent is not perfect, it exploits the weakness of the baseline model well, which is why it almost never loses after $7$ episodes.

\begin{figure}[ht]
\centering
    \includegraphics[width=0.4\linewidth]{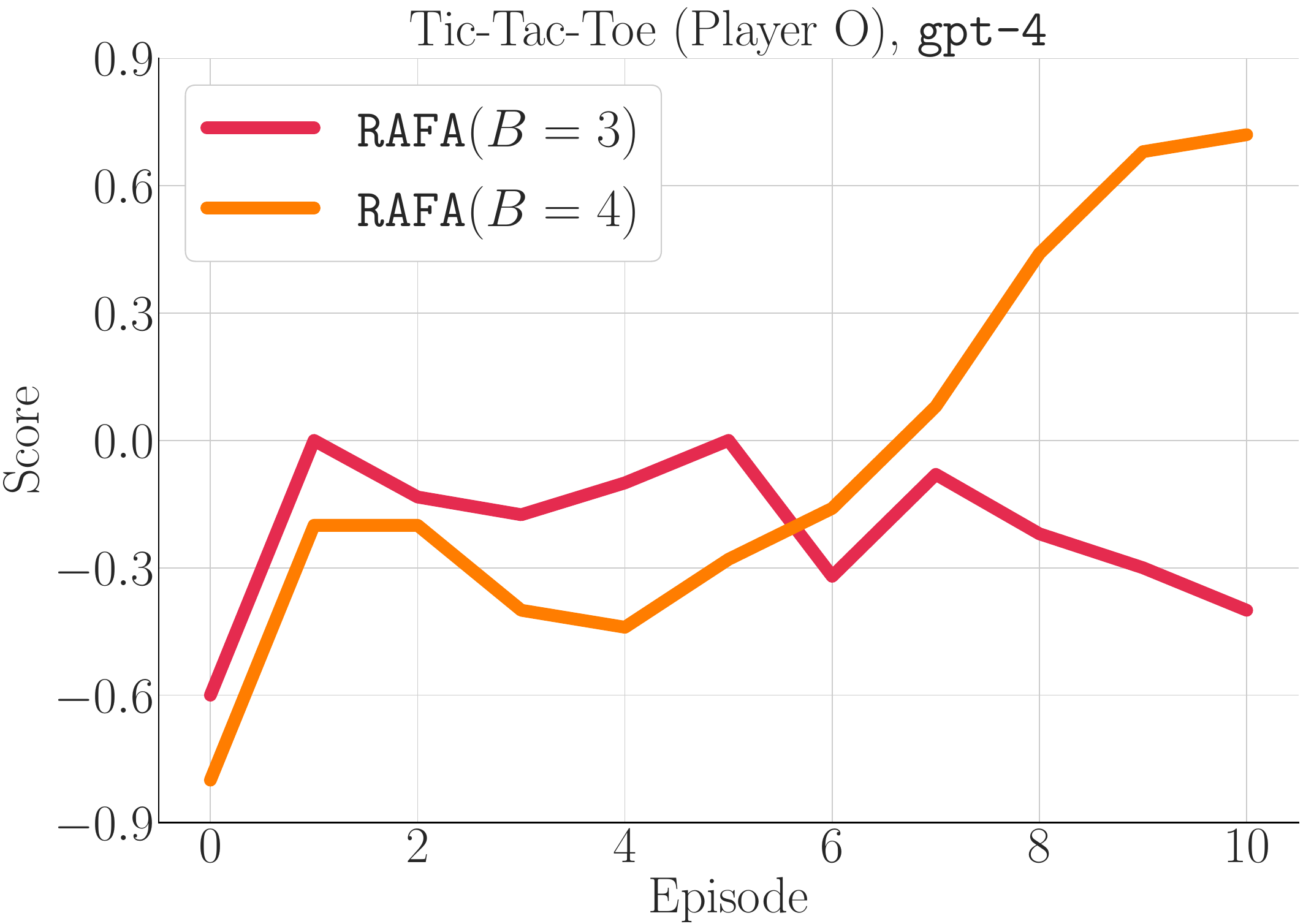}
    \caption{Score curves in the Tic-Tac-Toe game. We use \texttt{gpt-4} as backend. Results are averaged across 10 simulations and smoothed with a window size of $5$.}
    \label{fig:tictactoe_results_gpt4}
\end{figure}

\section{Prompts}
\label{sec:prompts}
In this section, we give details of the prompts used for each task.
\subsection{Game of 24}
\paragraph{Critic LLM.} For the LLM instance of the \texttt{Critic}, we prompt it with the current action (formula) with reward and feedback from the environment. The critic is required to determine whether each formula is valid or not and give a "sure" or "impossible" label for the formula. We use critic prompts to generate reflection for formula proposal and evaluation, respectively.
\vspace{0.3cm}
\begin{Verbatim}
Now we would like to play a game of 24. That is, given 4 numbers, try to use them with arithmetic operations (+ - * /) to get 24. Now we consider the following puzzle: {input}. 
Here is an attempt answer: 
{answer}
And we have the following feedback: 
{feedback}
Now using the above feedback, give 'sure' or 'impossible' labels for each formula with left numbers from each step. Give 'sure' if the formula is correct and can lead to 24 and give 'impossible' if the formula is incorrect or illegal. First repeat the formula with left numbers from each step above and then give the label, with the following form: {{formula}} (left: {{left numbers}}): {{label}}.
\end{Verbatim}
\vspace{0.3cm}
\begin{Verbatim}
Now we would like to play a game of 24. That is, given 4 numbers, try to use them with arithmetic operations (+ - * /) to get 24. Now we consider the following puzzle: {input}. 
Here is an attempt answer:
{answer}
And we have the following feedback:
{feedback}
Now using the above feedback, give 'sure' or 'impossible' labels for left numbers from each step. Give 'sure' if the formula is correct and left numbers can lead to 24 and give 'impossible' if the formula is incorrect or illegal. First repeat the left numbers from each step above and then give the label, with the following form: {{left numbers}}: {{label}}.
\end{Verbatim}

\paragraph{Elite LLM.} We adopt the same prompts used in Tree-of-Thoughts~\citep{tree_of_thought} to propose and evaluate formulas, except that we concatenate the reflections from each step to avoid making repeated mistakes.

\vspace{0.3cm}
\begin{Verbatim}
Now we would like to play a game of 24. That is, given 4 numbers, try to use them with arithmetic operations (+ - * /) to get 24.
Evaluate if given numbers can reach 24 and choose labels from 'sure', 'likely' and 'impossible'. 
What you have learned about the puzzle are summarized below.
{reflections}
Now use numbers and basic arithmetic operations (+ - * /) to generate possible next steps. Make sure use steps that is sure to leads to 24 and avoid steps that are impossible to generate 24. Note that it is possible that we are considering intermediate steps so the numbers of the input may be less than 4.
Example:
Input: 2 8 8 14 
Possible next steps: 
2 + 8 = 10 (left: 8 10 14)
8 / 2 = 4 (left: 4 8 14)
14 + 2 = 16 (left: 8 8 16)
2 * 8 = 16 (left: 8 14 16)
8 - 2 = 6 (left: 6 8 14)
14 - 8 = 6 (left: 2 6 8)
14 /  2 = 7 (left: 7 8 8)
14 - 2 = 12 (left: 8 8 12)
Example:
Input: 2 5 8
5 - 2 = 3 (left: 3 8)
5 * 2 = 10 (left: 10 8)
8 / 2 = 4 (left: 4 5)
Now try with the following input:
Input: {input}
Possible next steps:
 {input}
\end{Verbatim}
\vspace{0.3cm}
\begin{Verbatim}
Now we would like to play a game of 24. That is, given 4 numbers, try to use them with arithmetic operations (+ - * /) to get 24.
Evaluate if given numbers can reach 24 and choose labels from 'sure', 'likely' and 'impossible'. 
What you have learned about the puzzle are summarized below.
{reflections}
If the given numbers are already in the feedback above, just give the answer. Otherwise enumerate possible steps and try to give an approximate answer. Give the final answer in a separated line.
 {input}
\end{Verbatim}
\vspace{0.3cm}
\begin{Verbatim}
Now we would like to play a game of 24. That is, given 4 numbers, try to use them with arithmetic operations (+ - * /) to get 24.
Evaluate if given numbers can reach 24 and choose labels from 'sure', 'likely' and 'impossible'. 
What you have learned about the puzzle are summarized below.
{reflections}
Use numbers and basic arithmetic operations (+ - * /) to obtain 24. Given an input and an answer, give a judgement (sure/impossible) if the answer is correct, i.e. it uses each input exactly once and no other numbers, and reach 24.
Input: 4 4 6 8
Answer: (4 + 8) * (6 - 4) = 24
Judge: 
sure
Input: 2 9 10 12
Answer: 2 * 12 * (10 - 9) = 24
Judge: 
sure
Input: 4 9 10 13
Answer: (13 - 9) * (10 - 4) = 24
Judge: 
sure
Input: 4 4 6 8
Answer: (4 + 8) * (6 - 4) + 1 = 25
Judge: 
impossible
Input: 2 9 10 12
Answer: 2 * (12 - 10) = 24
Judge: 
impossible
Input: 4 9 10 13
Answer: (13 - 4) * (10 - 9) = 24
Judge: 
impossible
Input: {input}
Answer: {answer}
Judge:
\end{Verbatim}

For Chain-of-Thought baselines, we adopt the same methodology, and keep the original prompts except for adding reflections as below.

\vspace{0.3cm}
\begin{Verbatim}
Now we would like to play a game of 24. That is, given 4 numbers, try to use them with arithmetic operations (+ - * /) to get 24.
Evaluate if given numbers can reach 24 and choose labels from 'sure', 'likely' and 'impossible'. 
What you have learned about the puzzle are summarized below.
{reflections}
Now just remember the tips from before (if any) and focus on the new task. Use numbers and basic arithmetic operations (+ - * /) to obtain 24. Each step, you are only allowed to choose two of the remaining numbers to obtain a new number.
Input: 4 4 6 8
Steps:
4 + 8 = 12 (left: 4 6 12)
6 - 4 = 2 (left: 2 12)
2 * 12 = 24 (left: 24)
Answer: (6 - 4) * (4 + 8) = 24
Input: 2 9 10 12
Steps:
12 * 2 = 24 (left: 9 10 24)
10 - 9 = 1 (left: 1 24)
24 * 1 = 24 (left: 24)
Answer: (12 * 2) * (10 - 9) = 24
Input: 4 9 10 13
Steps:
13 - 10 = 3 (left: 3 4 9)
9 - 3 = 6 (left: 4 6)
4 * 6 = 24 (left: 24)
Answer: 4 * (9 - (13 - 10)) = 24
Input: 1 4 8 8
Steps:
8 / 4 = 2 (left: 1 2 8)
1 + 2 = 3 (left: 3 8)
3 * 8 = 24 (left: 24)
Answer: (1 + 8 / 4) * 8 = 24
Input: 5 5 5 9
Steps:
5 + 5 = 10 (left: 5 9 10)
10 + 5 = 15 (left: 9 15)
15 + 9 = 24 (left: 24)
Answer: ((5 + 5) + 5) + 9 = 24
Input: {input}
\end{Verbatim}

\subsection{ALFWorld}
\paragraph{Critic LLM.} For the LLM instance of the \texttt{Critic}, we prompt it with successful examples for each of the six task types to generate the value within each episode. Specifically, the critic LLM first determines if the value of the current state changes. If it changes, the critic LLM then generates the scalar value with the corresponding reason. We provide the following two examples in the put and cool tasks, respectively.
\vspace{0.3cm}
\begin{Verbatim}
You are in the middle of a room. Looking quickly around you, you see a cabinet 4, a cabinet 3, a cabinet 2, a cabinet 1, a countertop 1, a garbagecan 1, a handtowelholder 2, a handtowelholder 1, a sinkbasin 2, a sinkbasin 1, a toilet 1, a toiletpaperhanger 1, and a towelholder 1.Your task is to: put some spraybottle on toilet.
> critic: My task requires two sub-goals in order: take a spraybottle and put the spraybottle on the toilet. My current state satisfies zero of the two sub-goals. The value is 0/2=0.
> OK.
On the cabinet 1, you see a cloth 1, a soapbar 1, a soapbottle 1.
> OK.
The cabinet 2 is closed.
> OK.
You open the cabinet 2. The cabinet 2 is open. In it, you see a candle 1, and a spraybottle 2.
> OK.
You pick up the spraybottle 2 from the cabinet 2.
> critic: Now I take a spraybottle. My current state satisfies the first of the two sub-goals: take a spraybottle. The value is 1/2=0.5.
> OK.
On the toilet 1, you see a soapbottle 2.
> OK.
You put the spraybottle 2 in/on the toilet 1.
> critic: Now I put the spraybottle on the toilet. My current state satisfies all the two sub-goals. The value is 2/2=1.
\end{Verbatim}
\vspace{0.3cm}

\begin{Verbatim}
You are in the middle of a room. Looking quickly around you, you see a cabinet 16, a cabinet 15, a cabinet 14, a cabinet 13, a cabinet 12, a cabinet 11, a cabinet 10, a cabinet 9, a cabinet 8, a cabinet 7, a cabinet 6, a cabinet 5, a cabinet 4, a cabinet 3, a cabinet 2, a cabinet 1, a coffeemachine 1, a countertop 2, a countertop 1, a diningtable 1, a drawer 5, a drawer 4, a drawer 3, a drawer 2, a drawer 1, a fridge 1, a garbagecan 1, a microwave 1, a safe 1, a sinkbasin 1, a stoveburner 4, a stoveburner 3, a stoveburner 2, a stoveburner 1, and a toaster 1.
Your task is to: cool some pan and put it in stoveburner.
> critic: My task requires three sub-goals in order: take a pan, cool the pan with fridge, and put the pan in the stoveburner. My current state satisfies zero of the three sub-goals. The value is 0/3=0.
> OK.
On the stoveburner 1, you see nothing.
> OK.
On the stoveburner 2, you see a pot 1.
> OK.
On the stoveburner 3, you see a pan 1.
> OK.
You pick up the pan 1 from the stoveburner 3.
> critic: Now I take a pan. My current state satisfies the first of the three sub-goals: take a pan. The value is 1/3=0.333.
> OK.
The fridge 1 is closed.
> OK.
You cool the pan 1 using the fridge 1.
> critic: Now I cool the pan. My current state satisfies the first two of the three sub-goals: take a pan and cool the pan with fridge. The value is 2/3=0.667.
> OK.
On the stoveburner 1, you see nothing.
> OK.
You put the pan 1 in/on the stoveburner 1.
> critic: Now I put the pan in the stoveburner. My current state satisfies all the three sub-goals. The value is 3/3=1.
\end{Verbatim}

In addition to the prompt for value generation within each episode, we also prompt the \texttt{Critic} LLM to generate summary feedback according to the past failure trajectory. We use the same prompt for all six task types. Specifically, it contains the full failure trajectory (observation only), as well as the failure feedback at the end of each trajectory.

\vspace{0.3cm}
\begin{Verbatim}
You are in the middle of a room. Looking quickly around you, you see a cabinet 6, a cabinet 5, a cabinet 4, a cabinet 3, a cabinet 2, a cabinet 1, a coffeemachine 1, a countertop 3, a countertop 2, a countertop 1, a drawer 3, a drawer 2, a drawer 1, a fridge 1, a garbagecan 1, a microwave 1, a shelf 3, a shelf 2, a shelf 1, a sinkbasin 1, a stoveburner 4, a stoveburner 3, a stoveburner 2, a stoveburner 1, and a toaster 1.
Your task is to: put a clean plate in countertop.
> critic: My task requires three sub-goals in order: take a plate, clean the plate with sinkbasin, and put the plate on the countertop. My current state satisfies zero of the three sub-goals. The value is 0/3=0.
> OK.
The drawer 1 is closed.
> OK.
You open the drawer 1. The drawer 1 is open. In it, you see a creditcard 2, a knife 1, and a spatula 2.
> OK.
On the stoveburner 2, you see a pan 1.
> OK.
You pick up the pan 1 from the stoveburner 2.
> critic: Now I take a pan. My current state satisfies the first of the three sub-goals: take a plate. The value is 1/3=0.167.
> OK.
On the sinkbasin 1, you see nothing.
> OK.
You put the pan 1 in/on the sinkbasin 1.
> OK.
The microwave 1 is closed.
> OK.
You open the microwave 1. The microwave 1 is open. In it, you see a potato 1.
> OK.
On the garbagecan 1, you see a soapbottle 2.
> OK.
The microwave 1 is open. In it, you see a potato 1.
> OK.
On the coffeemachine 1, you see nothing.
> OK.
On the countertop 2, you see a bread 1, a cellphone 2, a cellphone 1, a papertowelroll 1, a plate 2, and a soapbottle 1.
> OK.
The drawer 2 is closed.
> OK.
You open the drawer 2. The drawer 2 is open. In it, you see a spatula 1.
> OK.
On the sinkbasin 1, you see a pan 1.
> OK.
On the cabinet 3, you see a cup 1.
> OK.
On the countertop 1, you see a apple 2, a dishsponge 2, a potato 3, and a potato 2.
STATUS: FAIL
Failure feedback: In this environment, my critic assigned a 1/3 value after taking a pan. However, the task is to take and clean a plate. I noticed that the plate was found on countertop 2. In the next trial, I will go to countertop 2 to take the plate, then go to a sinkbasin to clean the plate.

You are in the middle of a room. Looking quickly around you, you see a cabinet 20, a cabinet 19, a cabinet 18, a cabinet 17, a cabinet 16, a cabinet 15, a cabinet 14, a cabinet 13, a cabinet 12, a cabinet 11, a cabinet 10, a cabinet 9, a cabinet 8, a cabinet 7, a cabinet 6, a cabinet 5, a cabinet 4, a cabinet 3, a cabinet 2, a cabinet 1, a coffeemachine 1, a countertop 3, a countertop 2, a countertop 1, a drawer 3, a drawer 2, a drawer 1, a fridge 1, a garbagecan 1, a microwave 1, a shelf 3, a shelf 2, a shelf 1, a sinkbasin 1, a stoveburner 4, a stoveburner 3, a stoveburner 2, a stoveburner 1, and a toaster 1.
Your task is to: put two saltshaker in drawer.
> critic: My task requires four sub-goals in order: take the first saltshaker, put the first saltshaker in the drawer, take the second saltshaker, and put the second saltshaker in the drawer. My current state satisfies zero of the four sub-goals. The value is 0/4=0.
> OK.
On the coffeemachine 1, you see a cup 1, a fork 4, and a potato 2.
> OK.
On the toaster 1, you see a pan 1 and a knife 1.
> OK.
On the countertop 3, you see a bowl 2, a houseplant 1, a knife 2, a peppershaker 1, a saltshaker 4, a soapbottle 1, a spatula 2, and a tomato 3.
> OK.
You pick up the saltshaker 4 from the countertop 3.
> critic: Now I take the first saltshaker. My current state satisfies the first of the four sub-goals: take the first saltshaker. The value is 1/4=0.25.
> OK.
The drawer 1 is open. In it, you see nothing.
> OK.
You put the saltshaker 4 in/on the drawer 1.
> critic: Now I put the first creditcard in the dresser. My current state satisfies the first two of the four sub-goals: take the first saltshaker and put the first saltshaker in the drawer. The value is 2/4=0.5.
> OK.
On the stoveburner 1, you see a bread 1.
> OK.
On the stoveburner 2, you see a potato 1, a pot 2, and a knife 3.
> OK.
On the stoveburner 3, you see a fork 1, a fork 2, and a cup 2.
> OK.
On the stoveburner 4, you see a peppershaker 1 and a pan 2.
> OK.
You pick up the peppershaker 1 from the stoveburner 4.
> critic: Now I take the second saltshaker. My current state satisfies the first three of the four sub-goals: take the first saltshaker, put the first saltshaker in the drawer, and take the second saltshaker. The value is 3/4=0.75.
> OK.
The drawer 1 is open. In it, you see a saltshaker 4.
> OK.
You put the peppershaker 1 in/on the drawer 1.
> critic: Now I put the second creditcard in the dresser. My current state satisfies all the four sub-goals. The value is 4/4=1.
STATUS: FAIL
Failure feedback: In this environment, my critic assigned a 1/4 value after taking a saltshaker and a 3/4 value after taking a peppershaker. However, the task is to put two saltshaker in the drawer. I noticed that one saltshaker was put on the countertop 3. I didn't find the second saltshaker in this trial. In the next trial, I will first go to countertop 3 to take the first saltshaker and then explore other locations to find a second saltshaker, including drawer 2-3, shelf 1-3, and cabinet 1-20.

You are in the middle of a room. Looking quickly around you, you see a cabinet 5, a cabinet 4, a cabinet 3, a cabinet 2, a cabinet 1, a coffeemachine 1, a countertop 3, a countertop 2, a countertop 1, a drawer 3, a drawer 2, a drawer 1, a fridge 1, a garbagecan 1, a microwave 1, a shelf 3, a shelf 2, a shelf 1, a sinkbasin 1, a stoveburner 4, a stoveburner 3, a stoveburner 2, a stoveburner 1, and a toaster 1.
Your task is to: put a cool tomato in microwave.
> critic: My task requires three sub-goals in order: take a tomato, cool the tomato with fridge, and put the tomato in the microwave. My current state satisfies zero of the three sub-goals. The value is 0/3=0.
> OK.
On the countertop 1, you see a mug 1, and a saltshaker 1.
> OK.
On the sinkbasin 1, you see a fork 1, a glassbottle 2, a spatula 1, and a spoon 3.
> OK.
You pick up the fork 1 from the sinkbasin 1.
> critic: Now I take a fork. My current state satisfies the first of the six sub-goals: take a fork. The value is 1/3=0.333.
> OK.
On the countertop 2, you see a butterknife 2, a butterknife 1, a cellphone 1, a glassbottle 1, a knife 1, a lettuce 3, a peppershaker 1, a statue 2, a statue 1, and a vase 1.
> OK.
On the countertop 3, you see a bread 1, a butterknife 3, a creditcard 2, a houseplant 1, a knife 3, a knife 2, a mug 2, a peppershaker 2, and a spatula 2.
> OK.
On the stoveburner 4, you see a pan 1.
> OK.
The drawer 3 is closed.
> OK.
You open the drawer 3. The drawer 3 is open. In it, you see a saltshaker 3.
> OK.
The fridge 1 is closed.
> OK.
On the countertop 3, you see a bread 1, a butterknife 3, a creditcard 2, a houseplant 1, a knife 3, a knife 2, a mug 2, a peppershaker 2, and a spatula 2.
> OK.
On the cabinet 1, you see a dishsponge 1.
> OK.
The cabinet 2 is closed.
> OK.
You open the cabinet 2. The cabinet 2 is open. In it, you see nothing.
> OK.
On the cabinet 3, you see a cd 1 and a plate 4.
STATUS: FAIL
Failure feedback: In this environment, my critic assigned a 1/3 value after taking a fork. However, the task is to take and cool a tomato. I didn't find the tomato in this trial. In the next trial, I will explore other locations to find a tomato, including cabinet 4, cabinet 5, coffeemachine 1, microwave 1, shelf 1-3, stoveburner 1-4 and toaster 1, etc.
\end{Verbatim}

\paragraph{Model LLM and Elite LLM.} We use the same prompt for both the \texttt{Model} LLM and the \texttt{Elite} LLM. Specifically, we simply prompt these two instances with the successful past trajectories (observations and actions) for each task type and expect the \texttt{Model} LLM and \texttt{Elite} LLM to generate the possible next observation and the potential actions, respectively. Below, we provide two prompt examples in the put and cool tasks, respectively.

\vspace{0.3cm}
\begin{Verbatim}
You are in the middle of a room. Looking quickly around you, you see a cabinet 4, a cabinet 3, a cabinet 2, a cabinet 1, a countertop 1, a garbagecan 1, a handtowelholder 2, a handtowelholder 1, a sinkbasin 2, a sinkbasin 1, a toilet 1, a toiletpaperhanger 1, and a towelholder 1.
Your task is to: put some spraybottle on toilet.
> go to cabinet 1
On the cabinet 1, you see a cloth 1, a soapbar 1, a soapbottle 1.
> go to cabinet 2
The cabinet 2 is closed.
> open cabinet 2
You open the cabinet 2. The cabinet 2 is open. In it, you see a candle 1, and a spraybottle 2.
> take spraybottle 2 from cabinet 2
You pick up the spraybottle 2 from the cabinet 2.
> go to toilet 1
On the toilet 1, you see a soapbottle 2.
> put spraybottle 2 in/on toilet 1
You put the spraybottle 2 in/on the toilet 1.
\end{Verbatim}
\vspace{0.3cm}

\begin{Verbatim}
You are in the middle of a room. Looking quickly around you, you see a cabinet 16, a cabinet 15, a cabinet 14, a cabinet 13, a cabinet 12, a cabinet 11, a cabinet 10, a cabinet 9, a cabinet 8, a cabinet 7, a cabinet 6, a cabinet 5, a cabinet 4, a cabinet 3, a cabinet 2, a cabinet 1, a coffeemachine 1, a countertop 2, a countertop 1, a diningtable 1, a drawer 5, a drawer 4, a drawer 3, a drawer 2, a drawer 1, a fridge 1, a garbagecan 1, a microwave 1, a safe 1, a sinkbasin 1, a stoveburner 4, a stoveburner 3, a stoveburner 2, a stoveburner 1, and a toaster 1.
Your task is to: cool some pan and put it in stoveburner.
> go to stoveburner 1
On the stoveburner 1, you see nothing.
> go to stoveburner 2
On the stoveburner 2, you see a pot 1.
> go to stoveburner 3
On the stoveburner 3, you see a pan 1.
> take pan 1 from stoveburner 3
You pick up the pan 1 from the stoveburner 3.
> go to fridge 1
The fridge 1 is closed.
> cool pan 1 with fridge 1
You cool the pan 1 using the fridge 1.
> go to stoveburner 1
On the stoveburner 1, you see nothing.
> put pan 1 in/on stoveburner 1
You put the pan 1 in/on the stoveburner 1.
\end{Verbatim}

\subsection{Blocksworld}

\paragraph{Critic LLM.} We evaluate RAFA and RAP with the reward scheme proposed by \cite{rap}. We prompt the language model with the previous state-action trajectory and calculate the log probabilities of taking each feasible action. Given the action taken in the current state, the \texttt{Model} LLM predicts the next state and we calculate the percentage of subgoals completed in the next state. We adopt the prompt examples from \cite{rap} to ensure fairness in comparison.
\vspace{0.3cm}

\begin{Verbatim}
I am playing with a set of blocks where I need to arrange the blocks into stacks. Here are the actions I can do

Pick up a block
Unstack a block from on top of another block
Put down a block
Stack a block on top of another block

I have the following restrictions on my actions:
I can only pick up or unstack one block at a time.
I can only pick up or unstack a block if my hand is empty.
I can only pick up a block if the block is on the table and the block is clear. A block is clear if the block has no other blocks on top of it and if the block is not picked up.
I can only unstack a block from on top of another block if the block I am unstacking was really on top of the other block.
I can only unstack a block from on top of another block if the block I am unstacking is clear.
Once I pick up or unstack a block, I am holding the block.
I can only put down a block that I am holding.
I can only stack a block on top of another block if I am holding the block being stacked.
I can only stack a block on top of another block if the block onto which I am stacking the block is clear.
Once I put down or stack a block, my hand becomes empty.

[STATEMENT]
As initial conditions I have that, the red block is clear, the yellow block is clear, the hand is empty, the red block is on top of the blue block, the yellow block is on top of the orange block, the blue block is on the table and the orange block is on the table.
My goal is to have that the orange block is on top of the red block.

My plan is as follows:

[PLAN]
unstack the yellow block from on top of the orange block
put down the yellow block
pick up the orange block
stack the orange block on top of the red block
[PLAN END]

[STATEMENT]
As initial conditions I have that, the orange block is clear, the yellow block is clear, the hand is empty, the blue block is on top of the red block, the orange block is on top of the blue block, the red block is on the table and the yellow block is on the table.
My goal is to have that the blue block is on top of the red block and the yellow block is on top of the orange block.

My plan is as follows:

[PLAN]
pick up the yellow block
stack the yellow block on top of the orange block
[PLAN END]

[STATEMENT]
As initial conditions I have that, the red block is clear, the blue block is clear, the orange block is clear, the hand is empty, the blue block is on top of the yellow block, the red block is on the table, the orange block is on the table and the yellow block is on the table.
My goal is to have that the blue block is on top of the orange block and the yellow block is on top of the red block.

My plan is as follows:

[PLAN]
unstack the blue block from on top of the yellow block
stack the blue block on top of the orange block
pick up the yellow block
stack the yellow block on top of the red block
[PLAN END]

[STATEMENT]
As initial conditions I have that, the red block is clear, the blue block is clear, the yellow block is clear, the hand is empty, the yellow block is on top of the orange block, the red block is on the table, the blue block is on the table and the orange block is on the table.
My goal is to have that the orange block is on top of the blue block and the yellow block is on top of the red block.

My plan is as follows:

[PLAN]
unstack the yellow block from on top of the orange block
stack the yellow block on top of the red block
pick up the orange block
stack the orange block on top of the blue block
[PLAN END]
\end{Verbatim}
\vspace{0.3cm}

\paragraph{Model LLM.} we prompt the \texttt{Model} LLM with few-shot examples and the current state and action. The \texttt{Model} LLM generates the predicted next state description. We adopt the prompt examples from \cite{rap} to ensure fairness in comparison.

\begin{Verbatim}
I am playing with a set of blocks where I need to arrange the blocks into stacks. Here are the actions I can do 

Pick up a block 
Unstack a block from on top of another block 
Put down a block 
Stack a block on top of another block 

I have the following restrictions on my actions:
I can only pick up or unstack one block at a time. 
I can only pick up or unstack a block if my hand is empty. 
I can only pick up a block if the block is on the table and the block is clear. A block is clear if the block has no other blocks on top of it and if the block is not picked up. 
I can only unstack a block from on top of another block if the block I am unstacking was really on top of the other block. 
I can only unstack a block from on top of another block if the block I am unstacking is clear. Once I pick up or unstack a block, I am holding the block. 
I can only put down a block that I am holding. 
I can only stack a block on top of another block if I am holding the block being stacked. 
I can only stack a block on top of another block if the block onto which I am stacking the block is clear. Once I put down or stack a block, my hand becomes empty.

After being given an initial state and an action, give the new state after performing the action.

[SCENARIO 1]
[STATE 0] I have that, the white block is clear, the cyan block is clear, the brown block is clear, the hand is empty, the white block is on top of the purple block, the purple block is on the table, the cyan block is on the table and the brown block is on the table.
[ACTION] Pick up the brown block.
[CHANGE] The hand was empty and is now holding the brown block, the brown block was on the table and is now in the hand, and the brown block is no longer clear.
[STATE 1] I have that, the white block is clear, the cyan block is clear, the brown block is in the hand, the hand is holding the brown block, the white block is on top of the purple block, the purple block is on the table and the cyan block is on the table.

[SCENARIO 2]
[STATE 0] I have that, the purple block is clear, the cyan block is clear, the white block is clear, the hand is empty, the white block is on top of the brown block, the purple block is on the table, the cyan block is on the table and the brown block is on the table.
[ACTION] Pick up the cyan block.
[CHANGE] The hand was empty and is now holding the cyan block, the cyan block was on the table and is now in the hand, and the cyan block is no longer clear.
[STATE 1] I have that, the cyan block is in the hand, the white block is clear, the purple block is clear, the hand is holding the cyan block, the white block is on top of the brown block, the purple block is on the table and the brown block is on the table.
\end{Verbatim}
\vspace{0.5cm}

\begin{Verbatim}
I am playing with a set of blocks where I need to arrange the blocks into stacks. Here are the actions I can do 

Pick up a block 
Unstack a block from on top of another block 
Put down a block 
Stack a block on top of another block 

I have the following restrictions on my actions:
I can only pick up or unstack one block at a time. 
I can only pick up or unstack a block if my hand is empty. 
I can only pick up a block if the block is on the table and the block is clear. A block is clear if the block has no other blocks on top of it and if the block is not picked up. 
I can only unstack a block from on top of another block if the block I am unstacking was really on top of the other block. 
I can only unstack a block from on top of another block if the block I am unstacking is clear. Once I pick up or unstack a block, I am holding the block. 
I can only put down a block that I am holding. 
I can only stack a block on top of another block if I am holding the block being stacked. 
I can only stack a block on top of another block if the block onto which I am stacking the block is clear. Once I put down or stack a block, my hand becomes empty.

After being given an initial state and an action, give the new state after performing the action.

[SCENARIO 1]
[STATE 0] I have that, the white block is clear, the cyan block is clear, the brown block is clear, the hand is empty, the white block is on top of the purple block, the purple block is on the table, the cyan block is on the table and the brown block is on the table.
[ACTION] Unstack the white block from on top of the purple block.
[CHANGE] The hand was empty and is now holding the white block, the white block was on top of the purple block and is now in the hand, the white block is no longer clear, and the purple block is now clear.
[STATE 1] I have that, the purple block is clear, the cyan block is clear, the brown block is clear, the hand is holding the white block, the white block is in the hand, the purple block is on the table, the cyan block is on the table and the brown block is on the table.

[SCENARIO 2]
[STATE 0] I have that, the purple block is clear, the cyan block is clear, the white block is clear, the hand is empty, the cyan block is on top of the brown block, the purple block is on the table, the white block is on the table and the brown block is on the table.
[ACTION] Unstack the cyan block from on top of the brown block.
[CHANGE] The hand was empty and is now holding the cyan block, the cyan block was on top of the brown block and is now in the hand, the cyan block is no longer clear, and the brown block is now clear.
[STATE 1] I have that, the purple block is clear, the brown block is clear, the cyan block is in the hand, the white block is clear, the hand is holding the cyan block, the purple block is on the table, the white block is on the table and the brown block is on the table.
\end{Verbatim}
\vspace{0.3cm}

\begin{Verbatim}
I am playing with a set of blocks where I need to arrange the blocks into stacks. Here are the actions I can do 

Pick up a block 
Unstack a block from on top of another block 
Put down a block 
Stack a block on top of another block 

I have the following restrictions on my actions:
I can only pick up or unstack one block at a time. 
I can only pick up or unstack a block if my hand is empty. 
I can only pick up a block if the block is on the table and the block is clear. A block is clear if the block has no other blocks on top of it and if the block is not picked up. 
I can only unstack a block from on top of another block if the block I am unstacking was really on top of the other block. 
I can only unstack a block from on top of another block if the block I am unstacking is clear. Once I pick up or unstack a block, I am holding the block. 
I can only put down a block that I am holding. 
I can only stack a block on top of another block if I am holding the block being stacked. 
I can only stack a block on top of another block if the block onto which I am stacking the block is clear. Once I put down or stack a block, my hand becomes empty.

After being given an initial state and an action, give the new state after performing the action.

[SCENARIO 1]
[STATE 0] I have that, the white block is clear, the purple block is clear, the cyan block is in the hand, the brown block is clear, the hand is holding the cyan block, the white block is on the table, the purple block is on the table, and the brown block is on the table.
[ACTION] Put down the cyan block.
[CHANGE] The hand was holding the cyan block and is now empty, the cyan block was in the hand and is now on the table, and the cyan block is now clear.
[STATE 1] I have that, the cyan block is clear, the purple block is clear, the white block is clear, the brown block is clear, the hand is empty, the white block is on the table, the purple block is on the table, the cyan block is on the table, and the brown block is on the table.

[SCENARIO 2]
[STATE 0] I have that, the purple block is clear, the black block is in the hand, the white block is clear, the hand is holding the black block, the white block is on top of the brown block, the purple block is on the table, and the brown block is on the table.
[ACTION] Put down the black block.
[CHANGE] The hand was holding the black block and is now empty, the black block was in the hand and is now on the table, and the black block is now clear.
[STATE 1] I have that, the black block is clear, the purple block is clear, the white block is clear, the hand is empty, the white block is on top of the brown block, the purple block is on the table, the brown block is on the table, and the black block is on the table.
\end{Verbatim}
\vspace{0.5cm}

\begin{Verbatim}
I am playing with a set of blocks where I need to arrange the blocks into stacks. Here are the actions I can do 

Pick up a block 
Unstack a block from on top of another block 
Put down a block 
Stack a block on top of another block 

I have the following restrictions on my actions:
I can only pick up or unstack one block at a time. 
I can only pick up or unstack a block if my hand is empty. 
I can only pick up a block if the block is on the table and the block is clear. A block is clear if the block has no other blocks on top of it and if the block is not picked up. 
I can only unstack a block from on top of another block if the block I am unstacking was really on top of the other block. 
I can only unstack a block from on top of another block if the block I am unstacking is clear. Once I pick up or unstack a block, I am holding the block. 
I can only put down a block that I am holding. 
I can only stack a block on top of another block if I am holding the block being stacked. 
I can only stack a block on top of another block if the block onto which I am stacking the block is clear. Once I put down or stack a block, my hand becomes empty.

After being given an initial state and an action, give the new state after performing the action.

[SCENARIO 1]
[STATE 0] I have that, the white block is clear, the purple block is clear, the cyan block is in the hand, the brown block is clear, the hand is holding the cyan block, the white block is on the table, the purple block is on the table, and the brown block is on the table.
[ACTION] Stack the cyan block on top of the brown block.
[CHANGE] The hand was holding the cyan block and is now empty, the cyan block was in the hand and is now on top of the brown block, the brown block is no longer clear, and the cyan block is now clear.
[STATE 1] I have that, the cyan block is clear, the purple block is clear, the white block is clear, the hand is empty, the cyan block is on top of the brown block, the brown block is on the table, the purple block is on the table, and the white block is on the table.

[SCENARIO 2]
[STATE 0] I have that, the purple block is clear, the black block is in the hand, the white block is clear, the hand is holding the black block, the white block is on top of the brown block, the purple block is on the table, and the brown block is on the table.
[ACTION] Stack the black block on top of the purple block.
[CHANGE] The hand was holding the black block and is now empty, the black block was in the hand and is now on top of the purple block, the purple block is no longer clear, and the black block is now clear.
[STATE 1] I have that, the black block is clear, the white block is clear, the hand is empty, the black block is on top of the purple block, the white block is on top of the brown block, the brown block is on the table, and the purple block is on the table.
\end{Verbatim}
\vspace{0.3cm}

\subsection{Tic-Tac-Toe}

\paragraph{Elite LLM}

\begin{Verbatim}
In the game of Tic-Tac-Toe, two players, "X" and "O," alternate placing their symbols on a 3x3 grid. The objective is to be the first to get three of their symbols in a row, either horizontally, vertically, or diagonally. We use numbers to indicate empty positions, and then replace them with "X" or "O" as moves are made. For example, an empty board is denoted by

1 | 2 | 3
---------
4 | 5 | 6
---------
7 | 8 | 9

Your task is to identify the optimal position for the next move based on the current board state. Assume that it's your turn and you're playing as "{role}". Please make sure the optimal position is EMPTY. For example, in the following Tic-Tac-Toe Board:

1 | 2 | 3
---------
4 | X | 6
---------
7 | 8 | 9

Position 5 is occupied by "X". Thus, position 5 is not an optimal position. Provide only the optimal position in the first line. In the second line, give a brief explanation for this choice.

Current Tic-Tac-Toe Board:

{state}

Role: {role}

Optimal Position:
\end{Verbatim}

\paragraph{Model LLM}

\vspace{0.3cm}
\begin{Verbatim}
Predict the Next State of the Tic-Tac-Toe Board

In a game of Tic-Tac-Toe, two players, "X" and "O," take turns to place their symbols on a 3x3 grid. Your task is to predict what the board will look like after a specified move has been made.

Examples
{examples}
Now, Predict the Next State of the Following Tic-Tac-Toe Board:
Initial Tic-Tac-Toe Board:

{state}

Move: Player puts "{role}" in position {action}.

Updated Board:
\end{Verbatim}
\vspace{0.3cm}

\begin{Verbatim}
In Tic-Tac-Toe, each player takes turns placing their respective symbols ("X" or "O") on a 3x3 board. Your task is to predict where the opponent will place their symbol based on their past moves and the current board state.

Example
Tic-Tac-Toe Board:

O | X | O
---------
X | O | X
---------
7 | 8 | X

Opponent's Move: "O" in position 7
{examples}
Here's how the Tic-Tac-Toe board currently looks:
Tic-Tac-Toe Board:

{state}

Given the history and current board state, where do you think the opponent will place their "{role}" next? Please make sure the output is an empty position without "X" or "O".

Opponent's Move: "{role}" in position
\end{Verbatim}

\paragraph{Critic LLM}

\vspace{0.3cm}
\begin{Verbatim}
Determine the Winner in a Tic-Tac-Toe Game

In Tic-Tac-Toe, two players, "X" and "O" take turns to place their respective symbols on a 3x3 board. The first player to get three of their symbols in a row, either horizontally, vertically, or diagonally, wins the game. Your task is to evaluate the board state and determine if there is a winner.

Examples

Example
Tic-Tac-Toe Board:

O | X | O
---------
X | X | X
---------
O | O | X

Question: Is there a winner?

Answer: Let's think step by step.

First row: O X O, no winner
Second row: X X X, X wins

Therefore, "X" wins

Example
Tic-Tac-Toe Board:

X | 2 | O
---------
4 | O | X
---------
O | X | 9

Question: Is there a winner?

Answer: Let's think step by step.

First row: X 2 O, no winner
Second row: 4 O X, no winner
Third row: O X 9, no winner
First column: X 4 O, no winner
Second column: 2 O X, no winner
Thrid column: O X 9, no winner
Main diagonal: X O 9, no winner
Anti-diagonal: O O O, O wins

Therefore, "O" wins.
{examples}
Now, for the Current Tic-Tac-Toe Board:
Tic-Tac-Toe Board:

{state}

Question: Is there a winner?

Answer: Let's think step by step.
\end{Verbatim}
\vspace{0.3cm}

\begin{Verbatim}
In the game of Tic-Tac-Toe, two players alternate turns to fill a 3x3 grid with their respective symbols: "X" and "O". A board is considered "completely filled" when all nine cells of the grid contain either an 'X' or an 'O', with no empty spaces or other characters. 

Examples:
{examples}
Now for the Current Tic-Tac-Toe Board:
Tic-Tac-Toe Board:

{state}

Is the board completely filled?

Answer:
\end{Verbatim}
\vspace{0.3cm}

\end{document}